\documentclass{article}
\PassOptionsToPackage{numbers}{natbib}

     \usepackage[final]{neurips_2021}

\usepackage[utf8]{inputenc} %
\usepackage[T1]{fontenc}    %
\usepackage{hyperref}       %
\usepackage{url}            %
\usepackage{booktabs}       %
\usepackage{amsfonts}       %
\usepackage{nicefrac}       %
\usepackage{microtype}      %
\usepackage[dvipsnames,table]{xcolor}         %

\usepackage{adjustbox}

\let\cite\citep
\renewcommand{\citet}[1]{\citeauthor{#1}~\cite{#1}}

\usepackage{comment}

\usepackage{makecell}
\usepackage{multicol}

\definecolor{darkmidnightblue}{rgb}{0.0, 0.2, 0.4}
\hypersetup{
	colorlinks=true,
	linkcolor=darkmidnightblue,
	citecolor=darkmidnightblue,
	filecolor=darkmidnightblue,
	urlcolor=darkmidnightblue
}

\usepackage{siunitx}
\usepackage{subcaption}
\usepackage{algorithm}
\usepackage{algpseudocode}

\usepackage[font=small,labelfont=bf]{caption}
\usepackage{wrapfig}
\usepackage[leftcaption]{sidecap}

\usepackage{multirow}
\usepackage{pgfplots}
\pgfplotsset{compat=1.16}
\usetikzlibrary{fit,arrows.meta,mindmap,shadows,backgrounds,calc,positioning,shapes.misc}

\graphicspath{{figures/}}

\usepackage{amsmath}
\usepackage{amssymb}
\usepackage{amsthm}
\usepackage{bbm}
\usepackage{bm}

\usepackage{xspace}
\makeatletter
\DeclareRobustCommand\onedot{\futurelet\@let@token\@onedot}
\def\@onedot{\ifx\@let@token.\else.\null\fi\xspace}
\def\eg{e.g\onedot} 
\def\ie{i.e\onedot}

\def\wrt{w.r.t\onedot} 

\makeatother

\theoremstyle{plain}
\newtheorem{theorem}{Theorem}

\newtheorem{lemma}{Lemma}
\newtheorem*{lemma*}{Lemma}

\newtheorem*{subproblem*}{Subproblem}

\newtheorem*{problem*}{Problem}

\newtheorem*{conjecture*}{Conjecture}
\newtheorem{assumption}{Assumption}

\theoremstyle{definition}

\newtheorem*{definition*}{Definition}

\newtheorem*{remark}{Remark}

\newtheorem*{example*}{Example}

\DeclareMathOperator*{\argmin}{argmin}
\DeclareMathOperator*{\argmax}{argmax}

\DeclareMathOperator{\dom}{dom}

\DeclareMathOperator{\vect}{vec}
\DeclareMathOperator{\diag}{diag}

\DeclareMathOperator{\softmax}{softmax}

\newcommand{\abs}[1]{\left| #1 \right|} %
\newcommand{\norm}[1]{\left\| #1 \right\|} %
\newcommand{\set}[1]{\left\{ #1 \right\}}
\newcommand\inner[1]{\left\langle #1 \right\rangle}

\makeatletter
\newcommand*\dotp{\mathpalette\dotp@{.5}}
\newcommand*\dotp@[2]{\mathbin{\vcenter{\hbox{\scalebox{#2}{$\m@th#1\bullet$}}}}}
\makeatother

\newcommand{\given}{\, | \,}

\newcommand{\NN}{{\mathbb{N}}}
\newcommand{\RR}{{\mathbb{R}}}

\newcommand{\cE}{{\mathcal{E}}}

\newcommand{\cG}{{\mathcal{G}}}

\newcommand{\cN}{{\mathcal{N}}}
\newcommand{\cO}{{\mathcal{O}}}

\newcommand{\cS}{{\mathcal{S}}}

\newcommand{\cV}{{\mathcal{V}}}

\newcommand{\cX}{{\mathcal{X}}}

\usepackage{xparse}
\makeatletter
\DeclareDocumentCommand{\newmathcommand}{ m O{0} m }{%
	\ifcsname\expandafter\@gobble\string#1\space\endcsname
	\expandafter\expandafter\expandafter\let\expandafter\csname old\string#1\expandafter\endcsname\expandafter=\csname\expandafter\@gobble\string#1\space\endcsname
	\else
	\expandafter\let\csname old\string#1\endcsname=#1
	\fi
	\expandafter\newcommand\csname new\string#1\endcsname[#2]{#3}
	\DeclareRobustCommand#1{%
		\ifmmode
		\expandafter\let\expandafter\next\csname new\string#1\endcsname
		\else
		\expandafter\let\expandafter\next\csname old\string#1\endcsname
		\fi
		\next
	}%
}
\makeatother

\newcommand{\0}{{\mathbf{0}}}
\newcommand{\1}{{\mathbf{1}}}

\newmathcommand{\a}{{\mathbf{a}}}
\newmathcommand{\b}{{\mathbf{b}}}

\newmathcommand{\c}{{\mathbf{c}}} 

\newmathcommand{\d}{{\mathbf{d}}}
\newcommand{\e}{{\mathbf{e}}}
\newcommand{\f}{{\mathbf{f}}}
\newcommand{\g}{{\mathbf{g}}}
\newmathcommand{\k}{{\mathbf{k}}}
\newmathcommand{\m}{{\mathbf{m}}}
\newmathcommand{\o}{{\mathbf{o}}}
\newmathcommand{\p}{{\mathbf{p}}}
\newmathcommand{\q}{{\mathbf{q}}}
\newcommand{\s}{{\mathbf{s}}}
\newmathcommand{\t}{{\mathbf{t}}}

\newmathcommand{\u}{{\mathbf{u}}}
\newmathcommand{\v}{{\mathbf{v}}}
\newcommand{\x}{\mathbf{x}}
\newcommand{\y}{{\mathbf{y}}}
\newcommand{\z}{{\mathbf{z}}}

\newmathcommand{\C}{{\mathbf{C}}}
\newmathcommand{\G}{{\mathbf{G}}}
\newmathcommand{\H}{{\mathbf{H}}}
\newmathcommand{\M}{{\mathbf{M}}}

\newcommand{\A}{{\mathbf{A}}}

\newcommand{\I}{{\mathbf{I}}}
\newcommand{\K}{{\mathbf{K}}}
\newmathcommand{\P}{{\mathbf{P}}}
\newmathcommand{\Q}{{\mathbf{Q}}}
\newmathcommand{\S}{{\mathbf{S}}}

\newcommand{\X}{{\mathbf{X}}}
\newcommand{\Y}{{\mathbf{Y}}}
\newcommand{\Z}{{\mathbf{Z}}}

\newcommand{\thetab}{\pmb{\theta}}

\newcommand{\mub}{\pmb{\mu}}

\newcommand{\rhob}{\pmb{\rho}}

\newcommand{\Thetab}{\pmb{\Theta}}

\usepackage{mathtools}

\newcommand{\iter}[1]{^{#1}}

\newcommand{\fw}{Frank-Wolfe\xspace}
\newcommand{\voc}{PASCAL VOC\xspace}

\newcommand{\lp}{\mathrm{LP}}

\newcommand{\new}[1]{#1}
\newenvironment{newtext}
{}
{}

\newcommand{\meanstd}[2]{$\underset{\pm #2}{#1}$}

\usepackage{xifthen}
\newcommand{\lfw}[1][]{%
	\ifthenelse{\isempty{#1}}{$\ell_2$FW\xspace}{$\ell_2$FW\hspace{-2pt}$_{#1}$\xspace}%
}

\newcommand{\efw}[1][]{%
	\ifthenelse{\isempty{#1}}{$e$FW\xspace}{$e$FW\hspace{-2pt}$_{#1}$\xspace}%
}

\colorlet{darkgreen}{green!50.1960784313725!black}

\colorlet{efwcolor}{OliveGreen}
\colorlet{lfwcolor}{red}
\colorlet{pgdcolor}{Magenta}
\colorlet{pgmcolor}{Orange}
\colorlet{admmcolor}{Maroon}
\colorlet{mfcolor}{blue}
\colorlet{fwcolor}{Cyan}

\DeclareUnicodeCharacter{2212}{-}

\makeatletter
\renewcommand{\paragraph}{%
	\@startsection{paragraph}{4}%
	{\z@}{0.50ex \@plus 1ex \@minus .2ex}{-1em}%
	{\normalfont\normalsize\bfseries}%
}
\makeatother

\usepackage{minitoc}

\usepackage{tocloft}
\tightmtctrue 
\noptcrule

\newcommand{\mytitle}{Regularized Frank-Wolfe for Dense CRFs: Generalizing Mean Field and Beyond}

\title{\mytitle}

\author{%
	\DJ.Khu\^e L\^e-Huu %
	\qquad\qquad Karteek Alahari \\
	Univ.~Grenoble Alpes, Inria, CNRS, Grenoble INP, LJK\\
	38000 Grenoble, France\\
	\texttt{\{khue.le,karteek.alahari\}@inria.fr}
}

\begin{document}
	
\newgeometry{
	textheight=9.1in,
	textwidth=5.63in,
	top=1in,
	headheight=12pt,
	headsep=25pt,
	footskip=30pt
}
	
\maketitle

\begin{abstract}
	We introduce \emph{regularized Frank-Wolfe}, a general and effective algorithm for inference and learning of dense conditional random fields (CRFs). The algorithm optimizes a nonconvex continuous relaxation of the CRF inference problem using vanilla Frank-Wolfe with approximate updates, which are equivalent to minimizing a regularized energy function. Our proposed method is a generalization of existing algorithms such as mean field or concave-convex procedure. This perspective not only offers a unified analysis of these algorithms, but also allows an easy way of exploring different variants that potentially yield better performance. We illustrate this in our empirical results on standard semantic segmentation datasets, where several instantiations of our regularized Frank-Wolfe outperform mean field inference, both as a standalone component and as an end-to-end trainable layer in a neural network. We also show that dense CRFs, coupled with our new algorithms, produce significant improvements over strong CNN baselines.
\end{abstract}

\section{Introduction}

Fully-connected or \emph{dense} conditional random fields (CRFs)~\cite{krahenbuhl2011efficient}---combined with strong pixel-level classifiers such as a convolutional neural network (CNN)~\cite{fukushima1980neocognitron,lecun1989backpropagation}---have been a highly-successful paradigm for semantic segmentation. Top-performing systems on the PASCAL VOC benchmark~\cite{everingham2010pascal} used to include a CRF as either a post-processing step~\cite{chandra2016fast,chen2016attention,chen2017deeplab,dai2015boxsup,li2017not,lin2016efficient,lin2017refinenet} or a trainable component~\cite{arnab2016higher,liu2015semantic,shen2017semantic,sun2016mixed,yu2016multi,zheng2015conditional}.
However, as CNNs got stronger, the improvements that CRFs brought decreased over time, and as a result they fell out of favor since 2017~\cite{lin2017refinenet}.

In this paper, we revisit dense CRFs with two contributions. First, on the theoretical side, we propose \emph{regularized Frank-Wolfe}, a new class of algorithms for inference and learning of CRFs that perform better than the popular \emph{mean field}~\cite{krahenbuhl2011efficient,krahenbuhl2013parameter,parisi1988statistical}---the method of choice in the aforementioned works. Regularized Frank-Wolfe optimizes a nonconvex continuous relaxation of the CRF inference problem (\S\ref{sec:background}) by performing approximate conditional-gradient updates (\S\ref{sec:smoothing-perspective}), which is equivalent to minimizing a regularized energy using the generalized Frank-Wolfe method~\cite{mine1981minimization} (\S\ref{sec:regularized-energy-perspective}). Several of its instantiations lead to new algorithms that have not been studied before in the MAP inference literature (\S\ref{sec:instantiations}). Moreover, we show that it also includes several existing methods, including mean field and the concave-convex procedure~\cite{yuille2002concave}, as special cases (\S\ref{sec:special-cases}). This generalized perspective allows a unified analysis of all these old and new algorithms in a single framework (\S\ref{sec:theoretical-analysis}). In particular, we show that they achieve a sublinear rate of convergence $\cO(1/\sqrt{k})$ for suitable stepsize schemes, and in certain cases (such as strongly-convex regularizer or concave energy) this can be improved to $\cO(1/k)$ (\S\ref{sec:convergence-analysis}). Furthermore, we provide a tightness analysis for the resulting nonconvex relaxation of the regularized energy, recovering some existing tightness results \cite{berthod1982definition,lehuu2018continuous,ravikumar2006quadratic} as special cases (\S\ref{sec:tightness-analysis}). The proposed algorithms are easy to implement, converge quickly in practice, and have (sub)differentiable iterates. Such properties are important for successful \emph{gradient-based learning} via backpropagation~\cite{linnainmaa1970representation,rumelhart1986learning}. %

Our second contribution lies on the practical side. In addition to mean field and regularized Frank-Wolfe variants, we re-implement several existing first-order inference methods~\cite{larsson2017projected,lehuu2018continuous,lions1979splitting}---those that are amenable to gradient-based learning---for comparison. Remarkably, we find that dense CRFs can still achieve important improvements over the strong DeepLabv3+~\cite{chen2018encoder} CNN model for all these solvers (\S\ref{sec:experiments}). In particular, our best variant of regularized Frank-Wolfe achieves a mean intersection-over-union (mIoU) score of $88.0$ on the PASCAL VOC test set (\S\ref{sec:learning-performance}), improving over DeepLabv3+. We hope that these encouraging results could attract interest from the community in considering dense CRFs (again) for tasks such as semantic segmentation. Our source code is made publicly available under the GNU general public license for this purpose.\footnote{\url{https://github.com/netw0rkf10w/CRF}}

\section{Background}\label{sec:background}

\subsection{Inference in CRFs}\label{sec:inference-in-CRFs}
Let $\s\in \cS_1\times\cdots\times\cS_n$ denote an assignment to $n$ discrete random variables $S_1,\ldots,S_n$, where each variable $S_i$ takes values in a finite set of states (or \emph{labels}) $\cS_i$. Let $\cG=(\cV,\cE)$ be a graph of $n$ nodes ($\cV = \set{1,2,\dots,n}$). A Markov random field (MRF) defined by $\cG$ encodes a family of joint distributions that can be factorized as follows, where $\phi_{i}:\cS_i\to\RR_{+}$ and $\phi_{ij}:\cS_{i}\times\cS_j\to\RR_{+}$ are the so-called \emph{unary} and \emph{pairwise} (respectively) potential functions, and $Z$ is a normalization factor:
\begin{equation}\label{eq:mrf-distribution}
	p(\s) = \frac{1}{Z}\prod_{i\in\cV}\phi_{i}(s_i)\prod_{ij\in\cE}\phi_{ij}(s_i,s_j).
\end{equation}
Note that~\eqref{eq:mrf-distribution} can also include conditional distributions, \ie, $p(\s\given\o)$ with observed variables $\o$. In this case the potentials may also depend on $\o$, \eg, $\phi_{i}(s_i;\o)$, and this model is referred to as conditional random field (CRF)~\cite{lafferty2001conditional}. We will present later (\S\ref{sec:experimental-setup}) such a model for image segmentation.
In the following, we use MRF and CRF interchangeably. We assume further that all nodes have the same set of labels: $\cS_i = \cS \ \forall i$, with cardinality $d = \abs{\cS}$. It is convenient to express $p(\s)$ as $\frac{1}{Z}\exp(-e(\s))$, where the so-called \emph{energy} $e(\s)$ is defined as
\begin{equation}\label{eq:mrf-energy}
	e(\s) = \sum_{i\in\cV} \theta_{i}(s_i) + \sum_{ij\in\cE}\theta_{ij}(s_i,s_j),\quad\mbox{with } \theta_i(s_i) = -\log \phi_{i}(s_i) \mbox{ (idem for $\theta_{ij}$)}.
\end{equation}
The task of maximum a posteriori (MAP) inference consists in finding the most probable joint assignment, also known as \emph{energy minimization} (which is NP-Hard in general~\cite{shimony1994finding}):
\begin{equation}\label{eq:map-S}
	\s^* = \argmax_{\s\in\cS^n} p(\s) = \argmin_{\s\in\cS^n} e(\s).
\end{equation}

\subsection{Continuous relaxation of MAP inference}\label{sec:continuous-relaxation}

Let $x_{is}$ be a binary variable such that $x_{is}=1$ iff label $s$ is assigned to node $i$. Then, $\x_i = (x_{is})_{s\in\cS} \in \set{0,1}^{d}$ denotes the one-hot vector for node $i$. Let $\x\in\set{0,1}^{nd}$ be the concatenation of all $\x_i$, and let $\thetab_i = (\theta_{i}(s))_{s\in\cS} \in\RR^d$, $\Thetab_{ij} = (\theta_{ij}(s,t))_{t\in\cS}^{s\in\cS} \in\RR^{d\times d}$. The energy~\eqref{eq:mrf-energy} is then
\begin{equation}\label{eq:energy-xi}
	E(\x;\thetab) = \sum_{i\in\cV}\thetab_i^\top\x_i +  \sum_{ij\in\cE} \x_i^\top\Thetab_{ij}\x_j,
\end{equation}
where $\thetab$ is a parameter vector composed of all $\thetab_i$ and $\Thetab_{ij}$. The MAP inference problem~\eqref{eq:map-S} transforms to minimizing $E(\x;\thetab)$ over the new variables $\x$. A natural approach is to relax the binary constraint and solve the continuous relaxation $\min_{\x\in\cX} E(\x;\thetab)$ with
\begin{equation}\label{eq:set-X}
	\cX = \set{\x \in\RR^{nd}: \x\ge\0, \1^\top\x_i = 1\ \forall i\in\cV}.
\end{equation}
This continuous relaxation is known to be \emph{tight}~\cite{berthod1982definition,lehuu2018continuous,ravikumar2006quadratic}. %
For convenience, we represent the unary potentials as a vector $\u(\thetab)=(\thetab_{i})_{i\in\cV}\in\RR^{nd}$, and the pairwise potentials as a symmetric $n\times n$ block matrix $\P(\thetab) \in\RR^{nd\times nd}$, where the $(i,j)$ block is $\Thetab_{ij}$. The problem is reduced to
\begin{equation}\label{eq:energy-x}
	\min_{\x\in\cX} E(\x;\thetab) \triangleq \frac{1}{2}\x^\top\P(\thetab)\x + \u(\thetab)^\top\x.
\end{equation}
The reason we have made $\thetab$ explicit in~\eqref{eq:energy-xi} and~\eqref{eq:energy-x} is to provide more clarity when discussing the differentiability of their solutions for the \emph{learning} task. Note that an optimal solution $\x^*$ to~\eqref{eq:energy-xi} and~\eqref{eq:energy-x} is a function of $\thetab$, and thus should be written as $\x^*(\thetab)$. Therefore, when we say a solution is differentiable, it is understood that it is so with respect to $\thetab$. When there is no ambiguity, we omit $\thetab$ and write simply $E(\x), \P$, and $\u$. We will be interested in problem \eqref{eq:energy-x} in the rest of the paper, though it should be noted that our method also applies to the so-called linear programming (LP) relaxation, which takes the form $\min_{\x\in \cX_\lp} E_\lp(\x;\thetab)$, where $E_\lp(\x;\thetab) = \thetab^\top\x$ and $\cX_\lp$ is the so-called \emph{local polytope}~\cite{wainwright2005map}. We refer to Appendix~\ref{sec:all-special-cases} for the details.

\subsection{Vanilla Frank-Wolfe algorithm for MAP inference}\label{sec:frank-wolfe-for-map-inference}
Since the continuous energy is differentiable, it is natural to apply first-order methods such as \fw~\cite{frank1956algorithm} to solving~\eqref{eq:energy-x}~\cite{lehuu2018continuous}. Starting from a feasible $\x\iter{0}\in\cX$, \fw approximately solves \eqref{eq:energy-x} by iterating the following steps, where $\alpha_k\in [0,1]$ follows some stepsize scheme:
\begin{equation}
	\p\iter{k} \in \argmin_{\p\in\cX}  \inner{\nabla E(\x\iter{k}), \p},\qquad
	\x\iter{k+1} = \x\iter{k} + \alpha_k(\p\iter{k} - \x\iter{k}).\label{eq:vanilla-Frank-Wolfe}
\end{equation}
The same idea has also been successfully applied to other types of continuous relaxations of~\eqref{eq:energy-x}, such as linear programming (LP)~\cite{meshi2015smooth} or convex quadratic programming (QP)~\cite{desmaison2016efficient,ravikumar2006quadratic}.

\section{Regularized Frank-Wolfe for inference}\label{sec:regularized-frank-wolfe}

In this section, we introduce the proposed regularized Frank-Wolfe as a general class of algorithms for MAP inference. The motivation and general framework are presented in \S\ref{sec:smoothing-perspective} and \S\ref{sec:regularized-energy-perspective}. We show that several new inference algorithms can be obtained using different regularizers (\S\ref{sec:instantiations}). Moreover, regularized Frank-Wolfe also includes a number of existing algorithms as special cases (\S\ref{sec:special-cases}).

\subsection{A smoothing perspective}\label{sec:smoothing-perspective}

We have seen in \S\ref{sec:frank-wolfe-for-map-inference} the (vanilla) Frank-Wolfe method for solving MAP inference. Unfortunately, from a \emph{learning} perspective, this algorithm is problematic. Indeed, CRF learning with stochastic gradient descent (SGD) is typically done by backpropagating through the optimization steps \cite{krahenbuhl2013parameter,schwing2015fully,zheng2015conditional}, but the iterate $\p\iter{k}$ in~\eqref{eq:vanilla-Frank-Wolfe} is piecewise constant and thus its gradient (\wrt $\thetab$) is zero almost everywhere (\S\ref{sec:non-diff-vanilla-FW}), which makes learning not possible.\footnote{\new{Note that backpropagation typically requires the operations to be differentiable (at least) almost everywhere, which $\p\iter{k}$ satisfies. Therefore, the issue here does not lie in \emph{differentiability}, but in the resulting \emph{zero} gradients.} Blackbox differentiation~\cite{pogancic2020differentiation} can deal with this scenario, but it is limited to LP relaxations.} This issue will be illustrated later in the experiments. A potential solution is to add a (typically strongly-)convex regularization term:
\begin{equation}\label{eq:direction-regularized}
	\p\iter{k} \in \argmin_{\p\in\cX} \set{\inner{\nabla E(\x\iter{k}), \p} + r(\p)},\qquad
	\x\iter{k+1} = \x\iter{k} + \alpha_k(\p\iter{k} - \x\iter{k}).
\end{equation}
With appropriately chosen regularizers, \eqref{eq:direction-regularized} becomes suitable for gradient-based learning.

The technique of approximating an update step by a regularized one is quite standard in the optimization literature. For example, the classical proximal gradient method~\cite{lions1979splitting} can be interpreted the same way. Our update~\eqref{eq:direction-regularized} is inspired by Nesterov's smoothing~\cite{nesterov2005smooth}, and thus is similar in spirit to its many applications~\cite{jojic2010accelerated,niculae2017regularized,song2014learning}. In the next section, we present a different, more general perspective to view~\eqref{eq:direction-regularized}, which offers more flexibility in designing new algorithms, in making connections to existing ones, and in analyzing their theoretical properties in a unified manner.

\subsection{A regularized energy perspective}\label{sec:regularized-energy-perspective}

Instead of approximating the local updates of a \emph{given algorithm} (in this case: vanilla Frank-Wolfe) to minimize the \emph{same} objective function, one may choose to keep the algorithm, and approximate instead the objective. At first glance, however, this idea does not seem to be a good one. Indeed, if we replace $E$ by some function $E'$ and apply vanilla Frank-Wolfe, 
the updates~\eqref{eq:vanilla-Frank-Wolfe} remains piecewise constant, thus we have the same zero-gradient issue. It turns out that, if we choose a more appropriate ``given algorithm'', this idea can work. Such a choice is the \emph{generalized Frank-Wolfe} algorithm~\cite{bach2015duality,bredies2007generalized,mine1981minimization}. %

Consider the following problem, where $f:\RR^m\to\RR\cup\set{+\infty}$ is differentiable but possibly nonconvex, and $g:\RR^m\to\RR\cup\set{+\infty}$ is proper, closed, and convex but possibly non-differentiable:
\begin{equation}\label{eq:composite}
	\min_\x F(\x) \triangleq f(\x) + g(\x).
\end{equation}
Generalized Frank-Wolfe solves~\eqref{eq:composite} by iterating
\begin{equation}\label{eq:direction-FW-generalized}
	\p\iter{k} \in \argmin_{\p} \set{ \inner{\nabla f(\x\iter{k}), \p} + g(\p)}, \qquad
	\x\iter{k+1} = \x\iter{k} + \alpha_k(\p\iter{k} - \x\iter{k}).
\end{equation}
If $g$ is the indicator function $\delta_\cX$ of $\cX$ (\ie, $\delta_\cX(\x) = 0$ if $\x\in\cX$ and $\delta_\cX(\x) = +\infty$ otherwise), then the algorithm clearly reduces to vanilla Frank-Wolfe~\eqref{eq:vanilla-Frank-Wolfe} for $\min_{\x\in\cX} f(\x)$. Therefore, generalized Frank-Wolfe applied to~\eqref{eq:energy-x} with $f=E$ and $g=\delta_\cX$ will yield exactly the same updates \eqref{eq:vanilla-Frank-Wolfe}. Now let us apply this algorithm to an approximate objective $E_r(\x) = E(\x) + r(\x)$ for some function $r$. Choosing $f = E$ and $g = r + \delta_\cX$, it is straightforward that \eqref{eq:direction-FW-generalized} reduces to \eqref{eq:direction-regularized}. Therefore, we have recovered the same algorithm as in~\S\ref{sec:smoothing-perspective}, but this time through different machinery.

This framework offers a great flexibility as one can choose $f$ and $g$ in many different ways to obtain new algorithms. The only conditions are $f$ being differentiable and $g$ being convex, so that the subproblem in~\eqref{eq:direction-FW-generalized} is well-defined and globally solvable.\footnote{Further mild conditions are required for convergence (\S\ref{sec:convergence-analysis}).} 
For example, instead of choosing $f = E$ and $g= r+ \delta_\cX$ as above, one can choose $f(\x) = \frac{1}{2}\x^\top\P\x$ and $g(\x) = \u^\top\x + r(\x) + \delta_\cX(\x)$. We will recover later in~\S\ref{sec:special-cases} some existing algorithms (as special cases) through this kind of decomposition. 
Finally, we present Algorithm~\ref{algo:regularized-FW} for (approximately) solving MAP inference~\eqref{eq:energy-x}.

\begin{algorithm}
	\caption{\label{algo:regularized-FW}Generic regularized Frank-Wolfe for (approximately) solving MAP inference~\eqref{eq:energy-x}.}
	\begin{algorithmic}[1]
		\State \label{algo:choice}Choose a regularizer $r$ such that there exist $f$ (differentiable) and $g$ (convex) satisfying $f + g = E + r + \delta_\cX$. Typically (but not necessarily) $r$ is convex on $\cX$ and is constant on $\cX\cap\set{0,1}^{nd}$.
		\State Initialization: $k\gets 0$, $\x\iter{0}\in\cX$, number of iterations $N$.
		\State \label{algo:rWF-direction}Compute $\p\iter{k} \in\argmin_{\p} \set{\inner{\nabla f(\x\iter{k}), \p} + g(\p)}$ and compute the stepsize $\alpha_k$.
		\State Update $\x\iter{k+1} = \x\iter{k} + \alpha_k(\p\iter{k} - \x\iter{k})$. Let $k\gets k+1$ and go to Step~\ref{algo:rWF-direction} until $k=N$.
		\State \label{algo:rounding}Rounding: convert $\x$ to a discrete solution and return.
	\end{algorithmic}
\end{algorithm}%

While the choice of $(r, f, g)$ can be highly flexible, it would make little sense to optimize a function that has nothing to do with the original objective (\ie, the discrete energy). Let $\overline{\cX} = \cX\cap\set{0,1}^{nd}$ denote the discrete domain of our problem. If we choose $r$ such that it is constant on $\overline{\cX}$ (as suggested in Step~\ref{algo:choice} above), then minimizing $E$ on $\overline{\cX}$ is equivalent to minimizing $E + r$ on $\overline{\cX}$, and thus Algorithm~\ref{algo:regularized-FW} actually solves the continuous relaxation of a (different) discrete problem that is equivalent to MAP inference. Further discussion on this matter, as well as on the rounding Step~\ref{algo:rounding}, are deferred until \S\ref{sec:tightness-analysis}.

\begin{newtext}
	Finally, we should note that adding a strongly-convex regularizer is not new in the MAP inference literature \cite{hu2018sdca,jojic2010accelerated,meshi2015smooth,sontag2007new,tang2016bethe}. In particular, some previous work even applied (vanilla) Frank-Wolfe to optimizing such regularized energy \cite{meshi2015smooth,sontag2007new,tang2016bethe}. All these algorithms, however, suffer from the aforementioned zero-gradient issue, as already explained in the beginning of this section.
\end{newtext}

\subsection{Particular instantiations }\label{sec:instantiations}

The previous section presents regularized Frank-Wolfe as a general algorithm for inference. We now discuss concrete examples of its instantiations. To the best of our knowledge, \emph{all the algorithms presented in this section are new and have not been studied previously in the MAP inference literature}. \new{In particular, despite some similarities with proximal gradient~\cite{lions1979splitting} and mirror descent~\cite{beck2003mirror,nemirovskij1983problem}, our following euclidean and entropic variants are actually different from these methods.\footnote{\new{In proximal gradient and mirror descent, the current iterate is constrained to stay close to the previous one, while this is not the case in our method. See Appendix~\ref{sec:derivation-peer-methods} for the details.}}}

\paragraph{Euclidean Frank-Wolfe} Perhaps the most natural choice is $\ell_2$ regularization. In Algorithm~\ref{algo:regularized-FW}, let us choose $f(\x) = E(\x)$ and $r(\x) = \frac{\lambda}{2}\norm{\x}_2^2$, where $\lambda> 0$ is a regularization weight. Let $\Pi_\cX(\v)$ be the projection of a vector $\v$ onto $\cX$. It can be shown (\S\ref{sec:lfw-details}) that Step~\ref{algo:rWF-direction} in Algorithm~\ref{algo:regularized-FW} becomes
\begin{equation}\label{eq:update-l2-FW}
	\p\iter{k}  = \argmin_{\p\in\cX} \set{\inner{\P\x\iter{k} + \u, \p} + \frac{\lambda}{2}\norm{\p}_2^2} = \Pi_\cX\left(-\frac{1}{\lambda}(\P\x\iter{k} + \u)\right) \quad\forall k\ge 0.
\end{equation}

\paragraph{Entropic Frank-Wolfe} In Algorithm~\ref{algo:regularized-FW}, let us choose  $f(\x) = E(\x)$ and $r(\x) = -\lambda H(\x)$, where $\lambda > 0$ is a regularization weight and $H(\x) = -\sum_{i\in\cV} \sum_{s\in\cS} x_{is}\log x_{is}$ is the entropy of $\x$ over $\cX$.
It can be shown (\S\ref{sec:efw-details}) that Step~\ref{algo:rWF-direction} in Algorithm~\ref{algo:regularized-FW} becomes
\begin{equation}\label{eq:update-entropy-FW}
	\p\iter{k}  = \argmin_{\p\in\cX} \set{\inner{\P\x\iter{k} + \u, \p} -\lambda H(\p) }
	= \softmax\left(-\frac{1}{\lambda}(\P\x\iter{k} + \u)\right) \quad\forall k\ge 0,
\end{equation}
where $\v = \softmax(\x)$ with $\x\in\RR^{nd}$ means $\v\in\RR^{nd}$ and $v_{is} = \frac{\exp(x_{is})}{\sum_{t\in \cS} \exp(x_{it})} \ \forall i\in\cV,\forall s\in\cS.$
The resulting algorithm has a tight connection with (parallel) mean field~\cite{krahenbuhl2011efficient,krahenbuhl2013parameter} (discussed in \S\ref{sec:special-cases}).

\paragraph{Other variants} One can consider more sophisticated regularizers, \eg, a weighted combination of $\ell_2$ norm and entropy. Other options include the many different regularizers that have been used in diverse machine learning applications, such as $\ell_p$ norm~\cite{niculae2017regularized}, lasso variants~\cite{niculae2017regularized}, or binary entropy~\cite{amos2019limited}. Although these variants also lead to new MAP inference algorithms, their implementations are more sophisticated since their subproblems~\eqref{eq:direction-FW-generalized} require numerical solutions as no closed form ones exist.

\subsection{Recovering existing algorithms as special cases}\label{sec:special-cases}

In addition to the above new algorithms, regularized Frank-Wolfe also includes several existing ones as special cases. We present some of them below and refer to Appendix \ref{sec:all-special-cases} for further details.

\paragraph{Mean field} This is a special case of the above Entropic Frank-Wolfe. Indeed, if we choose $\lambda = 1$ in~\eqref{eq:update-entropy-FW} and a constant stepsize $\alpha_k = 1 \ \forall k\ge 0$ in Algorithm~\ref{algo:regularized-FW}, then it is straightforward that this algorithm is reduced to the following update step, where $\cN_i$ is the set of neighbors of node $i$:
\begin{equation*}\label{eq:mean-field-parallel}
	\resizebox{\linewidth}{!}{$\x\iter{k+1} = \softmax(-\P\x\iter{k} - \u) \iff x_{is}\iter{k+1} = \frac{1}{Z_i}\exp\bigg(-\theta_{i}(s) - \sum_{j\in\cN_i}\sum_{t\in\cS}\theta_{ij}(s,t)x_{jt}\iter{k}\bigg)\ \forall i\in\cV,s\in\cS$.}
\end{equation*}
This is precisely a (parallel) mean field update~\cite{krahenbuhl2011efficient,krahenbuhl2013parameter}. To conclude, parallel mean field is an instance of Entropic Frank-Wolfe with unit regularization weight and unit stepsize. Interestingly, the update~\eqref{eq:update-entropy-FW} is the well-known softmax function with \emph{temperature} in the deep learning literature \cite{hinton2015distilling}. One could have easily come up with such a simple extension of mean field by adding a temperature to softmax (yet surprisingly this has not been tried before), but here we have provided a principled way to achieve that. As shown later in the experiments, with suitable $\lambda$, this extension yields much better results than vanilla mean field. Finally, we should note that the tight connection between mean field and first-order methods has been noticed before. \citet{krahenbuhl2013parameter} proposed several mean-field-type variants based on the concave-convex procedure~\cite{yuille2002concave}, while closely similar variants can also be obtained through proximal gradient \cite{ajanthan2019proximal,baque2016principled}, but unlike our generalized algorithm, these algorithms cannot recover \emph{exactly} the original mean field of \citet{krahenbuhl2011efficient}. 

\paragraph{Concave-convex procedure} CCCP~\cite{yuille2002concave} solves~\eqref{eq:composite}, assuming $f$ is concave and $g$ is convex, by updating $\x\iter{k+1}$ as a solution to $-\nabla f(\x\iter{k}) \in \partial g(\x\iter{k+1})$,\footnote{In the original CCCP~\cite{yuille2002concave}, $g$ is differentiable, thus the update becomes $-\nabla f(\x\iter{k}) = \nabla g(\x\iter{k+1})$.}
which is precisely~\eqref{eq:direction-FW-generalized} with stepsize $\alpha_k = 1$. We conclude that CCCP is a special case of generalized Frank-Wolfe with $f$ concave and unit stepsize. As a result, many existing CCCP-based inference algorithms~\cite{desmaison2016efficient,krahenbuhl2013parameter} can be seen as special cases of regularized Frank-Wolfe. For example, the ones presented by \citet{desmaison2016efficient} are instantiations of the proposed algorithm with either $f(\x) = -\x^\top\diag(\c)\x$ and $r(\x) = E(\x) + \x^\top\diag(\c)\x$ (where $\c\in\RR^{nd}$ is large enough so that $r(\x)$ is convex), or $f(\x) = \x^\top (\P - \C)\x$ and $r(\x) = \u^\top\x + \x^\top\C\x$ (where $\C$ is some matrix such that $f$ is concave and $r$ is convex). Note that in these instantiations, Step~\ref{algo:rWF-direction} in Algorithm~\ref{algo:regularized-FW} requires an iterative (numerical) solution. Finally, all the algorithms presented by \citet{krahenbuhl2013parameter} are also instantiations of the proposed method because they are based on CCCP. We refer to Appendix~\ref{sec:all-special-cases} for further details.

\paragraph{Vanilla Frank-Wolfe} This is trivially a special case of regularized Frank-Wolfe and we briefly discuss it for completeness. Choosing $f(\x) = E(\x)$ and $r(\x) = 0$ we obtain the algorithm by~\citet{lehuu2018continuous}. Likewise, the one by \citet{desmaison2016efficient} corresponds to $f(\x) = E(\x) - \c^\top\x + \x^\top\diag(\c)\x$ and $r(\x) = 0$, where $\c\in\RR^{nd}$ is large enough for $f$ to be convex. \new{In addition, we can also recover existing LP-based algorithms by choosing $\cX = \cX_\lp, r(\x) = 0$, and $f(\x) = E_\lp(\x) + R(\x)$ with suitable $R(\x)$. Indeed, the one by \citet{meshi2015smooth} takes $R(\x)$ as the squared $\ell_2$-norm of linear constraints, while the ones by \citet{sontag2007new} and \citet{tang2016bethe} correspond to $R(\x)$ being an entropy approximation and its generalization, respectively (see \S\ref{sec:all-special-cases}).}

\section{Theoretical analysis}\label{sec:theoretical-analysis}
\subsection{Convergence}\label{sec:convergence-analysis}

We provide a convergence analysis for the generalized Frank-Wolfe algorithm, and the results for CRF inference special cases will then follow as a consequence. Convergence of vanilla Frank-Wolfe has been well studied in the literature \cite{freund2016new,jaggi2013revisiting,lacoste2015global,lan2013complexity}. For generalized Frank-Wolfe, different analyses exist for the case where both $f$ and $g$ in~\eqref{eq:composite} are convex~\cite{bach2015duality,harchaoui2015conditional,mairal2013optimization,yu2017generalized}. We are particularly interested in the general case where $f$ is \textbf{nonconvex},\footnote{Note that $g$ is still assumed to be convex, so that the subproblem~\eqref{eq:direction-FW-generalized} can be solved to global optimality.} as the CRF energy is often highly so in practice. \citet{mine1981minimization} (and subsequently \citet{bredies2007generalized}) proved the global convergence of the algorithm under mild conditions, though no rate of convergence was given. Recently, \citet{beck2017first} obtained an $\cO(1/\sqrt{k})$ rate of convergence for convex $g$ under adaptive or line-search stepsizes. We extend their analysis with several contributions. First, we include the case where $g$ is strongly convex, which is important as our main variants for inference (\eg, mean field or $\ell_2$-Frank-Wolfe) use strongly-convex regularizers. Second, to also include CCCP~\cite{yuille2002concave} as a special case, we relax their Lipschitz smoothness assumption on $f$ to semi-concavity (which is weaker, as any $L$-smooth function is also $L$-concave). Third, we also consider much weaker stepsize schemes such as \emph{constant} or \emph{non-summable} ones. We show that for either concave $f$ or strongly-convex $g$, a better $\cO(1/k)$ rate of convergence can be achieved, even under the (weak) constant stepsize. It should be noted that our results are new.

All the results in this section are stated under the following assumptions, where $L_f$ and $\sigma_g$ are non-negative constants and $\norm{\cdot}$ denotes the $\ell_2$ norm. Their proofs are given in Appendix~\ref{sec:convergence-analysis-full}. %
\begin{assumption}\label{assumption-f}
	$f$ is differentiable and $L_f$-semi-concave (\ie, $f(\x) - \frac{L_f}{2}\norm{\x}^2$ is concave) on $\dom f$, which is assumed to be open and convex. When $L_f=0$, $f$ is concave.
\end{assumption}
\begin{assumption}\label{assumption-g}
	$g$ is proper, closed, and $\sigma_g$-strongly-convex (\ie, $g(\x) - \frac{\sigma_g}{2}\norm{\x}^2$ is convex), and $\dom g \subseteq \dom f$ is compact. When $\sigma_g > 0$, $g$ is strongly convex.
\end{assumption}
Let $\p_\x$ denote a solution of $\min_{\p} \set{\inner{\nabla f(\x), \p} + g(\p)}$ and let $\p\iter{k} = \p_{\x\iter{k}}$.
The following quantity, called the \emph{conditional gradient norm}~\cite{beck2017first}, will serve as an optimality measure:
\begin{equation}\label{eq:conditional-gradient-norm}
	S(\x) = \inner{\nabla f(\x), \x - \p_\x} + g(\x) - g(\p_\x).
\end{equation}
\begin{lemma}\label{lemma:stationarity}
	$S(\x) \ge \frac{\sigma_g}{2}\norm{\x - \p_\x}^2 \ \forall \x\in\dom f$, and $S(\x) = 0$ iff $\x$ is a stationary point of~\eqref{eq:composite}.
\end{lemma}
The following theorem contains our convergence results for the most common stepsize schemes, including the following \emph{adaptive} and \emph{line-search} stepsizes, respectively:
\begin{equation}
	\alpha_k = \min\left\{1, \frac{1}{L_f + \sigma_g}\left(\frac{S(\x\iter{k})}{\norm{\p\iter{k} - \x\iter{k}}^2} + \frac{\sigma_g}{2}\right) \right\},\qquad
	\alpha_k = \argmin_{\alpha\in [0,1]} F(\x\iter{k} + \alpha(\p\iter{k} - \x\iter{k})).\label{eq:stepsize-adaptive-linesearch}
\end{equation}

\begin{theorem}\label{theorem:convergence}
	Let $F^*$ be the minimum value of $F$, $\Omega$ be the diameter of $\dom g$, $\Delta_k = F(\x\iter{k}) - F^*$, $\omega = \frac{\sigma_g}{L_f+\sigma_g}$, $\rho(\alpha) = \alpha\min\left\{1, 2 - \frac{\alpha}{\omega}\right\}$, $\eta(\alpha) = \frac{1}{2}\left[(L_f+\sigma_g)\alpha - \sigma_g \right]$, and $\mu = \sqrt{2L_f\Delta_0}$. For any $k\ge 0$, we have $\min_{0\le i\le k} S(\x\iter{i}) \le B_k,$
	where the bound $B_k$ is given as follows:
	
	{\normalfont\small
		\resizebox{\textwidth}{!}{
			\begin{tabular}{l|c|c|c|c}
				\hline
				& constant stepsize & constant step length &  non-summable & \multirow{2}{*}{\makecell[c]{adaptive or \\ line search \eqref{eq:stepsize-adaptive-linesearch}}} \\
				& $\alpha_k = \alpha > 0 \ \forall k$ & $\alpha_k = \frac{\alpha}{\norm{\p\iter{k} - \x\iter{k}}} \ \forall k$ & $\sum_{k=0}^{+\infty} \alpha_k = \infty$ &  \\\hline
				convex $g$	& \cellcolor[gray]{0.9} $\frac{\Delta_0}{\alpha(k+1)} + \frac{L_f\Omega^2\alpha}{2}$& \cellcolor[gray]{0.9} $\frac{\Delta_0\Omega}{\alpha(k+1)} + \frac{L_f\Omega\alpha}{2}$ & $\frac{\Delta_0 + \frac{L_f\Omega^2}{2}\sum_{i=0}^k \alpha_i^2}{\sum_{i=0}^k \alpha_i}$ & ${\scriptstyle\max}\big( \frac{2\Delta_0}{k+1} , \frac{\mu\Omega}{\sqrt{k+1}} \big)$ \\\hline
				\multirow{2}{*}{\makecell[l]{strongly\\ convex $g$}}
				& \cellcolor[gray]{0.9} $\frac{\Delta_0}{\alpha(k+1)} + $ $\scriptstyle\eta(\alpha)\Omega^2 \ \forall \alpha\ge 2\omega$
				& \cellcolor[gray]{0.9} & & \\
				& $\frac{\Delta_0}{\rho(\alpha)(k+1)} \quad$ $\scriptstyle\forall \alpha < 2\omega$ & \cellcolor[gray]{0.9} \multirow{-2}{*}{$\Big(\frac{\Delta_0}{\alpha\sqrt{2\sigma_g}(k+1)} + \frac{(L_f + \sigma_g)\alpha}{2\sqrt{2\sigma_g}}\Big)^2$}& \multirow{-2}{*}{$\frac{\Delta_{k(\omega)}}{\sum_{i=k(\omega)}^k \alpha_i}$} & \multirow{-2}{*}{$\frac{\Delta_0}{\omega(k+1)}$} \\ \hline %
				concave $f$ & $\frac{\Delta_0}{\alpha(k+1)}$ & $\frac{\Delta_0\Omega}{\alpha(k+1)}$ & $\frac{\Delta_0}{\sum_{i=0}^k \alpha_i}$ & $\frac{\Delta_0}{k+1}$\\\hline
			\end{tabular}
		}
	}
	
	In the above, $k(\omega) = \min\set{k: \alpha_i < 2\omega \ \forall i\ge k}$, with further assumption that $\lim_{k\to\infty}\alpha_k= 0$ for (jointly) non-concave $f$ and non-summable $\alpha_k$. For the non-highlighted cases, we have $\lim_{k\to\infty}S(\x\iter{k}) = 0$ and any limit point of the sequence $(\x\iter{k})_{k\ge 0}$ is a stationary point of~\eqref{eq:composite}. %
\end{theorem}
\new{The table in Theorem~\ref{theorem:convergence} also provides rates of convergence for the algorithm. Prior to our work, the $\cO(1/\sqrt{k})$ rate for the adaptive or line-search stepsizes (top-right cell of the table, due to \citet{beck2017first}) was the best for \emph{nonconvex} objectives.}\footnote{\new{If both $f$ and $g$ are \emph{convex}, a better rate of $\cO(1/k)$ exists \cite{bach2015duality,beck2017first,yu2017generalized}. In addition, if $g$ is the indicator function of a convex set $\cX$ (\ie, vanilla Frank-Wolfe) and either $f$ or $\cX$ is strongly convex, a linear rate can be obtained~\cite{kerdreux2021affine,locatello2017unified,pedregosa2020linearly}. Note that we use quite different machinery from all these analyses, due to the nonconvexity.}} We have improved this rate to $\cO(1/k)$ when $f$ is concave or $g$ is strongly convex, even under weaker stepsize schemes. In particular, convergence is guaranteed for all considered stepsize schemes when $f$ is concave, for which the best bound is obtained when $\alpha_k = 1 \ \forall k$, which explains the default unit stepsize in CCCP~\cite{yuille2002concave} (see \S\ref{sec:special-cases}). Convergence is also guaranteed for the (diminishing) non-summable scheme (which includes common stepsizes such as $\alpha_k = 2/(k+2)$ or $\alpha_k = 1/\sqrt{k}$), but the rate depends on the rate of divergence of $\sum_{i=0}^k\alpha_i$. More detailed results and analyses can be found in Appendix~\ref{sec:convergence-analysis-full}.

\paragraph{Convergence for MAP inference}
 For all the instantiations of regularized Frank-Wolfe presented in \S\ref{sec:instantiations} and \S\ref{sec:special-cases}, it is easy to check that Assumptions~\ref{assumption-f} and~\ref{assumption-g} are satisfied. In addition, the regularizers in most of them (euclidean or entropic variants, including mean field) are strongly convex, thus we would expect a rate of convergence of at least $\cO(1/k)$ in practice for these algorithms under the adaptive, line search, or (suitable) constant stepsizes. \new{Note that the adaptive scheme requires to know $L_f$ and $\sigma_g$, which is possible in our case: a lower bound on $\sigma_g$ is $\lambda$ for both euclidean and entropic variants, while an upper bound on $L_f$ is $\norm{\P}_2$ for the energy~\eqref{eq:energy-x}. In practice, however, these bounds could be too loose to yield good convergence.}

\begin{newtext}
	\paragraph{Convergent mean field} It is well-known that parallel mean field may diverge~\cite{krahenbuhl2013parameter}. Our Entropic Frank-Wolfe can be viewed as an improved variant of mean field that is globally convergent for different stepsize schemes, without resorting to a concave approximation as done by \citet{krahenbuhl2013parameter}. Our above analysis also provides an explanation for a known phenomenon~\cite{baque2016principled}: \emph{damped} mean field (corresponding to Entropic Frank-Wolfe with $\lambda = 1$ and $\alpha_k = \alpha < 1 \ \forall k$) is more likely (than mean field) to guarantee convergence when the energy is not concave. %
\end{newtext}

\subsection{Tightness of the relaxation}\label{sec:tightness-analysis}

We have seen that regularized Frank-Wolfe (Algorithm~\ref{algo:regularized-FW}) minimizes a modified continuous energy. It is thus reasonable to ask whether doing so also minimizes the original discrete energy (which is the main objective). In this section, we partially answer this question by providing some tightness guarantee for this regularized continuous relaxation. Our analysis is quite general and also includes several existing tightness results~\cite{berthod1982definition,lehuu2018continuous,ravikumar2006quadratic} as special cases. All proofs can be found in Appendix~\ref{sec:tightness-full}.

The last step in Algorithm~\ref{algo:regularized-FW} consists in converting $\x$ to a discrete solution. We consider two such rounding schemes. The simplest one is perhaps \textbf{nearest rounding}, which assigns each node $i$ with the label $s_i \in \argmax_{t\in\cS} x_{it}$. 
Intuitively, this sets $\x_i$ to the nearest vertex of the simplex $\set{\x_i\in\RR_+^d: \1^\top\x_i = 1}$. The second scheme, called \textbf{BCD rounding} \cite{lehuu2018continuous,ravikumar2006quadratic}, consists in minimizing $E(\x)$ over $\x_i$ while keeping all $\x_j$ ($j\neq i$) fixed (\ie, block coordinate descent), which amounts to iteratively assigning each node $i$ with label $s_i \in \argmin_{s\in\cS}\Big\{\theta_{i}(s) + \sum_{j\in\cN_i} \sum_{t\in\cS} \theta_{ij}(s,t)x_{jt}\Big\}$.
In practice, we only use nearest rounding because BCD rounding is too expensive for dense graphs. However, an important property of the latter is that it does not increase the energy, which is useful for our theoretical analysis. The following theorem provides an additive bound on the energy.

\begin{theorem}\label{theorem:energy-additive-bound}
	Let $\x_r^*$ be a global minimum of $E_r(\x) = E(\x) + r(\x)$ over $\cX$, $\bar{\x}_r^*$ be the discrete solution rounded from $\x_r^*$, and $E^*$ be the minimum discrete energy. Assume that $r(\x)$ is bounded:\footnote{A sufficient condition is $r$ being continuous, as $\cX$ is compact.} $m\le r(\x)\le M \ \forall \x\in\cX$. We have $E^* \le E(\bar{\x}_r^*) \le E^* + M - m + C$, where $C = \sqrt{n\left(1-\frac{1}{d}\right)}\left(\norm{\u}_2 + \sqrt{n}\norm{\P}_{2}\right)$ for nearest rounding and $C = 0$ for BCD rounding.
\end{theorem}
Let us derive the energy BCD bound for some particular cases (see \S\ref{sec:tightness-full} for details). Obviously with no regularization ($r=0$), we have $M = m = 0$ and thus $E(\bar{\x}_r^*) \le E^* \le E(\bar{\x}_r^*)$, yielding $E(\bar{\x}_r^*) = E^*$, \ie, the relaxation is tight. We have thus recovered a previously known result \cite{berthod1982definition,lehuu2018continuous,ravikumar2006quadratic}. For $r(\x) = -\c^\top\x + \x^\top\diag(\c)\x$ with $\c\ge\0$, we have $M = 0$ and $m=-\frac{1}{4}\1^\top\c$, which recovers exactly the additive bound given by \citet{ravikumar2006quadratic} for the convex QP relaxation. For the $\ell_2$ regularizer $r(\x) = \frac{\lambda}{2}\norm{\x}_2^2$, we have $M = \frac{\lambda n}{2}$ and $m = \frac{\lambda n}{2d}$, thus we obtain a bound of $\frac{\lambda n}{2}\left(1 - \frac{1}{d}\right)$. For the entropy regularizer $r(\x) = -\lambda H(\x)$, we have $M = 0$ and $m = -\lambda n\log d$, thus the bound is $\lambda n\log d$, which is worse than the $\ell_2$ bound for any $d\ge 5$.

Note that the bound provided by Theorem~\ref{theorem:energy-additive-bound} is achieved from a \emph{global} minimum of the regularized relaxation. This can be attained in some cases, \eg, when the energy is submodular or when the (convex) regularizer is large enough to make the objective convex. In the general case, however, the algorithm is only guaranteed to reach a stationary point, and the (theoretical) quality of such point remains unknown. \new{It would be interesting to investigate whether the algorithm can provide an approximation guarantee for some classes of energies (\eg, supermodular ones), similar to some existing algorithms~\cite{bian2019optimal}. These open questions are left for future work.}

\section{Experiments}\label{sec:experiments}

We compare regularized Frank-Wolfe with existing methods on the semantic segmentation task, in terms of both \emph{inference} and \emph{learning} performance. Two variants, namely Euclidean Frank-Wolfe (\textbf{\lfw}) and Entropic Frank-Wolfe (\textbf{\efw}) (\S\ref{sec:instantiations}), will be compared to the following methods: Mean field (\textbf{MF})~\cite{krahenbuhl2011efficient,krahenbuhl2013parameter} (which is our baseline), nonconvex vanilla Frank-Wolfe (\textbf{FW})~\cite{lehuu2018continuous} (\S\ref{sec:frank-wolfe-for-map-inference}), projected gradient descent (\textbf{PGD})~\cite{larsson2017projected,lehuu2018continuous}, fast proximal gradient method (\textbf{PGM})~\cite{beck2009fista,lions1979splitting}, and alternating direction method of multipliers (\textbf{ADMM})~\cite{lehuu2017alternating,lehuu2018continuous}. 
\new{Convex vanilla Frank-Wolfe~\cite{desmaison2016efficient} and (entropic) mirror descent~\cite{beck2003mirror,nemirovskij1983problem} were found to perform poorly in our experiments, and thus excluded from the presentation}. Other methods based on CCCP~\cite{desmaison2016efficient} or LP relaxation~\cite{ajanthan2017efficient,desmaison2016efficient} are also excluded due to their sophisticated implementations. 
For all methods, we set the initial solution to $\x\iter{0} = \softmax(-\u)$, following previous work \cite{krahenbuhl2011efficient}. \new{Further details on implementation, running time, and memory footprint can be found in Appendices~\ref{sec:derivation-peer-methods}--\ref{sec:experimental-setup-detailed}.}

\subsection{Experimental setup} \label{sec:experimental-setup}

Our segmentation model is a standard combination of a CNN and a CRF \cite{krahenbuhl2013parameter,zheng2015conditional} (Appendix~\ref{sec:model-details}). For the CNN part, we consider two strong architectures: DeepLabv3 with ResNet101 backbone \cite{chen2017rethinking}, and DeepLabv3+ with Xception65 backbone \cite{chen2018encoder}. The CRF part is a fully-connected one~\cite{krahenbuhl2011efficient} in which any pair of pixels $(i,j)$ is an edge with potential $\theta_{ij}(s,t) = \mu(s,t)k(\f_i,\f_j)\ \forall s,t\in\cS$, where $\mu:\cS\times\cS\to\RR$ is called label compatibility function, and $k$ is a Gaussian kernel over image features based on pixel coordinates and colors. The setup of our models are similar to \citet{zheng2015conditional}. We use the Potts compatibility function: $\mu(s,t) = w\mathbbm{1}_{[s\neq t]}$ with $w=1$ for the inference experiments in \S\ref{sec:exp:inference}, and also for CRF initialization in the learning experiments in \S\ref{sec:learning-performance}. For all experiments, a fully-trained CNN is needed. We follow closely the published recipes~\cite{chen2017rethinking,chen2018encoder} for this task. We first pretrain DeepLabv3 and DeepLabv3+ on the COCO dataset~\cite{lin2014microsoft} and then finetune them on \voc (\emph{trainaug}) and Cityscapes (\emph{train}) to obtain similar results to previous work \cite{chen2017rethinking,chen2018encoder} (Table~\ref{tab:miou-potts}, CNN column). Finally, we perform experiments on two popular datasets: (augmented) PASCAL VOC~\cite{everingham2010pascal} and Cityscapes~\cite{cordts2016cityscapes}. Further details are given in Appendix~\ref{sec:experimental-setup-detailed}. %

\subsection{Inference performance}\label{sec:exp:inference}

In this section, we compare the performance of regularized Frank-Wolfe against the competing methods in terms of inference. We consider a Potts CRF on top of a CNN, which is the typical setup for using dense CRF in post-processing. Figure~\ref{fig:energy:comparison} shows the \emph{discrete} energy per inference iteration for each method, averaged over the \num{1449} \emph{val} images of PASCAL VOC, using DeepLabv3+. One can observe that Frank-Wolfe variants completely outperform the other methods. In addition, regularized Frank-Wolfe outperforms all the other methods for a large range of $\lambda$, as shown in Figure~\ref{fig:energy:relative}.

	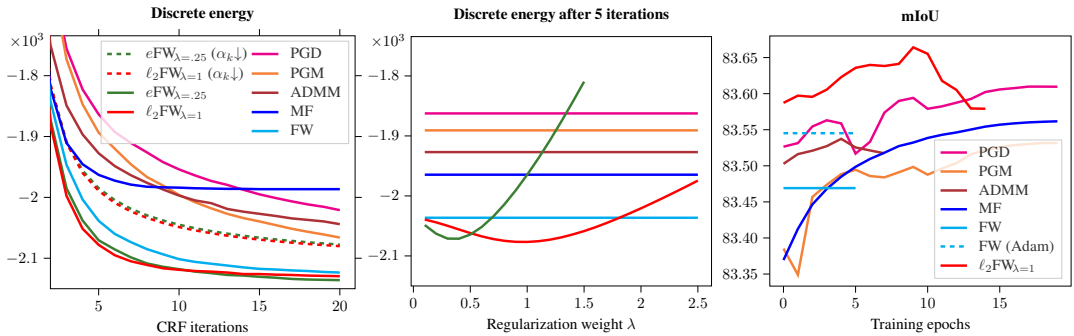
\begin{figure}[!tb]
		\addtolength{\leftskip}{-0.15cm}
		\addtolength{\rightskip}{-0.15cm}
	\begin{subfigure}[b]{0.33\linewidth}%
		\resizebox{1.05\linewidth}{!}{
			\begin{tikzpicture} %

\begin{axis}[
legend cell align={left},
legend style={fill opacity=0.8, draw opacity=1, text opacity=1, draw=white!80!black},
tick align=outside,
tick pos=both,
title={\textbf{Discrete energy}},
x grid style={white!69.0196078431373!black},
xlabel={CRF iterations},
xmin=2, xmax=20.95,
xtick style={color=black},
y grid style={white!69.0196078431373!black},
ymin=-2150.075, ymax=-1734.425,
scaled y ticks={real:1000},
ytick scale label code/.code={$\times 10^3$},
every y tick scale label/.style={at={(yticklabel cs:0.86)},yshift=12.pt,xshift=2pt,anchor=south west},
ytick style={color=black},
legend columns=5,
transpose legend
]

\addplot [ultra thick, efwcolor, dashed]
table {%
	0 0
	1 -1533.00983436853
	2 -1814.93081435473
	3 -1908.11180124224
	4 -1955.88388543823
	5 -1984.89820565908
	6 -2004.38405797101
	7 -2018.63354037267
	8 -2029.44979296066
	9 -2037.93823326432
	10 -2044.92701863354
	11 -2050.65804002761
	12 -2055.43167701863
	13 -2059.54158040028
	14 -2063.09644582471
	15 -2066.24654934438
	16 -2069.02432712215
	17 -2071.51293995859
	18 -2073.72550034507
	19 -2075.70755693582
	20 -2077.52933057281
};
\addlegendentry{\efw[\lambda=.25] ($\alpha_k$$\downarrow$)}

\addplot [ultra thick, lfwcolor, dashed]
table {%
	0 0
	1 -1533.00845410628
	2 -1813.994478951
	3 -1910.52708764665
	4 -1959.86438923395
	5 -1989.2380952381
	6 -2008.96652864044
	7 -2023.09213250518
	8 -2033.70203588682
	9 -2042.04606625259
	10 -2048.81504485852
	11 -2054.27795031056
	12 -2058.89527260179
	13 -2062.92028985507
	14 -2066.38388543823
	15 -2069.39130434783
	16 -2072.05503795721
	17 -2074.41597653554
	18 -2076.53088336784
	19 -2078.44151138716
	20 -2080.12491373361
};
\addlegendentry{\lfw[\lambda=1] ($\alpha_k$$\downarrow$)}

\addplot [ultra thick, efwcolor]
table {%
	1 -1534.5
	2 -1871.5
	3 -1985
	4 -2038
	5 -2070.5
	6 -2085.5
	7 -2098.5
	8 -2108.5
	9 -2114
	10 -2118
	11 -2122
	12 -2124.5
	13 -2127.5
	14 -2128.5
	15 -2130.5
	16 -2132
	17 -2134
	18 -2135
	19 -2135.5
	20 -2136
};
\addlegendentry{\efw[\lambda=.25]}

\addplot [ultra thick, lfwcolor]
table {%
	1 -1534.5
	2 -1880
	3 -1997.5
	4 -2051.5
	5 -2078
	6 -2095
	7 -2105
	8 -2111
	9 -2116.5
	10 -2119
	11 -2121
	12 -2122
	13 -2123.5
	14 -2126
	15 -2126
	16 -2127
	17 -2127.5
	18 -2128.5
	19 -2129
	20 -2129.5
};
\addlegendentry{\lfw[\lambda=1]}

\addlegendimage{empty legend}
\addlegendentry{}

\addplot [ultra thick, pgdcolor]
table {%
	1 -1311.5
	2 -1628
	3 -1755.5
	4 -1822.5
	5 -1863.5
	6 -1892.5
	7 -1910.5
	8 -1926
	9 -1942
	10 -1954
	11 -1964
	12 -1973
	13 -1979.5
	14 -1987.5
	15 -1995
	16 -2001
	17 -2006.5
	18 -2011.5
	19 -2015.5
	20 -2021
};
\addlegendentry{PGD}

\addplot [ultra thick, pgmcolor]
table {%
1 -1311.5
2 -1628
3 -1773
4 -1847
5 -1893.5
6 -1920
7 -1945.5
8 -1966.5
9 -1980.5
10 -1996
11 -2008
12 -2018
13 -2027
14 -2034
15 -2039.5
16 -2045.5
17 -2051.5
18 -2056.5
19 -2061
20 -2066
};
\addlegendentry{PGM}

\addplot [ultra thick, admmcolor]
table {%
	1 -1495
	2 -1749.5
	3 -1849
	4 -1896.5
	5 -1927.5
	6 -1948.5
	7 -1963.5
	8 -1976
	9 -1987
	10 -1997
	11 -2005
	12 -2009.5
	13 -2019
	14 -2023.5
	15 -2027
	16 -2030.5
	17 -2034.5
	18 -2037
	19 -2039.5
	20 -2044
};
\addlegendentry{ADMM}

\addplot [ultra thick, mfcolor]
table {%
1 -1534.5
2 -1820.5
3 -1910.5
4 -1945.5
5 -1963
6 -1972.5
7 -1979
8 -1982.5
9 -1983.5
10 -1984
11 -1985
12 -1985.5
13 -1986
14 -1986
15 -1986.5
16 -1986.5
17 -1986.5
18 -1986.5
19 -1986.5
20 -1986.5
};
\addlegendentry{MF}

\addplot [ultra thick, fwcolor]
table {%
	1 -1534.5
	2 -1841.5
	3 -1946.5
	4 -2003.5
	5 -2038.5
	6 -2060
	7 -2073.5
	8 -2085
	9 -2093.5
	10 -2101.5
	11 -2106
	12 -2109
	13 -2112
	14 -2115
	15 -2117.5
	16 -2118.5
	17 -2120
	18 -2121
	19 -2122.5
	20 -2123.5
};
\addlegendentry{FW}

\end{axis}

\end{tikzpicture}
		}
		\caption{\label{fig:energy:comparison}Discrete energy comparison.}
	\end{subfigure} %
	\begin{subfigure}[b]{0.33\linewidth}
		\resizebox{1.02\linewidth}{!}{
			\begin{tikzpicture}

\begin{axis}[
legend cell align={left},
legend style={fill opacity=0.8, draw opacity=1, text opacity=1, draw=white!80!black},
tick align=outside,
tick pos=left,
title={\textbf{Discrete energy after 5 iterations}},
x grid style={white!69.0196078431373!black},
xlabel={Regularization weight \(\displaystyle \lambda\)},
xmin=-0.02, xmax=2.62,
xtick style={color=black},
ymin=-2150.075, ymax=-1734.425,
scaled y ticks={real:1000},
ytick scale label code/.code={$\times 10^3$},
every y tick scale label/.style={at={(yticklabel cs:0.86)},yshift=12.pt,xshift=2pt,anchor=south west},
ytick style={color=black}
]

\addplot [ultra thick, pgmcolor]
table {%
0.1 -1890.73947550035
0.2 -1890.73947550035
0.3 -1890.73947550035
0.4 -1890.73947550035
0.5 -1890.73947550035
0.6 -1890.73947550035
0.7 -1890.73947550035
0.8 -1890.73947550035
0.9 -1890.73947550035
1 -1890.73947550035
1.1 -1890.73947550035
1.2 -1890.73947550035
1.3 -1890.73947550035
1.4 -1890.73947550035
1.5 -1890.73947550035
1.6 -1890.73947550035
1.7 -1890.73947550035
1.8 -1890.73947550035
1.9 -1890.73947550035
2 -1890.73947550035
2.1 -1890.73947550035
2.2 -1890.73947550035
2.3 -1890.73947550035
2.4 -1890.73947550035
2.5 -1890.73947550035
};

\addplot [ultra thick, mfcolor]
table {%
0.1 -1964.62525879917
0.2 -1964.62525879917
0.3 -1964.62525879917
0.4 -1964.62525879917
0.5 -1964.62525879917
0.6 -1964.62525879917
0.7 -1964.62525879917
0.8 -1964.62525879917
0.9 -1964.62525879917
1 -1964.62525879917
1.1 -1964.62525879917
1.2 -1964.62525879917
1.3 -1964.62525879917
1.4 -1964.62525879917
1.5 -1964.62525879917
1.6 -1964.62525879917
1.7 -1964.62525879917
1.8 -1964.62525879917
1.9 -1964.62525879917
2 -1964.62525879917
2.1 -1964.62525879917
2.2 -1964.62525879917
2.3 -1964.62525879917
2.4 -1964.62525879917
2.5 -1964.62525879917
};

\addplot [ultra thick, admmcolor]
table {%
0.1 -1926.95358868185
0.2 -1926.95358868185
0.3 -1926.95358868185
0.4 -1926.95358868185
0.5 -1926.95358868185
0.6 -1926.95358868185
0.7 -1926.95358868185
0.8 -1926.95358868185
0.9 -1926.95358868185
1 -1926.95358868185
1.1 -1926.95358868185
1.2 -1926.95358868185
1.3 -1926.95358868185
1.4 -1926.95358868185
1.5 -1926.95358868185
1.6 -1926.95358868185
1.7 -1926.95358868185
1.8 -1926.95358868185
1.9 -1926.95358868185
2 -1926.95358868185
2.1 -1926.95358868185
2.2 -1926.95358868185
2.3 -1926.95358868185
2.4 -1926.95358868185
2.5 -1926.95358868185
};

\addplot [ultra thick, pgdcolor]
table {%
0.1 -1862.48723257419
0.2 -1862.48723257419
0.3 -1862.48723257419
0.4 -1862.48723257419
0.5 -1862.48723257419
0.6 -1862.48723257419
0.7 -1862.48723257419
0.8 -1862.48723257419
0.9 -1862.48723257419
1 -1862.48723257419
1.1 -1862.48723257419
1.2 -1862.48723257419
1.3 -1862.48723257419
1.4 -1862.48723257419
1.5 -1862.48723257419
1.6 -1862.48723257419
1.7 -1862.48723257419
1.8 -1862.48723257419
1.9 -1862.48723257419
2 -1862.48723257419
2.1 -1862.48723257419
2.2 -1862.48723257419
2.3 -1862.48723257419
2.4 -1862.48723257419
2.5 -1862.48723257419
};

\addplot [ultra thick, fwcolor]
table {%
0.1 -2036.38008971705
0.2 -2036.38008971705
0.3 -2036.38008971705
0.4 -2036.38008971705
0.5 -2036.38008971705
0.6 -2036.38008971705
0.7 -2036.38008971705
0.8 -2036.38008971705
0.9 -2036.38008971705
1 -2036.38008971705
1.1 -2036.38008971705
1.2 -2036.38008971705
1.3 -2036.38008971705
1.4 -2036.38008971705
1.5 -2036.38008971705
1.6 -2036.38008971705
1.7 -2036.38008971705
1.8 -2036.38008971705
1.9 -2036.38008971705
2 -2036.38008971705
2.1 -2036.38008971705
2.2 -2036.38008971705
2.3 -2036.38008971705
2.4 -2036.38008971705
2.5 -2036.38008971705
};

\addplot [ultra thick, lfwcolor]
table {%
	0.1 -2039.2922705314
	0.2 -2044.06763285024
	0.3 -2049.78071083506
	0.4 -2055.88440303658
	0.5 -2061.90648723257
	0.6 -2067.36870255349
	0.7 -2071.81193926846
	0.8 -2074.8630089717
	0.9 -2076.52777777778
	1 -2076.6856452726
	1.1 -2075.47860593513
	1.2 -2072.9218426501
	1.3 -2069.11473429952
	1.4 -2064.23964803313
	1.5 -2058.61093857833
	1.6 -2052.14406487233
	1.7 -2044.97998619738
	1.8 -2037.25138026225
	1.9 -2029.11059351277
	2 -2020.68668046929
	2.1 -2011.8897515528
	2.2 -2002.85058661146
	2.3 -1993.56452726018
	2.4 -1984.13854382333
	2.5 -1974.4641131815
};

\addplot [ultra thick, efwcolor]
table {%
	0.1 -2048.60334713596
	0.2 -2062.93995859213
	0.3 -2071.13940648723
	0.4 -2071.32125603865
	0.5 -2064.39734299517
	0.6 -2051.82522429262
	0.7 -2034.71825396825
	0.8 -2014.00189786059
	0.9 -1990.58229813665
	1 -1964.6231884058
	1.1 -1936.92908902692
	1.2 -1907.31331953071
	1.3 -1875.9135610766
	1.4 -1843.2472394755
	1.5 -1809.56142167012
};
\end{axis}

\end{tikzpicture}
		}
		\caption{\label{fig:energy:relative} Effect of regularization weight.}
	\end{subfigure}%
	\begin{subfigure}[b]{0.33\linewidth}
		\resizebox{1.03\linewidth}{!}{
			\begin{tikzpicture}

\begin{axis}[
legend cell align={left},
legend style={fill opacity=0.8, draw opacity=1, text opacity=1, draw=white!80!black, , at={(0.97,0.03)}, anchor=south east},
tick align=outside,
tick pos=left,
title={\textbf{mIoU}},
x grid style={white!69.0196078431373!black},
xlabel={Training epochs},
xmin=-0.95, xmax=19.95,
xtick style={color=black},
y grid style={white!69.0196078431373!black},
ymin=83.3332800865173, ymax=83.6799359321594,
ytick style={color=black},
ytick={83.3,83.35,83.4,83.45,83.5,83.55,83.6,83.65,83.7},
yticklabels={$83.30$,$83.35$,$83.40$,$83.45$,$83.50$,$83.55$,$83.60$,$83.65$,$83.70$}
]

\addplot [ultra thick, pgdcolor]
table {%
0 83.526474237442
1 83.5314452648163
2 83.5544407367706
3 83.5630893707275
4 83.5588037967682
5 83.5169017314911
6 83.5331857204437
7 83.5736811161041
8 83.5898995399475
9 83.5941135883331
10 83.5790276527405
11 83.5824251174927
12 83.587509393692
13 83.5926413536072
14 83.6022615432739
15 83.605819940567
16 83.6076676845551
17 83.6097478866577
18 83.6097359657288
19 83.6096107959747
};
\addlegendentry{PGD}

\addplot [ultra thick, pgmcolor]
table {%
	0 83.3853244781494
	1 83.3490371704102
	2 83.4567785263062
	3 83.4737002849579
	4 83.4884464740753
	5 83.4945321083069
	6 83.4857702255249
	7 83.4839701652527
	8 83.4910094738007
	9 83.4984540939331
	10 83.4877550601959
	11 83.4961652755737
	12 83.5038244724274
	13 83.5159122943878
	14 83.5219740867615
	15 83.5258424282074
	16 83.5286557674408
	17 83.5301756858826
	18 83.5315704345703
	19 83.5314691066742
};
\addlegendentry{PGM}

\addplot [ultra thick, admmcolor]
table {%
	0 83.5027456283569
	1 83.5162043571472
	2 83.5209369659424
	3 83.5272371768951
	4 83.5373759269714
	5 83.525425195694
	6 83.5213184356689
	7 83.5173368453979
};
\addlegendentry{ADMM}

\addplot [ultra thick, mfcolor]
table {%
0 83.3693563938141
1 83.4126234054565
2 83.4463596343994
3 83.4679543972015
4 83.4848701953888
5 83.498603105545
6 83.5095286369324
7 83.5181295871735
8 83.5273206233978
9 83.5322439670563
10 83.538556098938
11 83.5429787635803
12 83.5462391376495
13 83.5502803325653
14 83.5543155670166
15 83.5569858551025
16 83.5588097572327
17 83.5601806640625
18 83.5610330104828
19 83.5616827011108
};
\addlegendentry{MF}

\addplot [ultra thick, fwcolor]
table {%
	0 83.4690988063812
	1 83.4690988063812
	2 83.4690988063812
	3 83.4690988063812
	4 83.4690988063812
	5 83.4690988063812
};
\addlegendentry{FW}

\addplot [ultra thick, dashed, fwcolor]
table {%
	0 83.5451781749725
	1 83.5451781749725
	2 83.5451781749725
	3 83.5451781749725
	4 83.5451781749725
	5 83.5451781749725
};
\addlegendentry{FW (Adam)}

\addplot [ultra thick, lfwcolor]
table {%
0 83.587247133255
1 83.5972189903259
2 83.5955679416656
3 83.6056053638458
4 83.6228549480438
5 83.6357653141022
6 83.6398005485535
7 83.6383104324341
8 83.6412370204926
9 83.6641788482666
10 83.655196428299
11 83.6177051067352
12 83.6051225662231
13 83.5794627666473
14 83.5790932178497
};
\addlegendentry{\lfw[\lambda=1]}

\end{axis}

\end{tikzpicture}
		}
		\caption{\label{fig:mIoU-val} Learning failure of vanilla FW.}
	\end{subfigure}%
	\caption{\label{fig:energy}\new{\textbf{Results on PASCAL VOC validation set} using DeepLabv3+ and Potts dense CRF. \textbf{(a)} Comparison between CRF solvers ($\alpha_k$$\downarrow$ means $\alpha_k = k/(k+2) \ \forall k$) shows that Frank-Wolfe variants clearly outperform the other methods in terms of energy minimization. \textbf{(b)} Performance of regularized Frank-Wolfe can be greatly affected by $\lambda$, but it can still achieve lower energies than the other methods for a large range of $\lambda$. \textbf{(c)} Vanilla Frank-Wolfe completely fails to learn because of the zero-gradient issue.}}
\end{figure}

Table~\ref{tab:miou-potts} shows the performance on the validation sets of \voc and Cityscapes, for a Potts CRF with both DeepLabv3 and DeepLabv3+ as backbone. In this experiment, we run all the methods for \num{10} iterations. One can observe that \lfw achieved the best performance, followed by \efw. 

\begin{wraptable}{r}{0.6\textwidth}
	\resizebox{\linewidth}{!}{
		\begin{tabular}{llcccccccccccc}
			\toprule
			& &	CNN & PGD & PGM & ADMM & MF & FW & \efw[.7] & \efw[.3] & \lfw\\ \toprule
			\parbox[t]{2mm}{\multirow{2}{*}{\rotatebox[origin=c]{90}{VOC}}}
			& DL3 & $81.83$ & $82.23$& $82.23$ & $82.22$ & $82.21$ & $82.27$ & $82.26$ & $\mathbf{82.29}$ & $\mathbf{82.29}$ \\\cmidrule{2-11}
			& DL3+ & $82.89$ & $83.36$ & $83.37$ & $83.38$ & $83.45$ &$83.43$ & $83.45$ & $83.48$  & $\mathbf{83.50}$\\\toprule
			\parbox[t]{2mm}{\multirow{2}{*}{\rotatebox[origin=c]{90}{CITY}}} & DL3 & $76.73$& $76.88$ & $76.86$ & $76.95$ & $76.97$ & $76.86$ & $76.99$ & $76.99$ & $\mathbf{77.03}$\\\cmidrule{2-11}
			& DL3+ & $79.55$& $79.64$ & $79.63$ & $\mathbf{79.66}$ & $79.63$ & $79.64$ & $79.65$ & $\mathbf{79.66}$ &$\mathbf{79.66}$\\
			\bottomrule
		\end{tabular}
	}
	\caption{\label{tab:miou-potts}\textbf{Validation mIoU using a Potts CRF} on top of the pre-trained CNN models. DL means DeepLab. %
	} \vspace{-0.2cm}
\end{wraptable}

We should note some inconsistency compared to the energy results previously presented in Figure~\ref{fig:energy:comparison}. For example, \efw achieved much lower energy than MF, yet the mIoU gap is marginal; also, FW accuracy is slightly worse than MF while the energy is much better (lower). This can be explained by the fact that the Potts model is not a perfect representation (\ie, lower energy in this model does not necessarily translate to higher accuracy). In the next section, we will see how the methods perform when the CRF parameters are learned from data.

\subsection{Learning performance}\label{sec:learning-performance}

In this section, we evaluate the performance of the methods for joint CNN-CRF end-to-end training. The CNN is initialized with its fully-trained weights on the corresponding dataset, and the CRF is initialized with the Potts model with random noise added. We train the model for $20$ epochs with $5$ CRF iterations,\footnote{\new{While we use the same number of iterations at \emph{test time} to simplify the evaluation protocol, it should be noted that using more iterations could be beneficial. See Appendix~\ref{sec:results-fine-grained} for some results.}} using the same poly schedule as before. As the CNN has been already fully-trained, we set its learning rate to a small value of $0.0001$. For the CRF, we tried \num{4} different values of initial learning rates $\eta_0\in\set{1.0, 0.1, 0.01, 0.001}$ and found that \num{1.0} is too high (training diverges quickly) while \num{0.001} is too low (slow progress) for all methods. For the remaining candidates $\set{0.1, 0.01}$, we perform $4$ additional trainings for each method (\ie, a total of $5$ runs for each configuration).

Let us summarize our findings. First, we observe that (vanilla) FW fails to learn.
This is illustrated in Figure~\ref{fig:mIoU-val}, where we show the validation accuracy per epoch on PASCAL VOC for each method: FW did not make any progress. We tried a different optimizer (Adam~\cite{kingma2014adam}) and obtained similar results. This is expected as the gradient in vanilla FW is zero almost everywhere, as previously discussed in \S\ref{sec:smoothing-perspective} (see also \S\ref{sec:non-diff-vanilla-FW}). Our second observation is that training is quite unstable for PGD, PGM, ADMM, \efw[.3], and \lfw. In particular, $\eta_0 = 0.1$ is still too high for these methods, and even with $\eta_0 = 0.01$, some of the runs produced bad results. By contrast, MF and \efw[.7] are stable for both learning rates, with $0.1$ being slightly better. A possible explanation is that PGD, PGM, ADMM, and \lfw all employ a simplex projection step that is not fully differentiable (but only so almost everywhere). For \efw[.3] (which is fully differentiable), we hypothesize that the low regularization makes the problem less ``smooth'', which may also harm gradient-based training. Finally, with the above training scheme, we observe that none of the CRF methods could improve over the CNN (but rather the opposite) on Cityscapes. We have seen that the Potts CRF was able to achieve some marginal improvements (Table~\ref{tab:miou-potts}), thus it is reasonable to expect even better performance with end-to-end training.

In view of the above observations, we present a simple trick to make CRF training more stable. The idea is to replace the CRF output $\x^*$ with $\frac{1}{2} (\x^* + \x\iter{0})$, where we recall that the initialization $\x\iter{0}$ is the $\softmax$ of the CNN logits. Intuitively, this adds a skip connection from the CNN to the CRF output in the computation graph, which makes the gradient of the loss propagate directly to the CNN. We found that this trick also slightly improves \efw[.7], but has a negative effect on MF. Therefore, it is applied to all methods except MF. Finally, as Cityscapes requires a very high number of epochs, we set this value to $100$. Also because training on Cityscapes requires a lot more computing resources, we only perform a single run on DeepLabv3+. The results are presented in Table~\ref{tab:miou-joint-learning}.

\begin{wraptable}{r}{0.6\textwidth}
	\resizebox{\linewidth}{!}{
		\begin{tabular}{llcccccccccc}
			\toprule
			& &	CNN & PGD & PGM & ADMM & MF & \efw[.7] & \efw[.3] & \lfw\\ \toprule
			\multirow{2}{*}{VOC} & DL3 & $81.83$ & \meanstd{83.69}{0.20} & \meanstd{\mathbf{83.75}}{0.23} & \meanstd{83.68}{0.06} & \meanstd{83.69}{0.10} & \meanstd{83.50}{0.10} & \meanstd{83.25}{0.20} & \meanstd{\mathbf{83.75}}{0.13} \\\cmidrule{2-10}
			& DL3+ & $82.89$ & \meanstd{84.82}{0.23} & \meanstd{84.79}{0.20} & \meanstd{84.83}{0.06} & \meanstd{84.87}{0.17} & \meanstd{84.64}{0.23} & \meanstd{84.50}{0.16} & \meanstd{\mathbf{85.14}}{0.09}\\\toprule
			CITY & DL3+ & $79.55$ & $79.80$ & $79.62$ & $79.62$ & $79.74$ & $79.70$ & $79.58$ & $\mathbf{79.95}$\\
			\bottomrule\\
		\end{tabular}
	} \vspace{-12pt}
	\caption{\label{tab:miou-joint-learning}\textbf{Validation mIoU under joint training}. For PASCAL VOC, we report the mean and standard deviation from $5$ runs. %
	}
\end{wraptable}

Again, \lfw consistently achieves the best results. Interestingly, while \efw[.3] achieved similar performance to \lfw in terms of energy minimization (Figure~\ref{fig:energy:comparison} and Table~\ref{tab:miou-potts}), its performance is worse in joint training. Compared to Table~\ref{tab:miou-potts}, we see that joint training produced much larger improvements over the CNNs, up to $2.25\%$ on PASCAL VOC and $0.4\%$ on Cityscapes. %

\begin{wraptable}{r}{0.4\textwidth}
	\centering
	\resizebox{1.0\linewidth}{!}{
		\begin{tabular}{l|c|c}
			\toprule
			Model & VOC & CITY\\ \midrule
			DeepLabv3+ \cite{chen2018encoder} & $87.8$ & $82.1$\\
			DeepLabv3+ (this work) & $87.6$ & $83.5$\\
			DeepLabv3+ with \lfw CRF & $\mathbf{88.0}$  & $\mathbf{83.6}$\\
			\bottomrule
		\end{tabular}
	}
	\caption{\label{tab:miou-test-sets}Performance on the \emph{test} sets. Submission URLs are given in Appendix~\ref{sec:submission-urls}.
	}
\end{wraptable}

\paragraph{Performance on the test sets} We select the best performing method (DeepLabv3+ with \lfw CRF) for evaluation on the test sets. For PASCAL VOC, we further train our model on the union of the \emph{train} and \emph{val} subsets for \num{50} epochs. For Cityscapes, we further train $200$ epochs on \emph{train} and \emph{train\_extra}, using the high-quality annotations provided by \citet{tao2020hierarchical} (for \emph{train\_extra}). At the $150\textsuperscript{th}$ epoch, we replace \emph{train\_extra} with \emph{val}. For this fine-tuning step, learning rates were set to $0.001$ for CNN and $0.1$ for CRF. For prediction, we apply test time augmentation including left-right flipping and multi-scales. For reference, we train DeepLabv3+ alone using the same recipes. Table~\ref{tab:miou-test-sets} shows that we were able to closely match the performance reported by~\citet{chen2018encoder}. Adding the \lfw CRF yields an improvement of $0.4$ points on PASCAL VOC. Unfortunately we only observe a marginal improvement on Cityscapes. \vspace{-3pt}

\begin{newtext}
	\subsection{Ablation studies}
	\paragraph{Trainable $\alpha_k$ and $\lambda$} It is possible to learn $\alpha_k$ and $\lambda$ from data by simply setting them to be \emph{trainable}. We carried out such an experiment with \lfw and \efw but did not observe significant improvements, though we should note that a more sophisticated training recipe (\eg, using custom learning rates for these variables) might lead to better results. Details are provided in Appendix~\ref{sec:results-trainable-alpha}. %
	
	\paragraph{Fine-grained analysis} We observe that CRF improved over CNN on most of the semantic classes. In particular, on \texttt{bicycle} (known to be the most challenging class of PASCAL VOC~\cite{chen2017rethinking}), \lfw and \efw achieved improvements of over $10\%$ absolute in mIoU. See Appendix~\ref{sec:results-fine-grained} for the details.
\end{newtext}

\section{Discussion \& conclusion}\label{sec:conclusion}

\paragraph{Why does it work?} Theoretically, all the methods in \S\ref{sec:experiments} should reach a stationary point, so how can one be better than another? In fact, Figure~\ref{fig:energy:comparison} only shows that Frank-Wolfe variants work better than the other methods \emph{in the first few iterations}, but not necessarily in a later stage. Indeed, the same conclusion no longer holds after $100$ iterations (see \S\ref{sec:inference-results-full}), \new{but this long regime is not practical because it would lead to vanishing/exploding gradients~\cite{zheng2015conditional} and to potentially prohibitive memory consumption.} %
As to why Frank-Wolfe achieves lower energy in the early stage, we hypothesize that this could be due to the discreteness of its iterates \eqref{eq:vanilla-Frank-Wolfe}. With small $\lambda$, the solution by regularized Frank-Wolfe should be close to the vanilla one, and thus also benefits from this property. It is important to note that the benefit of regularized Frank-Wolfe does not lie in the extra (sometimes small) energy improvement over vanilla Frank-Wolfe, but in its ability to seamlessly solve the zero-gradient issue. %

\paragraph{How to tune $\lambda$?} We found that similar curves to Figure~\ref{fig:energy:relative} can be obtained using a small random subset (\eg, $10$ samples) of the data, which suggests a quick way of tuning $\lambda$ by random subsampling. In practice, this step takes only a few seconds, which is negligible in most training scenarios. %

\paragraph{Limitations} While one variant of regularized Frank-Wolfe (\lfw) consistently achieves the best results, the difference compared to the other methods is sometimes small. In addition, the improvement of dense CRFs over CNNs is marginal on the Cityscapes test set. Nevertheless, we hope the encouraging results on PASCAL VOC could attract interest from the community in CRF research, potentially leading to creative ways of overcoming these limitations.

\begin{newtext}
	\paragraph{Societal impact} Semantic segmentation models can be used in surveillance systems, which might raise potential privacy concerns. Furthermore, the datasets that our models were trained on are known to present strong built-in bias~\cite{torralba2011unbiased}, thus they should be used with caution.
\end{newtext}

\begin{newtext}
\begin{ack}
	This work was supported in part by the ANR grant AVENUE (ANR-18-CE23-0011), and was partly done when the first author was affiliated with Manifold Perception (\href{https://mption.com}{\texttt{mption.com}}). The experiments were performed using HPC resources from GENCI-IDRIS (Grants 2020-AD011011321 and 2020-AD011011881). The authors thank the anonymous reviewers and meta-reviewer for their constructive feedback that helped improve the manuscript.
\end{ack}
\end{newtext}

{
	\small
	\bibliographystyle{plainnat}
	\bibliography{Attention,CRFs,MRFs,Optimization,DeepLearning}
	
}

\newpage
\appendix

\doparttoc %
\faketableofcontents %

\part{Appendices} %

\parttoc

\renewcommand\thesection{\Alph{section}}
\renewcommand\thesubsection{\thesection.\arabic{subsection}}
\setcounter{section}{0}%

\section{Details on special cases of regularized Frank-Wolfe inference}\label{sec:all-special-cases}

We have seen in \S\ref{sec:regularized-frank-wolfe} multiple instantiations of regularized Frank-Wolfe, leading to new algorithms for MAP inference, as well as recovering many existing ones. In this section we provide further details on this matter.

Recall the notation $n=\abs{\cV}, m = \abs{\cE},d=\abs{\cS}$, where $\cV,\cE$ and $\cS$ are the sets of nodes, edges, and labels, respectively.

\subsection{Algorithms based on QP relaxation with vanilla Frank-Wolfe}

\paragraph{Nonconvex vanilla Frank-Wolfe} This algorithm, previously studied by~\citet{lehuu2018continuous}, was already presented in~\S\ref{sec:frank-wolfe-for-map-inference}. It consists in applying vanilla Frank-Wolfe directly to the energy~\eqref{eq:energy-x}.

\paragraph{Convex vanilla Frank-Wolfe} This involves the \emph{convex} QP relaxation of MAP inference introduced by~\citet{ravikumar2006quadratic}. The idea is to add a sufficiently large vector $\c$ to the diagonal of $\P$ to make it positive semidefinite. If $\x\in\set{0,1}^{nd}$ then it is easy to check that $\x^\top\diag(\c)\x = \c^\top\x$ for any $\c\in\RR^{nd}$. Therefore, the (discrete) energy can be written as
\begin{equation}\label{eq:energy-convex}
	E(\x) = \frac{1}{2}\x^\top(2\diag(\c) + \P)\x + (\u-\c)^\top\x.
\end{equation}
It can be shown that the above function is convex if $\c$ is chosen as follows:
\begin{equation}\label{eq:c-convex-QP}
	c_{is} = \frac{1}{2}\sum_{j\in\cN_i}\sum_{t\in\cS}\theta_{ij}(s,t)\quad\forall i\in\cV,s\in\cS,
\end{equation}
where $\cN_i$ denotes the set of neighbors of node $i$. Applying vanilla Frank-Wolfe to minimizing the above convex energy over $\cX$ yields the algorithm presented in section 4 of~\citet{desmaison2016efficient}.

\subsection{Algorithms based on  LP relaxation with vanilla Frank-Wolfe}
Let us first present the LP relaxation of MAP inference. We use the same notation leading to the energy formulation~\eqref{eq:energy-xi}, namely the indicator variables $x_{is}\in\set{0,1}$, the indicator vectors $\x_i\in\set{0,1}^d$, and the potential vectors $\thetab_{i}\in\RR^d$ for all nodes $i\in\cV$ and labels $s\in\cS$. In addition, define for all edges $ij\in\cE$ and pairs of labels $(s,t)\in\cS^2$:
\begin{itemize}
	\item New pairwise indicator variables $x_{ijst} = x_{is}x_{jt}\in\set{0,1}$.
	\item New pairwise indicator vectors $\x_{ij} = (x_{ijst})_{s\in\cS,t\in\cS}\in\set{0,1}^{d^2}$.
	\item New pairwise potential vectors $\thetab_{ij} = (\theta_{ij}(s,t))_{s\in\cS,t\in\cS} \in\RR^{d^2}$, which can be viewed as the flatten version of the potential matrices $\Thetab_{ij}$ in~\eqref{eq:energy-xi}.
\end{itemize}
Then, the energy~\eqref{eq:energy-xi} can be rewritten as a linear function:
\begin{equation}\label{eq:energy-xi-lp}
	E_\lp(\x;\thetab) = \sum_{i\in\cV}\thetab_i^\top\x_i +  \sum_{ij\in\cE} \thetab_{ij}^\top\x_{ij},
\end{equation}
where by slight abuse of notation, we let $\x$ and $\thetab$ again denote the vectors of all variables and parameters, respectively. Note that $\x$ and $\thetab$ are now $(nd+md^2)$-dimensional vectors and not $nd$-dimensional as in~\eqref{eq:energy-xi}. The LP relaxation consists in minimizing $E_\lp$ over the following \emph{local polytope}~\cite{wainwright2005map}:
\begin{equation}\label{eq:set-X-lp}
	\cX_\lp = \set{\x \in\RR^{nd+md^2} \ \middle| \, 
		\begin{aligned}
			\x &\ge\0,\\
			\1^\top\x_i &= 1\ \forall i\in\cV,\\
			\sum_{t\in\cS}x_{ijst} &= x_{is} \ \forall ij\in\cE,\forall s\in\cS,\\
			\sum_{s\in\cS}x_{ijst} &= x_{jt} \ \forall ij\in\cE,\forall t\in\cS.
		\end{aligned}
	}.
\end{equation}
The last two constraints in the above (called \emph{local consistency}) can be written as $\A\x = \0$ for some $(2md)\times (nd+md^2)$ matrix $\A$. We can thus rewrite the LP relaxation compactly as:
\begin{equation}\label{eq:map-LP}
	\min E_\lp(\x;\thetab) \triangleq \thetab^\top\x, \quad \mbox{s.t.}\quad \x\in \cX_\lp \triangleq \set{\x \in\RR_+^{nd+md^2}: \1^\top\x_i = 1\ \forall i\in\cV, \A\x = \0}.
\end{equation}
As presented in \S\ref{sec:special-cases}, \citet{sontag2007new,meshi2015smooth,tang2016bethe} apply vanilla Frank-Wolfe to minimize a regularized LP energy:
\begin{equation}
	\min_{\x\in\cX_\lp} E_\lp(\x;\thetab) + r(\x)
\end{equation}
for some regularizer $r$. These works differ in the choice of $r$.
\paragraph{Local-consistency regularization} Choosing $r(\x) = \frac{\lambda}{2}\norm{\A\x}_2^2$ we obtain the algorithm presented by~\citet{meshi2015smooth} (which corresponds to the primal algorithm in the top-right cell of their Table 1).

\paragraph{Bethe and TRW entropic regularization} \citet{sontag2007new} also apply vanilla Frank-Wolfe to a regularized LP energy (corresponding to Step 3 in their Algorithm 1; note that we consider only the first outer iteration of their algorithm). They consider regularizers of the form
\begin{equation}
	r(\x) = -\tilde{H}(\x),
\end{equation}
where $\tilde{H}(\x)$ is some approximation to the entropy $H(\x)$ of the distribution over $\x$. 

Define the singleton entropy
\begin{equation}
	H(\x_i) = -\sum_{s\in\cS} x_{is}\log x_{is}\quad\forall i\in\cV,
\end{equation}
and the pairwise mutual information
\begin{equation}
	I(\x_{ij}) = \sum_{s\in\cS}\sum_{t\in\cS} x_{ijst}\log\frac{x_{ijst}}{x_{is}x_{jt}} = -H(\x_{ij}) + H(\x_i) + H(\x_j)\quad \forall ij\in\cE.
\end{equation}
The so-called Bethe approximation is defined as:
\begin{equation}
	\tilde{H}_{\mathrm{Bethe}}(\x) = \sum_{i\in\cV}H(\x_i) - \sum_{ij\in\cE}I(\x_{ij}).
\end{equation}
The second approximation considered by~\cite{sontag2007new} is called tree-reweighted (TRW) approximation. To achieve this, we decompose the the graph into a convex combination of spanning trees according to some distribution (over the trees), and let $\rho_{ij}$ be the so-called \emph{edge appearance probability}, which is computed as the number of spanning trees containing the edge $ij$ in the current decomposition, divided by the total number of all possible spanning trees containing $ij$ (in the entire distribution). The TRW approximation is then given by
\begin{equation}
	\tilde{H}_{\mathrm{TRW}}(\x) = \sum_{i\in\cV}H(\x_i) - \sum_{ij\in\cE}\rho_{ij}I(\x_{ij}).
\end{equation}

\paragraph{$\rhob$-reweighted entropic regularization} \citet{tang2016bethe} consider a more general term than the previous ones, based on the following approximation to $z \log z$ for $z\in[0,1]$, parameterized by $\eta\in [0,1]$:
\begin{equation}
	H_{\eta}(z) = \begin{cases}
		-z\log z&\mbox{if } z\in[\eta,1],\\
		-\eta\log \eta - (1+\log\eta)(z-\eta) - \frac{(z-\eta)^2}{2\eta}&\mbox{if } z\in[0,\eta].
	\end{cases}
\end{equation}
Define a similar version for vectors:
\begin{equation}
	H_{\eta}(\z) = \sum_{i=1}^p H_\eta(z_i)\quad \forall \z\in\RR^p.
\end{equation}
Their $\rhob$-reweighted approximation to the entropy $H(\x)$ is given by:
\begin{equation}
	\tilde{H}_{\eta}^{\rho}(\x) = \sum_{i\in\cV}  H_{\eta}(\x_i) - \sum_{ij\in\cE}\rho_{ij}\left[-H_{\eta}(\x_{ij}) + H_{\eta}(\x_i) + H_{\eta}(\x_j)\right]
\end{equation}
\citet{tang2016bethe} apply vanilla Frank-Wolfe to $E_\lp + r$ where $r = -\tilde{H}_{\eta}^{\rho}$. Note that their work consists in learning parameters of graphical models through maximum likelihood estimation. Here we only consider the inference part presented in their Section 3.2, which is used as a subroutine for learning.

\subsection{Algorithms based on the concave-convex procedure}
In the dense CRF model proposed by~\citet{krahenbuhl2011efficient}, the pairwise potentials consist of weighted sums of Gaussian kernels:
\begin{equation}
	\theta_{ij}(s,t) = \sum_{c=1}^C \mu^{(c)}(s,t)k^{(c)}(\f_i,\f_j)\quad \forall i,j\in\cV,\forall s,t\in\cS,
\end{equation}
where $C$ is the number of components, $\mu^{(c)}:\cS\times\cS \to \RR$ are the so-called label compatibility functions, and $k^{(c)}$ are Gaussian kernels over some image features $(\f_i,\f_j)$ (\S\ref{sec:model-details} presents a concrete example implemented for our experiments).

Define kernel matrices $\K^{(c)}\in\RR^{n\times n}$ with elements $K^{(c)}_{ij} = k(\f_i,\f_j)$ and compatibility matrices $\M^{(c)}\in\RR^{d\times d}$ with elements $M^{(c)}_{st} = \mu^{(c)}(s,t)$. Let
\begin{equation}
	\M = \sum_{c=1}^C\M^{(c)}.
\end{equation}
If we assume that $\K^{(c)}\in\RR^{n\times n}$ has unit diagonal: $K^{(c)}_{ii} = 1 \ \forall i, \forall c$, then our pairwise potential matrix $\P$ can be written as
\begin{equation}
	\P = \sum_{c=1}^{C}\left(\K^{(c)} - \I_n\right)\otimes \M^{(c)} = - \I_n \otimes \M + \sum_{c=1}^{C}\K^{(c)}\otimes \M^{(c)},
\end{equation}
where $\otimes$ denotes the Kronecker product, and $\I_n$ is the $n\times n$ identity matrix.

In the following, the concave-convex procedure (CCCP)~\cite{yuille2002concave} is applied to minimizing $f(\x) + g(\x)$ where $f$ is concave and $g$ is convex. (We integrate the constraint set $\cX$ into $g$ using its indicator function $\delta_\cX$, for consistency with our presentation of regularized Frank-Wolfe.)

\paragraph{Convergent mean field} \citet{krahenbuhl2013parameter} proposed (in their section 3.1) to minimizing a regularized energy $E(\x) + \x^\top\log\x$ (entropic regularizer) by applying CCCP to:
\begin{align}
	f(\x) &= \frac{1}{2}\x^\top (\P + \I_n\otimes\M) \x + \u^\top\x,\\
	g(\x)  &= -\frac{1}{2}\x^\top(\I_n\otimes \M)\x + \x^\top\log\x + \delta_\cX(\x).
\end{align}

\paragraph{Convergent mean field using concave approximation} \citet{krahenbuhl2013parameter} proposed (in their section 3.2) a more efficient algorithm using:
\begin{align}
	f(\x) &= \frac{1}{2}\x^\top(\P + \I_n\otimes\M)\x + \u^\top\x,\\
	g(\x)  &= \x^\top\log\x + \delta_\cX(\x)
\end{align}

\paragraph{CCCP for QP relaxation 1} \citet{desmaison2016efficient} proposed (in their section 5.1) the following application of CCCP:
\begin{align}
	f(\x) &= -\x^\top\diag(c)\x,\\
	g(\x) &= \frac{1}{2}\x^\top(2\diag(\c) + \P)\x + \u^\top\x + \delta_\cX(\x),
\end{align}
where $\c$ is defined by~\eqref{eq:c-convex-QP}.

\paragraph{CCCP for QP relaxation 2} Inspired by \citet{krahenbuhl2013parameter}, \citet{desmaison2016efficient} also proposed (in their section 5.2) another more efficient variant:
\begin{align}
	f(\x) &= \frac{1}{2}\x^\top (\P + \I_n\otimes\M) \x,\\
	g(\x)  &= -\frac{1}{2}\x^\top(\I_n\otimes \M)\x + \u^\top\x + \delta_\cX(\x).
\end{align}

\subsection{Summary of special cases}
We provide in Table~\ref{tab:summary-special-cases} a summary of special cases discussed in this section as well as in \S\ref{sec:instantiations} and \S\ref{sec:special-cases}. There we show how these algorithms can be obtained from regularized Frank-Wolfe by suitably choosing $f,g$ and $r$ in Algorithm~\ref{algo:regularized-FW}. Recall that $f + g = E + r +\delta_\cX$.

\begin{table*}[!htb]
	\centering
		\resizebox{\linewidth}{!}{
	\begin{tabular}{p{50mm}lr}
		\toprule
		\textbf{Algorithm} & $f(\x)$ & $g(\x) - \delta_\cX(\x)$  \\\midrule\midrule
		Parallel mean field \newline \citet{krahenbuhl2011efficient} & $E(\x)$ & $\x^\top\log\x$   \\\midrule
		Convergent mean field 1\newline
		\S3.1 in \citet{krahenbuhl2013parameter} & $E(\x) + \frac{1}{2}\x^\top\left(\I_n\otimes \M\right)\x$ & $- \frac{1}{2}\x^\top\left(\I_n\otimes \M\right)\x + \x^\top\log\x$   \\\midrule
		Convergent mean field 2\\
		\S3.2 in \citet{krahenbuhl2013parameter} & $E(\x) + \frac{1}{2}\x^\top\left(\I_n\otimes \M\right)\x$ & $\x^\top\log\x $   \\\midrule
		Nonconvex \fw\newline
		\citet{lehuu2018continuous} & $E(\x)$ & 0 \\\midrule
		Convex \fw \newline
		\S4 in \citet{desmaison2016efficient} & $E(\x) - \c^\top\x + \x^\top\diag(\c)\x$ & 0  \\\midrule
		CCCP for QP relaxation 1\newline
		\S5.1 in \citet{desmaison2016efficient} & $-\x^\top\diag(\c)\x$ & $E(\x) + \x^\top\diag(\c)\x$ \\\midrule
		CCCP for QP relaxation 2 \newline
		\S5.2 in \citet{desmaison2016efficient} & $E(\x) + \frac{1}{2}\x^\top\left(\I_n\otimes \M\right)\x - \u^\top\x$ & $-\frac{1}{2}\x^\top(\I_n\otimes \M)\x + \u^\top\x$ \\\midrule
		LP local-consistency Frank-Wolfe\newline
		\citet{meshi2015smooth} & $E_\lp(\x) + \frac{\lambda}{2}\norm{\A\x}_2^2$ & 0\\\midrule
		LP Bethe Frank-Wolfe\newline
		\citet{sontag2007new} & $E_\lp(\x) + \tilde{H}_{\mathrm{Bethe}}(\x)$ & 0\\\midrule
		LP TRW Frank-Wolfe\newline
		\citet{sontag2007new} & $E_\lp(\x) + \tilde{H}_{\mathrm{TRW}}(\x)$ & 0\\\midrule
		LP $\rhob$-reweighted Frank-Wolfe\newline
		\citet{tang2016bethe} & $E_\lp(\x) + \tilde{H}_{\eta}^{\rho}(\x)$ & 0\\\midrule
		Euclidean Frank-Wolfe\newline
		(This work) & $E(\x)$ & $\frac{\lambda}{2}\norm{\x}_2^2$\\\midrule
		Entropic Frank-Wolfe\newline
		(This work) & $E(\x)$ & $\lambda\x^\top \log\x$\\\midrule
		Lasso Frank-Wolfe\newline
		(This work, not implemented)& $E(\x)$ & $\lambda \norm{\x}_1$ \\
		\bottomrule
	\end{tabular}
	}
	\caption{\label{tab:summary-special-cases}Summary of special cases of regularized Frank-Wolfe.}
\end{table*}

\section{Detailed convergence analysis}\label{sec:convergence-analysis-full}
In this section, let $\norm{\cdot}$ denote the $\ell_2$ norm. The following lemma is useful for the proofs.

\begin{lemma}
	If $f$ and $g$ satisfy Assumption~\ref{assumption-f} and~\ref{assumption-g}, then
	\begin{align}
		f(\y) &\le f(\x) + \inner{\nabla f(\x), \y-\x} + \frac{L_f}{2}\norm{\y-\x}^2 \quad\forall \x,\y\in\dom f,\label{eq:f-inequality}\\
		g(\y) &\ge g(\x) + \inner{\d, \y-\x} + \frac{\sigma_g}{2}\norm{\y-\x}^2 \quad\forall \x,\y\in\dom g, \forall \d\in \partial g(\x).\label{eq:g-inequality}
	\end{align}
\end{lemma}
\begin{proof}
	For a convex function $h$, we have
	\begin{equation}
		h(\y) \ge h(\x) + \inner{\d, \y-\x} \quad\forall \x,\y,\forall \d\in \partial h(\x).
	\end{equation}
	Applying the above inequality with, respectively, $h(\x) = -f(\x) + \frac{L_f}{2}\norm{\x}_2^2$ and $h(\x) = g(\x) - \frac{\sigma_g}{2}\norm{\x}_2^2$, we obtain \eqref{eq:f-inequality} and \eqref{eq:g-inequality}. (Note that for the second case, $h$ is convex and thus $\partial g(\x) = \partial h(\x) + \sigma_g \x$.)
\end{proof}

\subsection{Proof of Lemma~\ref{lemma:stationarity}}
First we show that
\begin{equation}\label{eq:s-inequality}
	S(\x) \ge \frac{\sigma_g}{2}\norm{\x - \p_\x}^2 \ \forall \x\in\dom f.
\end{equation}
Notice that
\begin{equation}
	\p_\x \in	\argmin_{\p} \set{ \inner{\nabla f(\x), \p} + g(\p) } \iff -\nabla f(\x) \in \partial g(\p_\x).
\end{equation}
Hence, applying \eqref{eq:g-inequality} we obtain
\begin{equation}
	g(\x) \ge g(\p_\x) + \inner{-\nabla f(\x), \x-\p_\x} + \frac{\sigma_g}{2}\norm{\x-\p_\x}^2,
\end{equation}
which is precisely \eqref{eq:s-inequality}.

To complete the proof, we need to show that $S(\x^*) = 0$ if and only if $\x^*$ is a stationary point of \eqref{eq:composite}, \ie, $-\nabla f(\x^*) \in \partial g(\x^*)$. The following is due to \citet{beck2017first}. Notice that
\begin{equation}
	S(\x) = \max_{\p}\set{\inner{\nabla f(\x), \x-\p} + g(\x) - g(\p)},
\end{equation}
we have
\begin{align}
	S(\x^*) = 0
	\iff S(\x^*) \le 0
	&\iff \inner{\nabla f(\x^*), \x^*-\p} + g(\x^*) - g(\p) \le 0 \ \forall \p\\
	&\iff g(\p) \ge g(\x^*) + \inner{-\nabla f(\x^*), \p-\x^*} \ \forall \p\\
	&\iff -\nabla f(\x^*) \in \partial g(\x^*).
\end{align}
The proof is completed.

\subsection{Proof of Theorem~\ref{theorem:convergence}}

We need an additional lemma.
\begin{lemma}\label{lemma:fundamental-lemma}
	For any $\x\in\dom f$ and any $\alpha\in [0,1]$ we have
	\begin{equation}\label{eq:fundamental-lemma}
		F(\x + \alpha(\p_\x - \x)) -F(\x) \le -\alpha S(\x) + K(\alpha)\norm{\p_\x - \x}^2,
	\end{equation}
	where $K(\alpha) = \frac{1}{2}\left[(L_f+\sigma_g)\alpha^2 - \sigma_g \alpha\right]$.
\end{lemma}
\begin{proof}
	On one hand, from \eqref{eq:f-inequality} we have
	\begin{equation}
		f(\x + \alpha(\p_\x - \x)) \le  f(\x) + \alpha\inner{\nabla f(\x), \p_\x - \x} + \frac{L_f\alpha^2}{2}\norm{\p_\x - \x}^2.
	\end{equation}	
	On the other hand, from the $\sigma_g$-strong-convexity of $g$:
	\begin{equation}
		g(\x + \alpha(\p_\x - \x)) \le (1-\alpha)g(\x) + \alpha g(\p_\x) - \frac{\sigma_g \alpha(1-\alpha)}{2} \norm{\p_\x - \x}^2.
	\end{equation}	
	Summing up the above two inequalities, we obtain \eqref{eq:fundamental-lemma}.
\end{proof}
Let $S_k = S(\x\iter{k})$, $r_k = \norm{\p\iter{k} - \x\iter{k}}^2$, and $F_k = F(\x\iter{k})$. Applying \eqref{eq:fundamental-lemma} we have
\begin{equation}\label{eq:decrease-raw}
	\boxed{F_{k} - F_{k+1} \ge \alpha_k S_k - K(\alpha_k)r_k.}
\end{equation}
Therefore,
\begin{equation}
	\Delta_0 = F_0 - F^* \ge F_0 - F_{k+1} = \sum_{i=0}^k (F_i - F_{i+1}) \ge S\sum_{i=0}^k \alpha_i - \sum_{i=0}^k r_iK(\alpha_i),
\end{equation}
which implies
\begin{equation}\label{eq:bound-on-S}
	S \le \frac{\Delta_0 + \sum_{i=0}^k r_i K(\alpha_i)}{\sum_{i=0}^k \alpha_i}.
\end{equation}
This is an important inequality that will help us to obtain the convergence results for the weak stepsize schemes such as constant and non-summable ones.

For the adaptive (and line-search) stepsizes, the following observations will be useful. Notice that the RHS of \eqref{eq:fundamental-lemma} can be written as $\frac{1}{2}at^2 - bt$ where
\begin{equation*}
	a = (L_f + \sigma_g)\norm{\p_\x - \x}^2,\quad	b = S(\x) + \frac{\sigma_g}{2}\norm{\p_\x - \x}^2.
\end{equation*}
Therefore:
\begin{itemize}
	\item If $L_f = \sigma_g = 0$ then the RHS is just $-tS(\x)$.
	\item If $L_f = 0$ then the RHS is $-tS(\x) - \frac{\sigma_g}{2}t(1-t)\norm{\p_\x - \x}^2 \le -tS(\x)$.
	\item If $L_f + \sigma_g > 0$ then the minimum of the RHS is $$-\frac{b^2}{2a} = -\frac{\norm{\p_\x - \x}^2}{2(L_f + \sigma_g)}\left(\frac{S(\x)}{\norm{\p_\x - \x}^2} + \frac{\sigma_g}{2}\right)^2,$$ achieved at $$t^* = \frac{b}{a} = \frac{1}{L_f + \sigma_g}\left(\frac{S(\x)}{\norm{\p_\x - \x}^2} + \frac{\sigma_g}{2}\right).$$
\end{itemize}
The adaptive stepsize \eqref{eq:stepsize-adaptive-linesearch} are defined for the case $L_f + \sigma_g > 0$ is thus: 
\begin{equation}
	\alpha_k = \min\left\{1, \alpha_k^* \right\}, \quad \mbox{where } \alpha_k^* = \frac{1}{L_f + \sigma_g}\left(\frac{S_k}{r_k} + \frac{\sigma_g}{2}\right).\label{eq:stepsize-adaptive-short}
\end{equation}
Notice that the RHS of~\eqref{eq:decrease-raw} is a quadratic function of $\alpha_k$ with critical point $\alpha_k^*$. If $\alpha_k^* \le 1$, or equivalently $\frac{\sigma_g}{2} + L_f \ge \frac{S_k}{r_k}$, then $\alpha_k = \alpha_k^*$ and thus \eqref{eq:decrease-raw} becomes
\begin{equation}
	F_{k} - F_{k+1} \ge \alpha_k^* S_k - K(\alpha_k^*)r_k
	= \frac{r_k}{2(L_f + \sigma_g)}\left(\frac{S_k}{r_k} + \frac{\sigma_g}{2}\right)^2.\label{eq:toto}
\end{equation}
If $\alpha_k^* > 1$, or equivalently $\frac{\sigma_g}{2} + L_f < \frac{S_k}{r_k}$,  then $\alpha_k = 1$ and thus \eqref{eq:decrease-raw} becomes
\begin{equation}\label{eq:tata}
	F_{k} - F_{k+1} \ge S_k - K(1)r_k = S_k - \frac{L_f}{2}r_k.
\end{equation}

For simplicity and clarity, we will consider separately the two cases: $g$ is strongly convex ($\sigma_g > 0$) or simply convex ($\sigma_g = 0$). Recall that $\Omega$ is the (finite) diameter of $\dom g$, and thus we have $r_k\le \Omega^2 \ \forall k$, a fact that we will be using repeatedly in the sequel. Let $S = \min_{0\le i\le k} S_i$.

\subsubsection{Convex $g$}
In this section we consider the case where $g$ is convex but not strongly convex, \ie, $\sigma_g = 0$. Inequality~\eqref{eq:bound-on-S} becomes
\begin{equation}\label{eq:bound-on-S-convex}
	S \le \frac{\Delta_0 + \frac{L_f}{2}\sum_{i=0}^k r_i \alpha_i^2}{\sum_{i=0}^k \alpha_i}.
\end{equation}

\paragraph{Constant stepsize} Consider $\alpha_k = \alpha > 0 \ \forall k$. From $r_k\le \Omega^2$ and~\eqref{eq:bound-on-S-convex} we obtain
\begin{equation}
	S \le \frac{\Delta_0}{(k+1)\alpha} + \frac{L_f\Omega^2\alpha}{2}.
\end{equation}
The right-hand side converges to $\frac{1}{2}L_f\Omega^2\alpha$ as $k\to\infty$, \ie, the lower-bound $S$ on the conditional gradient norm converges to within $\frac{1}{2}L_f\Omega^2\alpha$. It is easy to deduce from the last inequality that $S\le L_f\Omega^2\alpha$ within $k \le \frac{2\Delta_0}{L_f\Omega^2\alpha^2}$ steps. We conclude that the algorithm converges to an approximate stationary point for the constant stepsize.

\paragraph{Constant step length} For a constant step length: $\norm{\x\iter{k+1} - \x\iter{k}} = \alpha$. Recall that $\x\iter{k+1} = \x\iter{k} + \alpha_k(\p\iter{k} - \x\iter{k})$, the corresponding stepsize is thus $\alpha_k = \frac{\alpha}{\norm{\p\iter{k} - \x\iter{k}}} = \frac{\alpha}{\sqrt{r_k}}$. 
Inequality~\eqref{eq:bound-on-S-convex} becomes
\begin{equation}\label{eq:bound-on-S-convex-constant-step-length}
	S \le \frac{\Delta_0 + \frac{L_f}{2} (k+1)\alpha^2}{\alpha\sum_{i=0}^k \frac{1}{\sqrt{r_k}}} \le \frac{\Delta_0 + \frac{L_f}{2} (k+1)\alpha^2}{\alpha \frac{k+1}{\Omega}} = \frac{\Delta_0\Omega}{(k+1)\alpha} + \frac{L_f\Omega\alpha}{2}.
\end{equation}
Therefore, $S$ converges to within $\frac{L_f\Omega\alpha}{2}$, and $S\le L_f\Omega \alpha$ within $k \le \frac{2\Delta_0}{L_f\alpha^2}$ steps. We conclude that the algorithm converges to an approximate stationary point for the stepsizes with constant step length.

\paragraph{Non-summable but square-summable stepsizes} Assume that the stepsizes $\alpha_k$ satisfy
\begin{equation}
	\sum_{k=0}^{+\infty} \alpha_k = \infty,\qquad \sum_{k=0}^{+\infty} \alpha_k^2 < \infty.
\end{equation}
A typical example is $\alpha_k = \frac{\alpha}{k + \beta}$, where $\alpha > 0$ and $\beta \ge 0$. This includes the common Frank-Wolfe stepsize $\alpha_k = \frac{2}{k+2}$. From $r_k\le \Omega^2$ and~\eqref{eq:bound-on-S-convex} we obtain
\begin{equation}\label{eq:bound-on-S-Delta}
	S \le \frac{\Delta_0 + \frac{L_f\Omega^2}{2}\sum_{i=0}^k \alpha_i^2}{\sum_{i=0}^k \alpha_i},
\end{equation}
which clearly converges to $0$ as $k\to\infty$. Therefore, the algorithm is guaranteed to converge to a stationary point in this case.

\paragraph{Diminishing (and non-summable) stepsizes} Assume that the stepsizes $\alpha_k$ satisfy
\begin{equation}
	\sum_{k=0}^{+\infty} \alpha_k = \infty,\qquad \lim_{k\to\infty} \alpha_k = 0.
\end{equation}
A typical example is $\alpha_k = \frac{\alpha}{\sqrt{k}}$, where $\alpha > 0$. Notice that for any $\epsilon > 0$, we have $\alpha_i^2 < \epsilon\alpha_i$ with $i$ large enough, it is straightforward to show that $\frac{\sum_{i=0}^k \alpha_i^2}{\sum_{i=0}^k \alpha_i} \to 0$ as $k\to\infty$, and thus \eqref{eq:bound-on-S-Delta} implies that $S\to 0$ as well. We conclude that the algorithm is guaranteed to converge to a stationary point.

\paragraph{Adaptive stepsizes} This result was obtained previously by \citet{beck2017first}. These stepsizes are computed according to~\eqref{eq:stepsize-adaptive-short}, which can be simplified as the following for $\sigma_g = 0$:
\begin{equation}
	\alpha_k = \min\left\{1, \alpha_k^* \right\}, \quad \mbox{where } \alpha_k^* = \frac{S_k}{L_fr_k}.\label{eq:stepsize-adaptive-short-convex}
\end{equation}
Then, if $\alpha_k^* \le 1$, \eqref{eq:toto} yields
\begin{align}
	F_{k} - F_{k+1} \ge \frac{S_k^2}{2L_f r_k} \ge \frac{S_k^2}{2L_f \Omega^2}.\label{eq:toto-convex}
\end{align}
If $\alpha_k^* > 1$ then \eqref{eq:tata} and $L_f r_k < S_k$ yield
\begin{equation}\label{eq:tata-convex}
	F_{k} - F_{k+1} \ge S_k - \frac{L_f}{2}r_k \ge \frac{S_k}{2}.
\end{equation}
Combining the two cases, we obtain 
\begin{equation}
	F_{k} - F_{k+1} \ge \frac{S_k}{2}\min\left\{1, \frac{S_k}{L_f\Omega^2}\right\} \ge \frac{S}{2}\min\left\{1, \frac{S}{L_f\Omega^2}\right\}.
\end{equation}
Therefore,
\begin{equation}
	\Delta_0 = F_0 - F^* \ge F_0 - F_{k+1} = \sum_{i=0}^k (F_i - F_{i+1}) \ge (k+1) \frac{S}{2}\min\left\{1, \frac{S}{L_f\Omega^2}\right\},
\end{equation}
which yields 
\begin{equation}
	S \le \max\left\{ \frac{2\Delta_0}{k+1} , \frac{\sqrt{2L_f\Omega^2\Delta_0}}{\sqrt{k+1}} \right\}.
\end{equation}
Therefore, the algorithm is guaranteed to converge to a stationary point, and the rate of convergence convergence is at least $\cO(1/\sqrt{k})$.

\subsubsection{Strongly-convex $g$} In this section we consider the case where $g$ is strongly convex with parameter $\sigma_g > 0$. 

Recall from \eqref{eq:decrease-raw} and Lemma~\ref{lemma:fundamental-lemma} that
\begin{equation}\label{eq:decrease-strongly-convex}
	F_{k} - F_{k+1} \ge \alpha_k S_k - K(\alpha_k)r_k \ \forall k \ge 0,\quad \mbox{where } K(\alpha) = \frac{1}{2}\alpha\left[(L_f+\sigma_g)\alpha - \sigma_g \right].
\end{equation}
Thus if $\alpha_k \le \frac{\sigma_g}{L_f + \sigma_g}$ we have $K(\alpha_k) \le 0$ and \eqref{eq:decrease-strongly-convex} yields
\begin{equation}\label{eq:decrease-strongly-convex-small}
	F_{k} - F_{k+1} \ge \alpha_k S_k.
\end{equation}
Consider now the case $\alpha_k > \frac{\sigma_g}{L_f + \sigma_g}$ for which $K(\alpha_k) > 0$. From~\eqref{eq:s-inequality} we have $r_k\le \frac{2S_k}{\sigma_g}$, and thus \eqref{eq:decrease-strongly-convex} yields
\begin{equation}\label{eq:decrease-strongly-convex-large}
	F_{k} - F_{k+1} \ge \alpha_k S_k - K(\alpha_k)\frac{2S_k}{\sigma_g} = \left(\alpha_k - \frac{2K(\alpha_k)}{\sigma_g}\right)S_k = \alpha_k \left(2 - \frac{L_f + \sigma_g}{\sigma_g}\alpha_k\right)S_k.
\end{equation}
Combining the two cases, we obtain
\begin{equation}\label{eq:decrease-strongly-convex-combined}
	F_{k} - F_{k+1} \ge \alpha_k\min\left(1, 2 - \frac{L_f + \sigma_g}{\sigma_g}\alpha_k\right) S_k.
\end{equation}
If $\alpha_k \ge \frac{2\sigma_g}{L_f + \sigma_g}$ then the RHS of \eqref{eq:decrease-strongly-convex-combined} is non-positive, thus this inequality is not helpful. In this case, we can obtain another inequality from \eqref{eq:decrease-strongly-convex}, noticing that $r_k\le\Omega^2 \ \forall k$ and $K(\alpha_k) > 0$:
\begin{equation}\label{eq:decrease-strongly-convex-large2}
	F_{k} - F_{k+1} \ge \alpha_k S_k - K(\alpha_k)\Omega^2
\end{equation}

\paragraph{Constant stepsize} Assume that $\alpha_k = \alpha \ \forall k \ge 0$. If $0 < \alpha < \frac{2\sigma_g}{L_f + \sigma_g}$ then \eqref{eq:decrease-strongly-convex-combined} yields
\begin{equation}
	F_{k} - F_{k+1} \ge \alpha\min\left(1, 2 - \frac{L_f + \sigma_g}{\sigma_g}\alpha\right) S \ \forall k \ge 0.  
\end{equation}
Hence
\begin{equation}
	\Delta_0 \ge F_0 - F_{k+1} = \sum_{i=0}^{k} (F_i - F_{i+1}) \ge (k+1)\alpha\min\left(1, 2 - \frac{L_f + \sigma_g}{\sigma_g}\alpha\right)S.
\end{equation}
We obtain
\begin{equation}\label{eq:bound-S-constant}
	S \le \frac{\Delta_0}{\alpha\min \left(1, 2 - \frac{L_f + \sigma_g}{\sigma_g}\alpha\right)(k+1)} \quad\forall \alpha < \frac{2\sigma_g}{L_f + \sigma_g}.
\end{equation}
We conclude that the algorithm is guaranteed to converge to a stationary point with rate of convergence of (at least) $\cO(1/k)$ for any $0 < \alpha < \frac{2\sigma_g}{L_f + \sigma_g}$.

For the remaining case $\alpha \ge \frac{2\sigma_g}{L_f + \sigma_g}$, we will derive an upper bound for $S$. Applying \eqref{eq:decrease-strongly-convex-large2} we obtain
\begin{equation}
	\Delta_0 \ge \sum_{i=0}^{k} (F_i - F_{i+1}) \ge \alpha \sum_{i=0}^k S_i - (k+1)K(\alpha)\Omega^2 \ge (k+1)\alpha S - (k+1)K(\alpha)\Omega^2,
\end{equation}
which yields
\begin{equation}
	S \le \frac{\Delta_0}{\alpha(k+1)} + \frac{K(\alpha)\Omega^2}{\alpha} = \frac{\Delta_0}{\alpha(k+1)} +\frac{1}{2}\left[(L_f+\sigma_g)\alpha - \sigma_g \right]\Omega^2.
\end{equation}
Therefore, for $\alpha \ge \frac{2\sigma_g}{L_f + \sigma_g}$ the algorithm converges to an approximate stationary point at which the conditional gradient norm is bounded above by $\frac{1}{2}\left[(L_f+\sigma_g)\alpha - \sigma_g \right]\Omega^2$.

\paragraph{Constant step length} Consider the stepsize $\alpha_k = \frac{\alpha}{\norm{\p\iter{k} - \x\iter{k}}} = \frac{\alpha}{\sqrt{r_k}}$ for which $\norm{\x\iter{k+1} - \x\iter{k}} = \alpha$. Inequality \eqref{eq:decrease-strongly-convex} becomes
\begin{align} %
	F_{k} - F_{k+1} &\ge \frac{\alpha}{\sqrt{r_k}} S_k - K\left(\frac{\alpha}{\sqrt{r_k}}\right)r_k \\
	&= \frac{\alpha}{\sqrt{r_k}} S_k - \frac{1}{2} \left[(L_f + \sigma_g)\frac{\alpha^2}{r_k} - \sigma_g\frac{\alpha}{\sqrt{r_k}}\right]r_k \\
	&= \frac{\alpha}{\sqrt{r_k}} S_k + \frac{\sigma_g\alpha\sqrt{r_k}}{2} - \frac{1}{2}(L_f +\sigma_g)\alpha^2 \\
	&\ge 2\sqrt{\frac{\alpha}{\sqrt{r_k}} S_k \frac{\sigma_g\alpha\sqrt{r_k}}{2}} - \frac{1}{2}(L_f +\sigma_g)\alpha^2\\
	&=\alpha\sqrt{2\sigma_g S_k} - \frac{1}{2}(L_f +\sigma_g)\alpha^2.
\end{align}
It follows that
\begin{align}
	\Delta_0 &\ge (k+1)\alpha\sqrt{2\sigma_g S} - \frac{k+1}{2}(L_f +\sigma_g)\alpha^2\\
	\implies \sqrt{S} &\le \frac{\Delta_0}{\alpha\sqrt{2\sigma_g}(k+1)} + \frac{(L_f + \sigma_g)\alpha}{2\sqrt{2\sigma_g}}.
\end{align}
For this stepsize scheme, the algorithm converges to an approximate stationary point, within $\frac{(L_f+\sigma_g)^2\alpha^2}{8\sigma_g}$. One can observe that, even though the strong convexity of $g$ still cannot guarantee convergence to a stationary point, it helps improve the bound as well as the rate of convergence from~\eqref{eq:bound-on-S-convex-constant-step-length}.

\paragraph{Diminishing (and non-summable) stepsizes} Assume that the stepsizes $\alpha_k$ satisfy
\begin{equation}
	\sum_{k=0}^{+\infty} \alpha_k = \infty,\qquad \lim_{k\to\infty} \alpha_k = 0.
\end{equation}
This scheme also includes the non-summable but square-summable one. Since $\lim_{k\to\infty} \alpha_k = 0$, there exists an integer $k(\omega)$ such that $\alpha_k \le \omega = \frac{\sigma_g}{L_f + \sigma_g} \ \forall k \ge k(\omega)$. Now applying \eqref{eq:decrease-strongly-convex-small} 
\begin{equation}
	\Delta_{k(\omega)} \ge F_{k(\omega)} - F_{k+1} = \sum_{i=k(\omega)}^{k} (F_i - F_{i+1}) \ge \sum_{i=k(\omega)}^{k} \alpha_i S_i \ge \left(\sum_{i=k(\omega)}^{k} \alpha_i\right)S,
\end{equation}
which yields
\begin{equation}
	S \le \frac{\Delta_{k(\omega)}}{\sum_{i=k(\omega)}^{k} \alpha_i}.
\end{equation}
Since $(\alpha_k)$ is non-summable, the algorithm converges to a stationary point. Compared to the non-strongly-convex case, we observe that the assumption that $\dom g$ is compact can be relaxed (its diameter $\Omega$ is not used in the proof).

\paragraph{Adaptive stepsizes} Recall that the stepsizes in this scheme are given by \eqref{eq:stepsize-adaptive-short} as
\begin{equation}
	\alpha_k = \min\left\{1, \alpha_k^* \right\}, \quad \mbox{where } \alpha_k^* = \frac{1}{L_f + \sigma_g}\left(\frac{S_k}{r_k} + \frac{\sigma_g}{2}\right).
\end{equation}
If $\alpha_k^* \le 1$, from \eqref{eq:toto} and the inequality $(a+b)^2\ge 4ab$, we obtain
 \begin{equation}
 	F_{k} - F_{k+1}  \ge  \frac{r_k}{2(L_f + \sigma_g)}4\frac{S_k}{r_k} \frac{\sigma_g}{2} = \frac{\sigma_gS_k}{L_f + \sigma_g}.
 \end{equation}
 If $\alpha_k^* > 1$, which is $r_k < \frac{S_k}{\frac{\sigma_g}{2} + L_f}$, then \eqref{eq:tata} yields
 \begin{equation}
 	F_{k} - F_{k+1} \ge S_k - \frac{L_f}{2} \frac{S_k}{\frac{\sigma_g}{2} + L_f}
 	= \frac{\sigma_g + L_f}{\sigma_g + 2L_f}S_k
 	\ge \frac{\sigma_g}{\sigma_g + L_f}S_k.
 \end{equation}
Therefore, we always have $F_{k} - F_{k+1} \ge \omega S_k$ where $\omega = \frac{\sigma_g}{\sigma_g + L_f}$. It then follows that
\begin{equation}
	\Delta_0 \ge \sum_{i=0}^{k} (F_i - F_{i+1}) \ge \sum_{i=0}^{k} \omega S_k \ge (k+1)\omega S \implies S \le \frac{\Delta_0}{\omega(k+1)}.
\end{equation}
Finally, the line search scheme is guaranteed to achieve the best decrease in the objective, thus the inequality $F_{k} - F_{k+1} \ge \omega S_k$ also holds and we obtain the same results for this scheme.

\subsubsection{Concave $f$} The results for this case can be obtained in a straightforward manner by setting $L_f = 0$ in the ``convex $g$'' case. In particular, for the adaptive stepsizes, \eqref{eq:stepsize-adaptive-short-convex} yields $\alpha_k = 1$ and thus it becomes the constant stepsize scheme.

\subsubsection{Summary of convergence results}
We summarize the results in Table~\ref{tab:convergence-summary}.

\begin{table*}[!htb]
	\centering
		\begin{tabular}{l|c|c|c}
& stepsize & decrease lower bound $\delta_k$ & optimality upper bound $B_k$\\
& & ($F_k -F_{k+1} \ge \delta_k$) & ($\min_{0\le i \le k} S_i \le B_k$)\\\hline\hline
\multirow{4}{*}{convex $g$}	& $\alpha_k=\alpha > 0$ & $\alpha S_k - \frac{L_f\Omega^2\alpha^2}{2}$ & $\frac{\Delta_0}{(k+1)\alpha} + \frac{L_f\Omega^2\alpha}{2}$ \\\cline{2-4}
& $\alpha_k = \frac{\alpha}{\norm{\p\iter{k} - \x\iter{k}}}$ & $\frac{\alpha}{\Omega} S_k - \frac{L_f\alpha^2}{2}$ & $\frac{\Delta_0\Omega}{(k+1)\alpha} + \frac{L_f\Omega\alpha}{2}$  \\\cline{2-4}
& \makecell[c]{$\sum_{k=0}^{+\infty} \alpha_k = \infty$ \\ $\lim_{k\to\infty}\alpha_k = 0$ }  & $\alpha_k S_k - \frac{L_f\Omega^2\alpha_k^2}{2}$ & $\frac{\Delta_0}{\sum_{i=0}^k \alpha_i} + \frac{L_f\Omega^2}{2}\frac{\sum_{i=0}^k \alpha_i^2}{\sum_{i=0}^k \alpha_i}$  \\\cline{2-4}
& \makecell[c]{adaptive or \\ line search \eqref{eq:stepsize-adaptive-linesearch}} & $\frac{1}{2}\min\left( S_k, \frac{S_k^2}{L_f\Omega^2} \right)$ & $\max\left( \frac{2\Delta_0}{k+1} , \frac{\sqrt{2L_f\Omega^2\Delta_0}}{\sqrt{k+1}} \right)$  \\\hline\hline
\multirow{5}{*}{strongly-convex $g$} & $\alpha_k=\alpha < 2\omega$ & $\alpha\min\left(1, 2 - \frac{\alpha}{\omega}\right)S_k$ & $\frac{\Delta_0}{\alpha\min\left(1, 2 - \frac{\alpha}{\omega}\right)(k+1)}$  \\\cline{2-4}
& $\alpha_k = \alpha \ge 2\omega$ & $\alpha S_k - K(\alpha)\Omega^2$ & $\frac{\Delta_0}{\alpha(k+1)} +\frac{K(\alpha)}{\alpha}\Omega^2$  \\\cline{2-4}
& $\alpha_k = \frac{\alpha}{\norm{\p\iter{k} - \x\iter{k}}}$ & $\alpha\sqrt{2\sigma_g S_k} - \frac{1}{2}(L_f +\sigma_g)\alpha^2$ & $\Big(\frac{\Delta_0}{\alpha\sqrt{2\sigma_g}(k+1)} + \frac{(L_f + \sigma_g)\alpha}{2\sqrt{2\sigma_g}}\Big)^2$  \\\cline{2-4}
& \makecell[c]{$\sum_{k=0}^{+\infty} \alpha_k = \infty$ \\ $\lim_{k\to\infty}\alpha_k = 0$ } & $\alpha_k\min\left(1, 2 - \frac{\alpha_k}{\omega}\right)S_k$ & $\frac{\Delta_{k(\omega)}}{\sum_{i=k(\omega)}^{k} \alpha_i}$ \\\cline{2-4}
 & \makecell[c]{adaptive or \\ line search \eqref{eq:stepsize-adaptive-linesearch}}  & $\omega S_k$ & $\frac{\Delta_0}{\omega(k+1)}$  \\\hline\hline
\multirow{4}{*}{concave $f$} & $\alpha_k=\alpha > 0$ & $\alpha S_k$ & $\frac{\Delta_0}{(k+1)\alpha}$\\\cline{2-4}
 & $\alpha_k = \frac{\alpha}{\norm{\p\iter{k} - \x\iter{k}}}$& $\frac{\alpha}{\Omega} S_k$ & $\frac{\Delta_0\Omega}{(k+1)\alpha}$\\\cline{2-4}
 & $\sum_{k=0}^{+\infty} \alpha_k = \infty$ & $\alpha_k S_k$ & $\frac{\Delta_0}{\sum_{i=0}^k \alpha_i}$ \\\cline{2-4}
 & \makecell[c]{adaptive or \\ line search \eqref{eq:stepsize-adaptive-linesearch}}  & $S_k$ & $\frac{\Delta_0}{k+1}$ \\\bottomrule
	\end{tabular}
\caption{\label{tab:convergence-summary}Summary of convergence analysis of the generalized Frank-Wolfe algorithm. Recall that $\omega = \frac{\sigma_g}{L_f + \sigma_g}$. Whenever a result does not involve $\Omega$, the assumption that $\dom g$ being compact can be relaxed.}
\end{table*}

\subsubsection{Convergence of $S(\x\iter{k})$} To complete the proof of Theorem~\ref{theorem:convergence}, we need to show that for the non-highlighted cases of its table (page~\pageref{theorem:convergence}), we have $\lim_{k\to\infty}S(\x\iter{k}) = 0$ and any limit point of the sequence $(\x\iter{k})_{k\ge 0}$ is a stationary point of~\eqref{eq:composite}. Indeed, for these cases, $(F_k)_{k\ge 0}$ is a decreasing sequence because $F_k - F_{k+1}$ is bounded below by a non-negative quantity $\delta_k$ (according to Table~\ref{tab:convergence-summary} presented in the previous section). Therefore, $F_k$ is convergent as it is bounded below by $F^*$. Consequently $F_k - F_{k+1} \to 0$, which implies $\delta_k \to 0$ and thus $S_k \to 0$ as well for the considered cases (see Table~\ref{tab:convergence-summary}). The results follow in a straightforward manner.

\section{Proofs of other theoretical results}\label{sec:proofs}

\subsection{Vanilla Frank-Wolfe fails to learn: the zero-gradient issue }\label{sec:non-diff-vanilla-FW}

We claimed in \S\ref{sec:smoothing-perspective} that vanilla Frank-Wolfe \eqref{eq:vanilla-Frank-Wolfe} is problematic for learning with SGD because its iterates are piecewise-constant and thus their gradients are zero almost everywhere (more precisely the gradient is undefined on the boundaries while being zero everywhere else). In this section, we present a theoretical justification for this claim.

It suffices to show that $\p^* = \argmin_{\p\in\cX} \inner{\c, \p}$ is piecewise-constant with respect to $\c$. Let $\Delta_d$ denote the simplex $\set{\z\in\RR^d \mid \1^\top \z = 1,  \z\ge \0}$. Clearly, the set $\cX$ (defined by \eqref{eq:set-X}) can be written as $\set{\x\in\RR^{nd} \mid \x_i\in\Delta_d \ \forall i\in\cV}$, and thus the above minimization problem can be reduced to solving the following problem for each $i\in\cV$ independently:
\begin{equation}
	\p_i^* \in \argmin_{\p_i\in\Delta_d} \inner{\c_i, \p_i}.
\end{equation}
For notational convenience, consider the following problem with a constant vector $\b = (b_1,b_2,\dots,b_d)\in\RR^d$:
\begin{equation}\label{eq:argmin}
	\z^* \in \argmin_{\z\in\Delta_d} \inner{\b, \z}.
\end{equation}
Let $s^*$ be the index of the minimum element of $\b$, \ie,  $s^* = \argmin_{s} b_s$. Let $\e_s\in\set{0,1}^d$ denote the one-hot vector where the $s\textsuperscript{th}$ element is one. We have:
\begin{equation}
	\inner{\b,\z} = \sum_{s=1}^{d} b_s z_s \ge \sum_{s=1}^{d} b_{s^*}z_s = b_{s^*} \sum_{s=1}^{d} z_s = b_{s^*} = \inner{\b,\e_{s^*}} \quad \forall \z\in\Delta_d.
\end{equation}
Therefore, $\e_{s^*}$ is an optimal solution to \eqref{eq:argmin}. It is straightforward that the index of the minimum element of a vector is piecewise constant, thus $\e_{s^*}$ is also piecewise constant (as a function of $\b$). Therefore, $\e_{s^*}$ is not continuous (thus non-differentiable) on the boundaries, while in the constant regions, its gradient is zero. 

\begin{remark}
	We can deduce that the iteration complexity of vanilla Frank-Wolfe is $\cO(nd)$.
\end{remark}

\subsection{Proof of the relaxation tightness (Theorem~\ref{theorem:energy-additive-bound})}\label{sec:tightness-full}

We give a proof of Theorem~\ref{theorem:energy-additive-bound} in \S\ref{sec:tightness-analysis}. Recall that we have to prove $E^* \le E(\bar{\x}_r^*) \le E^* + M - m + C$, where 
\begin{equation}C =
	\begin{cases}
		\sqrt{n\left(1-\frac{1}{d}\right)}\left(\norm{\u}_2 + \sqrt{n}\norm{\P}_{2}\right) &\mbox{for nearest rounding}\\
		0&\mbox{for BCD rounding}.
	\end{cases}
\end{equation}

Let $\x^*$ be such that $E(\x^*) = E^*$ and consider first the BCD rounding scheme. As this scheme is guaranteed to not increase the energy, we have
	\begin{equation*}
		E(\bar{\x}_r^*) \le E(\x_r^*) = E_r(\x_r^*) - r(\x_r^*) \le E_r(\x^*) - r(\x_r^*)
		= E(\x^*) + r(\x^*) - r(\x_r^*) \le E^* + M - m.
	\end{equation*}

It remains to prove the result for the nearest rounding scheme. In this scheme, the discrete energy may increase (or decrease), but it can be shown that the variation is bounded by the given constant: $\abs{E(\bar{\x}_r^*) - E(\x_r^*)} \le C$. Then, the rest of the proof is similar to the BCD case. This bounding inequality is proved as follows.

 Suppose that we obtain a discrete solution $\y\in\cX\cap\set{0,1}^{nd}$ from some $\x\in\cX$ using nearest rounding. We will prove that
\begin{equation}\label{eq:rounding-bound-C}
	\abs{E(\x) - E(\y)} \le C,
\end{equation}
where
\begin{equation}
	C = \sqrt{n\left(1-\frac{1}{d}\right)}\left(\norm{\u}_2 + \sqrt{n}\norm{\P}_{2}\right).
\end{equation}

\begin{lemma}
	For any $\z\in \Delta_d$ (see \S\ref{sec:non-diff-vanilla-FW} for notation) and its rounded vector $\v \in\Delta_d\cap\set{0,1}^d$, \ie, $v_i = 1$ if $i=\argmax_{1\le j\le d} v_j$ and $v_j = 0 \ \forall j\neq i$, we have
	\begin{equation}\label{eq:rounding-bound-d}
		\norm{\z-\v}_2^2 \le 1-\frac{1}{d}.
	\end{equation}
\end{lemma}
\begin{proof}
	Without loss of generality, assume that $z_1$ is the maximum element of $\z$. Then, we have $v_1 = 1$ and $v_j = 0 \ \forall j > 1$. 
	\begin{equation}
		\norm{\z-\v}_2^2 = \sum_{i=1}^{d}(z_i - v_i)^2 
		= (z_1-1)^2 + \sum_{i=1}^{d} z_i^2 
		= S^2 + z_2^2+\dots+z_d^2,\label{eq:titi}
	\end{equation}
	where $S = z_2+\dots+z_d$. We will make use of the following trivial inequality:
	\begin{equation}\label{eq:tete}
		z_i + S \le 1 \quad \forall i\ge 2.
	\end{equation}
	On one hand, summing the $d-1$ inequalities \eqref{eq:tete} (for $i=2,\dots,d$) we obtain
	\begin{equation}\label{eq:tutu}
		S + (d-1)S \le d-1 \implies S\le 1 - \frac{1}{d}.
	\end{equation}
	On the other hand, multiplying \eqref{eq:tete} with $z_i$ and summing up the obtained $d-1$ inequalities we get
	\begin{equation}\label{eq:tutu2}
		\sum_{i=2}^{d} z_i^2 + S^2 \le S
	\end{equation}
	Finally, from \eqref{eq:titi}, \eqref{eq:tutu}, and \eqref{eq:tutu2} we get \eqref{eq:rounding-bound-d}.
\end{proof}
Back to \eqref{eq:rounding-bound-C}. Applying \eqref{eq:rounding-bound-d} we have
\begin{equation}\label{eq:rounding-bound}
	\norm{\x-\y}_2^2 = \sum_{i=1}^n \norm{\x_i - \y_i}_2^2 \le n\left(1-\frac{1}{d}\right).
\end{equation}
On the other hand
\begin{equation}
	\norm{\x+\y}_2^2 = \sum_{i=1}^n \norm{\x_i + \y_i}_2^2 \le \sum_{i=1}^n \left[\1^\top(\x_i + \y_i)\right]^2 = 4n.
\end{equation}
Applying the two above inequalities, together with the triangle and Cauchy-Schwarz inequalities we have:
\begin{align*}
	\abs{E(\x) - E(\y)}
	&= \abs{\u^\top(\x-\y) + \frac{1}{2}\x^\top\P\x - \frac{1}{2}\y^\top\P\y}\\
	&\le \abs{\u^\top(\x-\y)} + \frac{1}{2}\abs{\x^\top\P\x - \y^\top\P\y}\\
	&= \abs{\u^\top(\x-\y)} + \frac{1}{2}\abs{(\x-\y)^\top\P(\x+\y)}\\
	&\le \norm{\u}_2\norm{\x-\y}_2 + \frac{1}{2}\norm{\x-\y}_2\norm{\P}_2\norm{\x+\y}_2\\
	&\le \norm{\u}_2\sqrt{n\left(1-\frac{1}{d}\right)} + \frac{1}{2}\sqrt{n\left(1-\frac{1}{d}\right)} \norm{\P}_2 2\sqrt{n}\\
	&= C.
\end{align*}
This completes the proof.

\section{Implementation details of all methods}\label{sec:derivation-peer-methods}

We present the implementation details for all the methods presented in the experiments (\S\ref{sec:experiments}). Recall that our problem of interest is
$$\min_{\x\in\cX} E(\x) \triangleq \frac{1}{2}\x^\top\P\x + \u^\top\x,$$ 
and that the same initialization $\x\iter{0} = \softmax(-\u)$ is used for all methods.

\paragraph{High-dimensional filtering for gradient computation} For all methods, we need to compute the energy gradient $\nabla E(\x) = \P\x + \u$ at each iteration, where the evaluation of $\P\x$ is an expensive operation because the graph is fully-connected (\ie, $\P$ is dense). Fortunately, since the pairwise potentials are Gaussian, this multiplication can be performed efficiently (and approximately) in $\cO(nd)$ time using high-dimensional filtering, which is the key idea behind the original dense CRFs paper \cite{krahenbuhl2011efficient}. We refer to this reference for more details. Our code is based on the efficient GPU implementation of \citet{monteiro2018conditional}.\footnote{\url{https://github.com/MiguelMonteiro/permutohedral_lattice}}

\subsection{Euclidean Frank-Wolfe (\lfw)} \label{sec:lfw-details}
We give the details for the main update \eqref{eq:update-l2-FW} of Euclidean Frank-Wolfe as presented in \S\ref{sec:instantiations}. This step follows from
\begin{align}
	\forall k\ge 0:\quad	\p\iter{k}  &= \argmin_{\p\in\cX} \set{\inner{\P\x\iter{k} + \u, \p} + \frac{\lambda}{2}\norm{\p}_2^2} \\
	&=\argmin_{\p\in\cX} \set{\frac{\lambda}{2}\left\|\p + \frac{1}{\lambda}(\P\x\iter{k} + \u)\right\|^2}\\
	&= \Pi_\cX\left(-\frac{1}{\lambda}(\P\x\iter{k} + \u)\right).
\end{align}
Recall that $\Pi_\cX(\v)$ denotes the projection of a vector $\v$ onto the set $\cX$. Recall also from~\eqref{eq:set-X} that $\cX = \set{\x \in\RR^{nd}: \x\ge\0, \1^\top\x_i = 1\ \forall i\in\cV}$, thus the projection on $\cX$ clearly reduces to $n$ \emph{independent} projections onto the probability simplex $\Delta_d = \set{\z\in\RR^d: \1^\top \z = 1,  \z\ge \0}$ for each $i\in\cV$. Projection onto the simplex is a rather well studied problem in the literature~\cite{condat2016fast}, and we present below the solution (that also shows how we implemented this operation).

\begin{lemma}
	For a given vector $\c\in\RR^d$, the optimal solution $\z^*$ to
	\begin{equation}
		\min_{\1^\top \z = 1, \z \ge \0} \|\z - \c\|^2
	\end{equation}
	is given as follows. Sort $\c$ in decreasing order to obtain a vector $\a = (a_1,a_2,\dots,a_d)$ (\ie, $a_1\ge a_2\ge\dots\ge a_d$) and let
	\begin{equation}
		\gamma_k = \frac{1}{k}(a_1+a_2+\dots+a_k - 1), \quad k=1,2,\dots,d.
	\end{equation}
	Let $k^*$ be the largest $k$ such that $a_k > \gamma_k$, then the optimal solution is given by
	\begin{equation}
		\z^* = \max(\c - \gamma_{k^*}, 0).
	\end{equation}
\end{lemma}
In the above, the ``$\max$'' and ``$-$'' operations are understood to be element-wise.
\begin{remark}
	If we use an $\cO(d\log d)$ sorting algorithm, then we see that the per-iteration complexity of Euclidean Frank-Wolfe is $\cO(nd\log d)$. It should be noted, however, that highly-efficient simplex-projection algorithms exist and have $\cO(d)$ complexity in practice \cite[Table 1]{condat2016fast}, yielding $\cO(nd)$ complexity, which is the same as in vanilla Frank-Wolfe (see \S\ref{sec:non-diff-vanilla-FW}).
\end{remark}

\subsection{Entropic Frank-Wolfe (\efw)} \label{sec:efw-details}
We give the details for the main update \eqref{eq:update-entropy-FW} of Entropic Frank-Wolfe as presented in \S\ref{sec:instantiations}. We need to show that
\begin{equation}
	\p\iter{k}  = \argmin_{\p\in\cX} \set{\inner{\P\x\iter{k} + \u, \p} -\lambda H(\p) }
	= \softmax\left(-\frac{1}{\lambda}(\P\x\iter{k} + \u)\right) \quad\forall k\ge 0,
\end{equation}
where $H(\x) = -\sum_{i\in\cV} \sum_{s\in\cS} x_{is}\log x_{is}$. Again, the above reduces to $n$ independent subproblems over each $i\in\cV$ to which the solutions are given by the following lemma.
\begin{lemma}\label{lem:entropy-conjugate}
	For a given vector $\c\in\RR^d$, the optimal solution $\z^*$ to
	\begin{equation}
		\min_{\1^\top \z = 1, \z \ge \0} \set{\inner{\c,\z} + \sum_{s=1}^{d}z_s\log z_s}
	\end{equation}
	is $\z^* = \softmax(-\c)$.
\end{lemma}
\begin{proof}
	The Lagrangian of the above problem is given by
	\begin{align}
		L(\z, \mub, \nu) &= \inner{\c,\z} + \sum_{s=1}^{d}z_s\log z_s + \mub^\top(-\z) + \nu(\1^\top \z - 1) \\
		&= -\nu + \sum_{s=1}^{d}\left(c_sz_s + z_s\log z_s - \mu_s z_s + \nu z_s\right),
	\end{align}
	where $\mub = (\mu_1,\mu_2,\dots,\mu_d)\ge \0$ and $\nu\in\RR$ are the Lagrange multipliers.
	
	Observe that the given problem is convex and the corresponding Slater's constraint qualification holds (\ie, there exists $\z\in\RR^d$ such that $\1^\top \z = 1$ and $\z > \0$), it suffices to solve the following	Karush–Kuhn–Tucker (KKT) system to obtain the optimal solution:
	\begin{align}
		\frac{\partial L(\z, \mub, \nu)}{\partial z_s} = c_s + \log z_s + 1 - \mu_s + \nu &= 0 \quad \forall 1\le s\le d,\\
		\1^\top \z &= 1,\\
		\z &\ge \0,\\
		\mub &\ge \0,\\
		\mu_s z_s &= 0 \quad \forall 1\le s\le d.
	\end{align}
	The first equation implies $z_s > 0 \ \forall s$, and thus in combination with the last, we obtain $\mu_s = 0 \ \forall s$. Therefore, the first equation becomes
	\begin{equation}\label{eq:z-in-nu-and-c}
		z_s = \exp(-1-\nu)\exp(-c_s) \ \forall s.
	\end{equation}
	 Summing up this result for all $s$, and taking into account the second equation, we obtain
	 \begin{equation}\label{eq:nu-in-c}
	 	\exp(-1-\nu) = \frac{1}{\sum_{s=1}^{d}\exp(-c_s)}.
	 \end{equation}
	Combining \eqref{eq:z-in-nu-and-c} and \eqref{eq:nu-in-c} we obtain
	\begin{equation}
		z_s = \frac{\exp(-c_s)}{\sum_{t=1}^{d}\exp(-c_t)} \ \forall 1\le s\le d.
	\end{equation}
	In other words, $\z = \softmax(-\c)$.
\end{proof}
\begin{remark}
	It is clear that the per-iteration complexity of Entropic Frank-Wolfe is $\cO(nd)$, which is the same as in vanilla Frank-Wolfe.
\end{remark}

\subsection{Projected gradient descent (PGD)}
This algorithm consists in the following updates:
\begin{equation}\label{eq:pgd}
	\p\iter{k} = \Pi_\cX(\x\iter{k} - \nabla E(\x)),\qquad \x\iter{k+1} = \x\iter{k} + \alpha_k(\p\iter{k} - \x\iter{k}),
\end{equation}
where the stepsize $\alpha_k$ follows one of the schemes presented in \S\ref{sec:convergence-analysis}. It is worth noting that this variant of PGD is eligible to exact line search \eqref{eq:stepsize-adaptive-linesearch}. We observe in our experiments that using the line search scheme produces the same results as setting $\alpha_k = 1$. Thus we used this constant scheme for both training and prediction. The same applies to the Frank-Wolfe variants.

\subsection{Fast proximal gradient method (PGM)}
The original PGM~\cite{lions1979splitting} consists in updating
\begin{equation}\label{eq:pgm}
	\x\iter{k+1} = \argmin_{\x\in\cX} \left\{\inner{\nabla E(\x\iter{k}), \x} + \frac{1}{2\alpha_k}\norm{\x-\x\iter{k}}_2^2\right\},
\end{equation}
which can be re-written as
\begin{equation}
	\x\iter{k+1} =\Pi_\cX(\x\iter{k} - \alpha_k\nabla E(\x\iter{k})).
\end{equation}
The above is precisely another variant of PGD (which is not eligible to exact line search \eqref{eq:stepsize-adaptive-linesearch}). While this algorithm is also supported by our implementation, the results presented in \S\ref{sec:experiments} are obtained using another variant called the \emph{fast} PGM, also known as FISTA~\cite{beck2009fista}. This algorithm consists in the following updates, where $\y\iter{0} = \x\iter{0} = \softmax(-\u)$ and $t_0 = 1$:
\begin{align}
	\x\iter{k+1} &= \Pi_\cX(\y\iter{k} - \alpha_k\nabla E(\y\iter{k})),\\
	t_{k+1} &= \frac{1+\sqrt{1+4t_k^2}}{2},\\
	\y\iter{k+1} &=\x\iter{k+1} + \frac{t_k -1}{t_{k+1}}(\x\iter{k+1} - \x\iter{k}).
\end{align}
While the optimal value of $\alpha_k$ can be determined using \emph{backtracking} \cite{beck2017first}, this process is very expensive as it requires evaluating the energy many times. Therefore, in practice, we use the constant scheme $\alpha_k = \alpha \in [0, 1]$. Doing a grid search on a random subset of $10$ validation images, we found (again) that $\alpha_k = 1$ is the best, and thus it is used for all the experiments.

\subsection{Entropic mirror descent (EMD)}
Mirror descent (MD)~\cite{beck2003mirror,nemirovskij1983problem} is a generalization of PGM to a more general distance function. Each iteration of MD takes the following form: 
\begin{equation}\label{eq:md}
	\x\iter{k+1} = \argmin_{\x\in\cX} \left\{\inner{\nabla E(\x\iter{k}), \x}+ \frac{1}{\alpha_k}B_{\phi}(\x,\x\iter{k})\right\},
\end{equation}
where $\phi:\cX\to\RR$ is a convex and continuously differentiable function on the interior of $\cX$, and $B_{\phi}:\cX\times\cX\to\RR$ is its associated Bregman divergence, defined by
\begin{equation}\label{eq:bregman}
B_{\phi}(\x,\y) = \phi(\x) - \phi(\y) - \inner{\nabla \phi(\y), \x - \y}.	
\end{equation}
Clearly, for $\phi(\x) = \frac{1}{2}\norm{\x}_2^2$ we recover the PGM update~\eqref{eq:pgm}. We provide an implementation for the so-called \emph{entropic} variant of mirror descent~\cite{beck2003mirror}, corresponding to choosing $\phi$ to be the negative entropy:
\begin{equation}
	\phi(\x) = \sum_{i=1}^n\sum_{s=1}^d x_{is}\log x_{is}.
\end{equation}
With this choice of $\phi$, it is easy to check that the Bregman divergence~\eqref{eq:bregman} becomes the following so-called Kullback-Leibler divergence:
\begin{equation}\label{eq:kullback-leibler}
	B_{\mathrm{KL}}(\x,\y) = \sum_{i=1}^n\sum_{s=1}^d x_{is}\log \frac{x_{is}}{y_{is}}.
\end{equation}
The MD update \eqref{eq:md} thus becomes
\begin{equation}
	\x\iter{k+1} = \argmin_{\x\in\cX} \left\{\inner{\alpha_k\nabla E(\x\iter{k}) - \log \x\iter{k}, \x} + \sum_{i=1}^n\sum_{s=1}^d x_{is}\log x_{is} \right\},
\end{equation}
where the $\log$ operation is taken element-wise. According to Lemma~\ref{lem:entropy-conjugate} (\S\ref{sec:efw-details}),we obtain
\begin{equation}
	\x\iter{k+1} = \softmax\left(\log \x\iter{k} - \alpha_k\nabla E(\x\iter{k})\right).
\end{equation}
Let $\g\iter{k}$ denote the gradient $\nabla E(\x\iter{k})$, the above reads
\begin{equation}\label{eq:emd-unstable}
	x_{is}\iter{k+1} = \frac{x_{is}\iter{k}\exp(-\alpha_k g_{is}\iter{k})}{\sum_{t=1}^{d} x_{it}\iter{k}\exp(-\alpha_k g_{it}\iter{k})} \quad \forall i\in\cV,\forall s\in\cS.
\end{equation}
\paragraph{Numerically stable EMD} In practice, the above expression of $\x\iter{k+1}$ may lead to numerical underflow or overflow. We overcome this by using the following modified iterate:
\begin{equation}\label{eq:emd-stable}
	x_{is}\iter{k+1} = \frac{(x_{is}\iter{k} + \epsilon)\exp(-\alpha_k g_{is}\iter{k} + m_i\iter{k} )}{\sum_{t=1}^{d} (x_{it}\iter{k} + \epsilon)\exp(-\alpha_k g_{it}\iter{k} + m_i\iter{k} )} \quad \forall i\in\cV,\forall s\in\cS,
\end{equation}
where $\epsilon = 10^{-10}$ and $m_i\iter{k} = \alpha_k\min_{1\le s\le d} g_{is}\iter{k} \ \forall i\in\cV$.

\subsection{Alternating direction method of multipliers (ADMM)}
The nonconvex ADMM for MAP inference \cite{lehuu2018continuous} consists in the following updates, where $\z\iter{0} = \softmax(-\u)$ and $\y\iter{0} = \0$:
\begin{align}
	\x\iter{k+1} &= \Pi_\cX\left(\z\iter{k} - \frac{1}{\rho}(\y\iter{k} + \frac{1}{2}\P\z\iter{k} + \u)\right),\label{eq:admm-x}\\
	\z\iter{k+1} &= \Pi_\cX\left(\x\iter{k+1} - \frac{1}{\rho}(-\y\iter{k} + \frac{1}{2}\P\x\iter{k+1})\right),\label{eq:admm-z}\\
	\y\iter{k+1} &= \y\iter{k} + \rho(\x\iter{k+1}  - \z\iter{k+1}).
\end{align}
We refer to the original paper \cite{lehuu2018continuous} for more details. In our experiments, we set $\rho = 1$ for simplicity. Since the expensive computation $\P\x$ are done two times in each ADMM iteration (one in \eqref{eq:admm-x}, another in \eqref{eq:admm-z}), this algorithm is roughly two times slower than the others. For a fair comparison, in our implementation we view \eqref{eq:admm-x} and \eqref{eq:admm-z} as two separate iterations (note that both $\x\iter{k+1}$ and $\z\iter{k+1}$ are feasible points).

Finally, we should note that the adaptive scheme for the penalty parameter $\rho$ proposed by \citet{lehuu2018continuous} is not applicable to our case, as we use only $5$ iterations in our experiments (which is equivalent to only $2.5$ regular iterations due to our above iteration separation).

\section{Detailed experimental setup and environment}\label{sec:experimental-setup-detailed}

\subsection{CNN-CRF architectures}\label{sec:model-details}
\paragraph{CNN-CRF}
Our segmentation model  is a standard combination of a CNN and a CRF~\cite{zheng2015conditional}. Given an input image $\Z\in\RR^{H\times W\times 3}$, the CNN produces an output $\Y\in\RR^{H\times W\times K}$ (where $K$ is the number of object classes) called the \emph{logits}, which is then fed into the CRF to produce a final output $\X\in\RR^{H\times W\times K}$:
\begin{equation}\label{eq:cnn-crf}
	\Y = \mathrm{CNN}(\Z;\thetab^u),\quad \X = \mathrm{CRF}(\Y; \thetab^p),
\end{equation}
where $\thetab^u$ and $\thetab^p$ are (typically trainable) parameters. The prediction is then obtained by taking the $\argmax$ along the last dimension of $\X$. For the CNN part, we consider two strong architectures: DeepLabv3 with ResNet101 backbone~\cite{chen2017rethinking}, and DeepLabv3+ with Xception65 backbone~\cite{chen2018encoder}. The reader is referred to the corresponding references for further details.

\paragraph{Dense CRF}
The CRF is defined over the input image such that each pixel is a node, and its labels are the object classes. Thus, using the notation defined in \S\ref{sec:inference-in-CRFs}, we have $n = HW$, $d=K$, and $\cS=\set{1,2,\dots,K}$. The CRF produces $\X$ in~\eqref{eq:cnn-crf} by minimizing the energy~\eqref{eq:energy-x} with appropriately constructed potentials, and then simply reshaping the solution $\x\in\RR^{HWK}$ into $H\times W\times K$. During training we skip the rounding step in CRF inference, so that the returned $\x$ is real-valued, which is more suitable for learning with the standard cross-entropy loss function. The unary potentials $\u$ is defined by to be the additive inverse of the logits $\Y$, reshaped correctly: $\u = -\vect (\Y)$, where $\vect$ denotes the flattening operator. We use the fully-connected model introduced by~\citet{krahenbuhl2011efficient} in which any pair of pixels $(i,j)$ is an edge with a pairwise potential of the form $\theta_{ij}(s,t) = \mu(s,t)k(\f_i,\f_j)\ \forall s,t\in\cS$, where $\mu:\cS\times\cS\to\RR$ is the so-called label compatibility function, and $k$ is a Gaussian kernel over some image features $(\f_i,\f_j)$. For a pixel $i$, we use its position $\p_i\in\NN^2$ and its color $\c_i\in[0,255]^3$ as features, and define the kernel as
\begin{equation}
	k(\f_i,\f_j) = w_1\exp\left(-\frac{\norm{\p_i - \p_j}_2^2}{2\alpha^2} - \frac{\norm{\c_i-\c_j}_2^2}{2\beta^2}\right) + w_2\exp\left(-\frac{\norm{\p_i-\p_j}_2^2}{2\gamma^2}\right)  \quad\forall i,j\in\cV,
\end{equation}
where $w_1,w_2$ are learnable kernel weights, and $\alpha,\beta,\gamma$ are hyperparameters. Following \citet{zheng2015conditional}, we use class-dependent kernel weights to increase the number of trainable parameters. Unlike \citet{zheng2015conditional}, for simplicity we use the default values $\alpha = 80, \beta = 13, \gamma = 3$ set by~\citet{krahenbuhl2011efficient} in all experiments, instead of doing a cross validation to find the best values. Finally, we use the Potts compatibility function: $\mu(s,t) = w\mathbbm{1}_{[s\neq t]}$ with $w=1$ for the inference experiments in \S\ref{sec:exp:inference}, and also for CRF initialization in the learning experiments in \S\ref{sec:learning-performance}.

\subsection{Datasets}

We provide further details on the datasets. PASCAL VOC~\cite{everingham2010pascal} contains \num{4369} images of \num{21} classes, split into \num{1464} (\emph{train}), \num{1449} (\emph{val}), and \num{1456} (\emph{test}) image subsets. As a standard practice, we augment the dataset with images from~\citet{hariharan2014simultaneous}, resulting in \num{10582} training images (\emph{trainaug}). Cityscapes~\cite{cordts2016cityscapes} contains \num{5000} images of \num{19} classes, split into \num{2975} (\emph{train}), \num{500} (\emph{val}), and \num{1525} (\emph{test}) image subsets. In addition, it also provides \num{19998} coarsely annotated images (\emph{train\_extra}). We report the performance in terms of mIoU across the semantic classes (\num{21} for \voc and \num{19} for Cityscapes).

\subsection{CNN training recipes} To fully train DeepLabv3 and DeepLabv3+, we follow closely the published recipes~\cite{chen2017rethinking,chen2018encoder} for this task. Below we present the most important information, and refer to the references for further details.

We first pretrain DeepLabv3 and DeepLabv3+ on the COCO~\cite{lin2014microsoft} dataset (by selecting only the images that contain the classes defined in \voc), and then finetune them on \voc (\emph{trainaug}) and Cityscapes (\emph{train}). During training, we apply data augmentation by (randomly) left-right flipping, scaling the input images (from $0.5$ to $2.0$), and cropping (with crop size of $513\times 513$ for \voc and $769\times 769$ for Cityscapes). We employ a \emph{poly} learning rate schedule: $\eta_m = \eta_0\left(1 - \frac{m}{M}\right)^{p}$,
where $\eta_0$ is the initial learning rate, $m$ is the step counter, and $M$ is the total number of training steps. For all trainings, we set $p=0.9$ and $\eta_0 = 0.001$ (except $\eta_0 = 0.01$ for pretraining on COCO), and a batch size of $16$ images. The value of $M$ is calculated from the number of training epochs, which we set to be \num{50} for COCO pretraining, \num{50} for finetuning on \voc, and \num{500} for finetuning on Cityscapes. We should note some differences compared to the original papers \cite{chen2017rethinking,chen2018encoder}. In particular, they did not specify the learning rate and number of steps for COCO pretraining. Furthermore, they used $\eta_0=0.0001$, which did not yield better results than $\eta_0=0.001$ in our implementation. Finally, in terms of the number of epochs, we used similar values to theirs. Indeed, they set \num{30000} and \num{90000} training steps for the \num{10582} and \num{2975} training images of \voc and Cityscapes, respectively. With a batch size of \num{16}, these are equivalent to \num{45} and \num{484} epochs. Table~\ref{tab:miou-val-reproduced} shows that our obtained results are similar to previous work \cite{chen2017rethinking,chen2018encoder}.

\begin{table}[!htb]
	\centering
	\begin{tabular}{llcc}
		\toprule
		Dataset & Model & Published \cite{chen2017rethinking,chen2018encoder} & Reproduced \\ \midrule
		\multirow{2}{*}{VOC} & DeepLabv3 & $78.51$ & $81.83$ \\ 
		& DeepLabv3+ & $82.45$ & $82.89$ \\ \midrule
		\multirow{2}{*}{Cityscapes} & DeepLabv3 & $77.82$ & $76.73$\\ 
		& DeepLabv3+ & $79.14$ & $79.55$\\
		\bottomrule\\
	\end{tabular}
	\caption{\label{tab:miou-val-reproduced}Performance of our reproduced DeepLab models compared to the original papers~\cite{chen2017rethinking,chen2018encoder}. The mIoU scores are obtained on the \emph{val} sets, without test time augmentation.}
\end{table}

\subsection{Training time and memory footprint} Our experiments are performed on a Linux server of 4 Nvidia V100 GPUs, using PyTorch 1.7. With a batch size of 16 on 4 GPUs (\ie, 4 images per GPU), both DeepLabv3 and DeepLabv3+ take ${\sim}7$min/epoch on PASCAL VOC (with $513\times 513$ crops). 
Thanks to our efficient GPU implementation (which will be made publicly available), plugging in the (5-step) CRF only increases that to ${\sim}9$ minutes ($1.2$--$1.3\times$ slower) for all inference methods (here we should note that CRF's running time is dominated by computing $\P\x$ at each step, which is why the running is similar across the methods). In terms of memory usage, DeepLabv3+ takes ${\sim}27.7$GB (per GPU for $4$ images) while DeepLabv3 takes ${\sim}15.6$GB. We found that the additional memory usage of the CRF (which has only $1323$ trainable parameters) are negligible for the Frank-Wolfe variants as well as for PGD, while PGM and ADMM require an extra amount of ${\sim}300$MB (probably due to the additional storage of the variable $\y$ at each iteration, see \S\ref{sec:derivation-peer-methods}).

\section{Additional results}

\subsection{Detailed results on the test sets}\label{sec:submission-urls}

The detailed results on the test sets can be found on the corresponding submission websites whose URLs are given in Table~\ref{tab:miou-test-sets-full}.

\begin{table}[!htb]
	\centering
	\begin{tabular}{l|l|l}
		\toprule
		Model & PASCAL VOC & Cityscapes\\ \midrule
		DeepLabv3+ \cite{chen2018encoder} & 87.8 & 82.1\\
		DeepLabv3+ (this work) & 87.6\textsuperscript{1} & $83.5\textsuperscript{3}$\\
		DeepLabv3+ with \lfw CRF & \textbf{88.0}\textsuperscript{2}  & 83.6\textsuperscript{4}\\
		\bottomrule
	\end{tabular}
	\begin{flushleft}
		$\textsuperscript{1}${\tiny \url{http://host.robots.ox.ac.uk:8080/anonymous/BUXULK.html}} \\
		$\textsuperscript{2}${\tiny \url{http://host.robots.ox.ac.uk:8080/anonymous/YFJJLW.html}} \\
		$\textsuperscript{3}${\tiny \url{https://www.cityscapes-dataset.com/anonymous-results/?id=845bd062fddae249ec0f4987d30f2f9be6e6716654513e7e6733d3f56e976532}} \\
		$\textsuperscript{4}${\tiny \url{https://www.cityscapes-dataset.com/anonymous-results/?id=84e788da7c55eeeb4840b70407ed665006494c99e5d34e0bd8704d66d9c8b864}}
	\end{flushleft}
	\caption{\label{tab:miou-test-sets-full}Performance on the \emph{test} sets.
	}
\end{table}
	
\subsection{Results for trainable $\alpha_k$ and $\lambda$}\label{sec:results-trainable-alpha}
We carried out an experiment with \lfw and \efw ($\lambda = 0.7$) in which we allow the stepsize $\alpha_k$ at each CRF iteration to be learnable (initialized at $0.5$). We observe the stepsizes at all the steps behave very similarly (\ie, increasing or decreasing together). In addition, for \efw they tend to increase during training, while for \lfw they tend to decrease. In addition, we also tried setting the regularization weight $\lambda$ to trainable. We initialized it at $1.0$ for \lfw and at $0.7$ for \efw. For both solvers, we found that $\lambda$ increased during training. Regarding accuracy, we did not observe significant differences compared to fixed $\alpha_k$ and fixed $\lambda$, although tuning the learning rates specifically for these variables could potentially lead to improved performance. See Table~\ref{tab:trainable-alpha} for the details.
	
	\begin{table}[!htb]
		\centering
		\begin{tabular}{llll}
			\toprule
			Regularizer & $\lambda$        & Stepsize      & mIoU         \\\toprule
			\multirow{4}{*}{$\ell_2$}          & 1.0 fixed     & 1.0 fixed     & 0.8490489721 \\
			 & 1.0 fixed     & 0.5 fixed     & 0.849458456  \\
			 & 1.0 fixed     & 0.5 learnable & 0.849185586  \\
		     & 1.0 learnable & 0.5 learnable & 0.8492224813 \\\midrule
			\multirow{4}{*}{Entropy}      & 0.7 fixed     & 1.0 fixed     & 0.8495011926 \\
			     & 0.7 fixed     & 0.5 fixed     & 0.8493972421 \\
			     & 0.7 fixed     & 0.5 learnable & 0.849845171  \\
			     & 0.7 learnable & 0.5 learnable & 0.8491678238\\
			\bottomrule\\
		\end{tabular}
	\caption{\label{tab:trainable-alpha}Comparison between trainable and fixed $\lambda$ and $\alpha_k$}
	\end{table}

\subsection{Results for fined-grained analysis}\label{sec:results-fine-grained}

We randomly picked a trained checkpoint among the different runs for DeepLabv3+ with \lfw ($\lambda=1.0$) and DeepLabv3+ with \efw ($\lambda=0.7$), and evaluated them using 5, 10, and 25 CRF iterations on the PASCAL VOC validation set. The results are shown in Table~\ref{tab:fine-grained-results}.

\begin{table}[!htb]
	\resizebox{\textwidth}{!}{
	\begin{tabular}{llllllllllllllllllllllll}
		\toprule
		Method     & Steps & mIoU  & background & aeroplane & bicycle & bird  & boat  & bottle & bus   & car   & cat   & chair & cow   & diningtable & dog   & horse & motorbike & person & pottedplant & sheep & sofa  & train & tvmonitor \\\toprule
		CNN &       & 82.89 & 95.79      & 91.80     & 44.89   & 89.92 & 71.49 & 83.54  & 94.68 & 91.54 & 95.42 & 52.36 & 95.51 & 70.25       & 93.63 & 93.08 & 88.27     & 90.20  & 68.03       & 92.62 & 66.95 & 92.33 & 78.45     \\\midrule
		\multirow{3}{*}{\lfw}     & 5     & 85.51 & 96.66      & 93.56     & 60.56   & 90.47 & 80.23 & 83.51  & 96.94 & 91.68 & 95.38 & 54.92 & 95.87 & 76.11       & 94.01 & 93.43 & 89.45     & 91.46  & 69.95       & 93.59 & 71.16 & 95.69 & 81.00     \\
		      & 10    & 85.52 & 96.67      & 93.56     & 60.49   & 90.47 & 80.23 & 83.53  & 96.92 & 91.68 & 95.38 & 55.11 & 95.87 & 76.16       & 94.00 & 93.41 & 89.43     & 91.45  & 69.98       & 93.57 & 71.39 & 95.67 & 80.95     \\
		      & 25    & 85.56 & 96.67      & 93.57     & 60.37   & 90.47 & 80.88 & 83.50  & 96.90 & 91.66 & 95.38 & 55.32 & 95.86 & 76.25       & 93.98 & 93.38 & 89.41     & 91.44  & 69.99       & 93.53 & 71.55 & 95.66 & 80.92     \\\midrule
		\multirow{3}{*}{\efw} & 5     & 84.55 & 96.33      & 94.88     & 55.66   & 90.74 & 75.54 & 83.63  & 95.58 & 89.60 & 94.71 & 54.33 & 95.93 & 75.78       & 93.84 & 93.14 & 91.21     & 91.02  & 69.56       & 92.96 & 69.87 & 90.60 & 80.65     \\
		 & 10    & 84.60 & 96.34      & 94.88     & 55.66   & 90.73 & 76.53 & 83.63  & 95.58 & 89.58 & 94.70 & 54.33 & 95.97 & 75.81       & 93.84 & 93.14 & 91.22     & 91.02  & 69.54       & 93.02 & 69.89 & 90.60 & 80.66     \\
		 & 25    & 84.55 & 96.33      & 94.88     & 55.66   & 90.73 & 75.47 & 83.63  & 95.58 & 89.60 & 94.70 & 54.33 & 95.97 & 75.81       & 93.84 & 93.14 & 91.22     & 91.02  & 69.54       & 93.02 & 69.89 & 90.60 & 80.66    \\
		\bottomrule\\
	\end{tabular}
}
\caption{\label{tab:fine-grained-results}Fined-grained results on PASCAL VOC validation set.}
\end{table}

\subsection{Additional inference results}\label{sec:inference-results-full}

\subsubsection{Results for longer inference regime}

We show in Figure~\ref{fig:energy-100steps} a comparison of the discrete energy across the methods on a subset of \num{10} \emph{val} images of PASCAL VOC for $100$ inference iterations, using DeepLabv3+ and Potts dense CRF. 

\begin{figure*}[!htb]
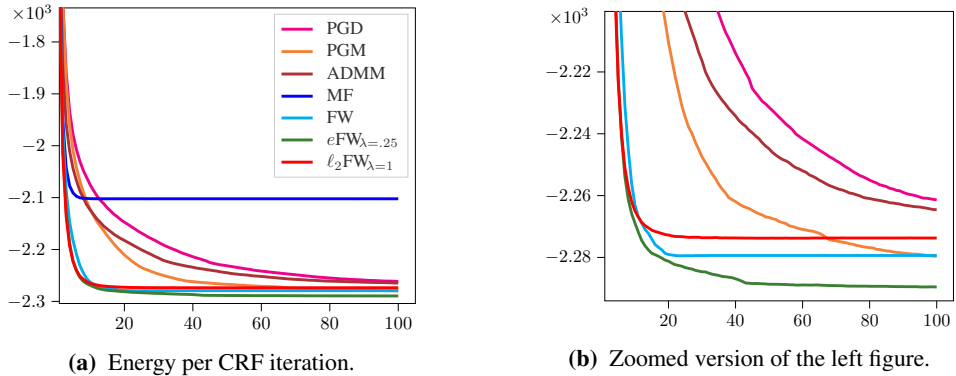

	\centering
	\begin{subfigure}{0.5\textwidth}
		\centering
		\resizebox{0.8\linewidth}{!}{
			\input{figures/energy_discrete_100steps_10samples.tex}
		}
		\caption{\label{fig:energy:comparison-100steps} Energy per CRF iteration.}
	\end{subfigure}%
	\begin{subfigure}{0.5\textwidth}
		\centering
		\resizebox{0.8\linewidth}{!}{
			\input{figures/energy_discrete_100steps_10samples_zoomed.tex}
		}
		\caption{\label{fig:energy:comparison-100steps-zoom} Zoomed version of the left figure.}
	\end{subfigure}%
	\caption{\label{fig:energy-100steps} \textbf{CRF energy} averaged over a subset of \num{10} \emph{val} images of \textbf{PASCAL VOC} using DeepLabv3+ and Potts dense CRF.
	}
\end{figure*}

From these results, we observe that:
\begin{enumerate}
	\item Vanilla FW and \lfw already converge after around $20$ iterations. \lfw does better than vanilla FW only in the early iterations.
	\item PGM surpasses \lfw at after $70$ iterations, and surpasses vanilla FW after $100$ iterations.
	\item PGD and ADMM are likely to surpass \lfw and vanilla, too, if given sufficient number of iterations, as these do not show any sign of convergence yet.
\end{enumerate}

The main observation here is that, the relative performance of the methods are different between the early (typically first $10$ iterations) and the later stage. In \S\ref{sec:conclusion} we gave some hypotheses on why the proposed regularized Frank-Wolfe may work better than the others. Our main argument is that vanilla Frank-Wolfe is already much better than the other methods (in the first few iterations), and what we do is to equip it with the ability of effectively learning with SGD (potential improvements in terms of energy are rather a byproduct and not the main objective, as the improvements are sometimes small). Let us summarize this situation as follows:

\begin{enumerate}
	\item Vanilla Frank-Wolfe outperforms other first-order methods such as PGD, PGM, and ADMM \textbf{during the first few iterations} (and may be surpassed at a later stage, as already shown).
	\item For SGD learning, in which only a small number of iterations (due to the vanishing/exploding gradient problems, as already observed in previous work~\cite{zheng2015conditional}), this behavior (reaching quickly a very low energy) of vanilla Frank-Wolfe is highly desirable.
	\item Unfortunately, vanilla Frank-Wolfe iterates are piecewise constant and thus the resulting gradients are zero almost everywhere, which makes learning through backpropagation impossible.
	\item Our regularized Frank-Wolfe is designed to precisely solve this zero-gradient issue.
\end{enumerate}

\subsubsection{Additional inference results}

Figures \ref{fig:energy-2} and \ref{fig:energy:cityscapes} show more results for the inference experiments.

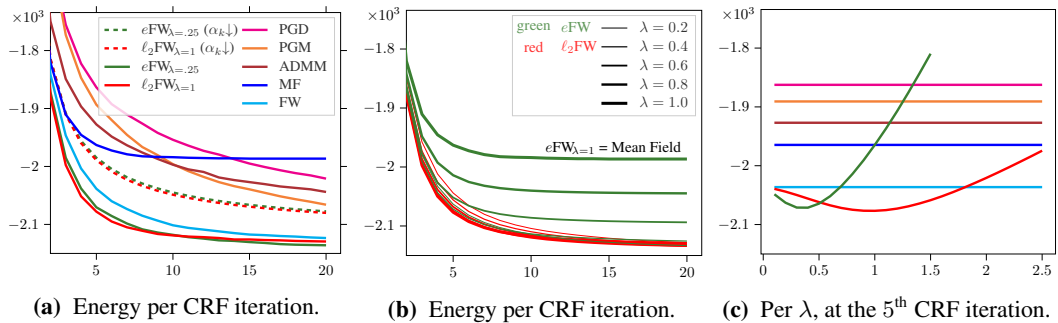
\begin{figure*}[!htb]
	\centering
	\begin{subfigure}{0.33\textwidth}
		\centering
		\resizebox{\linewidth}{!}{
			\begin{tikzpicture} %

\begin{axis}[
legend cell align={left},
legend style={fill opacity=0.8, draw opacity=1, text opacity=1, draw=white!80!black},
tick align=outside,
tick pos=both,
x grid style={white!69.0196078431373!black},
xmin=2, xmax=20.95,
xtick style={color=black},
y grid style={white!69.0196078431373!black},
ymin=-2150.075, ymax=-1734.425,
scaled y ticks={real:1000},
ytick scale label code/.code={$\times 10^3$},
every y tick scale label/.style={at={(yticklabel cs:0.86)},yshift=12.pt,xshift=2pt,anchor=south west},
ytick style={color=black},
legend columns=5,
transpose legend
]

\addplot [ultra thick, efwcolor, dashed]
table {%
	0 0
	1 -1533.00983436853
	2 -1814.93081435473
	3 -1908.11180124224
	4 -1955.88388543823
	5 -1984.89820565908
	6 -2004.38405797101
	7 -2018.63354037267
	8 -2029.44979296066
	9 -2037.93823326432
	10 -2044.92701863354
	11 -2050.65804002761
	12 -2055.43167701863
	13 -2059.54158040028
	14 -2063.09644582471
	15 -2066.24654934438
	16 -2069.02432712215
	17 -2071.51293995859
	18 -2073.72550034507
	19 -2075.70755693582
	20 -2077.52933057281
};
\addlegendentry{\efw[\lambda=.25] ($\alpha_k$$\downarrow$)}

\addplot [ultra thick, lfwcolor, dashed]
table {%
	0 0
	1 -1533.00845410628
	2 -1813.994478951
	3 -1910.52708764665
	4 -1959.86438923395
	5 -1989.2380952381
	6 -2008.96652864044
	7 -2023.09213250518
	8 -2033.70203588682
	9 -2042.04606625259
	10 -2048.81504485852
	11 -2054.27795031056
	12 -2058.89527260179
	13 -2062.92028985507
	14 -2066.38388543823
	15 -2069.39130434783
	16 -2072.05503795721
	17 -2074.41597653554
	18 -2076.53088336784
	19 -2078.44151138716
	20 -2080.12491373361
};
\addlegendentry{\lfw[\lambda=1] ($\alpha_k$$\downarrow$)}

\addplot [ultra thick, efwcolor]
table {%
	1 -1534.5
	2 -1871.5
	3 -1985
	4 -2038
	5 -2070.5
	6 -2085.5
	7 -2098.5
	8 -2108.5
	9 -2114
	10 -2118
	11 -2122
	12 -2124.5
	13 -2127.5
	14 -2128.5
	15 -2130.5
	16 -2132
	17 -2134
	18 -2135
	19 -2135.5
	20 -2136
};
\addlegendentry{\efw[\lambda=.25]}

\addplot [ultra thick, lfwcolor]
table {%
	1 -1534.5
	2 -1880
	3 -1997.5
	4 -2051.5
	5 -2078
	6 -2095
	7 -2105
	8 -2111
	9 -2116.5
	10 -2119
	11 -2121
	12 -2122
	13 -2123.5
	14 -2126
	15 -2126
	16 -2127
	17 -2127.5
	18 -2128.5
	19 -2129
	20 -2129.5
};
\addlegendentry{\lfw[\lambda=1]}

\addlegendimage{empty legend}
\addlegendentry{}

\addplot [ultra thick, pgdcolor]
table {%
	1 -1311.5
	2 -1628
	3 -1755.5
	4 -1822.5
	5 -1863.5
	6 -1892.5
	7 -1910.5
	8 -1926
	9 -1942
	10 -1954
	11 -1964
	12 -1973
	13 -1979.5
	14 -1987.5
	15 -1995
	16 -2001
	17 -2006.5
	18 -2011.5
	19 -2015.5
	20 -2021
};
\addlegendentry{PGD}

\addplot [ultra thick, pgmcolor]
table {%
1 -1311.5
2 -1628
3 -1773
4 -1847
5 -1893.5
6 -1920
7 -1945.5
8 -1966.5
9 -1980.5
10 -1996
11 -2008
12 -2018
13 -2027
14 -2034
15 -2039.5
16 -2045.5
17 -2051.5
18 -2056.5
19 -2061
20 -2066
};
\addlegendentry{PGM}

\addplot [ultra thick, admmcolor]
table {%
	1 -1495
	2 -1749.5
	3 -1849
	4 -1896.5
	5 -1927.5
	6 -1948.5
	7 -1963.5
	8 -1976
	9 -1987
	10 -1997
	11 -2005
	12 -2009.5
	13 -2019
	14 -2023.5
	15 -2027
	16 -2030.5
	17 -2034.5
	18 -2037
	19 -2039.5
	20 -2044
};
\addlegendentry{ADMM}

\addplot [ultra thick, mfcolor]
table {%
1 -1534.5
2 -1820.5
3 -1910.5
4 -1945.5
5 -1963
6 -1972.5
7 -1979
8 -1982.5
9 -1983.5
10 -1984
11 -1985
12 -1985.5
13 -1986
14 -1986
15 -1986.5
16 -1986.5
17 -1986.5
18 -1986.5
19 -1986.5
20 -1986.5
};
\addlegendentry{MF}

\addplot [ultra thick, fwcolor]
table {%
	1 -1534.5
	2 -1841.5
	3 -1946.5
	4 -2003.5
	5 -2038.5
	6 -2060
	7 -2073.5
	8 -2085
	9 -2093.5
	10 -2101.5
	11 -2106
	12 -2109
	13 -2112
	14 -2115
	15 -2117.5
	16 -2118.5
	17 -2120
	18 -2121
	19 -2122.5
	20 -2123.5
};
\addlegendentry{FW}

\end{axis}

\end{tikzpicture}
		}
		\caption{\label{fig:energy:comparison-2} Energy per CRF iteration.}
	\end{subfigure}%
	\begin{subfigure}{0.33\textwidth}
		\centering
		\resizebox{\linewidth}{!}{
			\begin{tikzpicture}

\begin{axis}[
legend cell align={left},
legend style={fill opacity=0.8, draw opacity=1, text opacity=1, draw=white!80!black},
legend entries={{\textcolor{efwcolor}{\efw}}, {\textcolor{lfwcolor}{\lfw}}, { }, { }, { }, {$\lambda = 0.2$},{$\lambda = 0.4$},{$\lambda = 0.6$},{$\lambda = 0.8$}, {$\lambda = 1.0$}},
legend columns=5,
transpose legend,
tick align=outside,
tick pos=left,
x grid style={white!69.0196078431373!black},
xmin=2, xmax=20.95,
xtick style={color=black},
y grid style={white!69.0196078431373!black},
ymin=-2150.075, ymax=-1734.425,
scaled y ticks={real:1000},
ytick scale label code/.code={$\times 10^3$},
every y tick scale label/.style={at={(yticklabel cs:0.86)},yshift=12.pt,xshift=2pt,anchor=south west},
ytick style={color=black}
]

\addlegendimage{dash pattern=on 0pt off 3pt, mark=text, text mark=green, efwcolor}
\addlegendimage{dash pattern=on 0pt off 3pt, mark=text, text mark=red, lfwcolor}
\addlegendimage{empty legend}
\addlegendimage{empty legend}
\addlegendimage{empty legend}

\addlegendimage{black}
\addlegendimage{thick, black}
\addlegendimage{very thick, black}
\addlegendimage{ultra thick, black}
\addlegendimage{line width=2pt, black}

\addplot [lfwcolor]
table {%
	0 0
	1 -1532.6635610766
	2 -1849.73878536922
	3 -1957.61680469289
	4 -2011.8262594893
	5 -2044.06763285024
	6 -2065.42063492063
	7 -2080.19013112491
	8 -2090.5864389234
	9 -2098.23964803313
	10 -2104.19755003451
	11 -2108.77104899931
	12 -2112.53105590062
	13 -2115.62146307799
	14 -2118.31780538302
	15 -2120.75241545894
	16 -2122.71290545204
	17 -2124.40907522429
	18 -2125.86111111111
	19 -2127.14078674948
	20 -2128.24120082816
};

\addplot [thick, lfwcolor]
table {%
	0 0
	1 -1532.66235334714
	2 -1863.60766045549
	3 -1972.71273291925
	4 -2025.45065562457
	5 -2055.88440303658
	6 -2075.47049689441
	7 -2088.56694271912
	8 -2097.76224982747
	9 -2104.52501725328
	10 -2109.76052449965
	11 -2114.06383712905
	12 -2117.51207729469
	13 -2120.46100759144
	14 -2122.79002760525
	15 -2124.73033126294
	16 -2126.41770186335
	17 -2127.90148378192
	18 -2129.12939958592
	19 -2130.12767425811
	20 -2131.05417529331
};

\addplot [very thick, lfwcolor]
table {%
	0 0
	1 -1532.66390614217
	2 -1872.90338164251
	3 -1986.24465148378
	4 -2038.43823326432
	5 -2067.36870255349
	6 -2085.20272601794
	7 -2096.89648033126
	8 -2105.02881297446
	9 -2111.1488957902
	10 -2115.84351276743
	11 -2119.60541752933
	12 -2122.64510006901
	13 -2125.08298826777
	14 -2127.08385093168
	15 -2128.72584541063
	16 -2130.11956521739
	17 -2131.28571428571
	18 -2132.20410628019
	19 -2133.1718426501
	20 -2133.99775707384
};

\addplot [ultra thick, lfwcolor]
table {%
	0 0
	1 -1532.66442374051
	2 -1877.57367149758
	3 -1994.51552795031
	4 -2046.94634230504
	5 -2074.8630089717
	6 -2091.40269151139
	7 -2102.08833678399
	8 -2109.66959972395
	9 -2115.19064872326
	10 -2119.41218081435
	11 -2122.56055900621
	12 -2124.88940648723
	13 -2126.80728088337
	14 -2128.38267770876
	15 -2129.68253968254
	16 -2130.80641821946
	17 -2131.72394755003
	18 -2132.40596963423
	19 -2133.03536922015
	20 -2133.61456176674
};

\addplot [line width=2pt, lfwcolor]
table {%
	0 0
	1 -1532.66287094548
	2 -1878.45790200138
	3 -1997.17443064182
	4 -2049.54002760525
	5 -2076.6856452726
	6 -2092.52329192547
	7 -2102.82729468599
	8 -2109.70859213251
	9 -2114.3309178744
	10 -2117.56262939959
	11 -2120.06487232574
	12 -2122.0248447205
	13 -2123.64164941339
	14 -2124.97739820566
	15 -2126.11335403727
	16 -2127.02070393375
	17 -2127.85697032436
	18 -2128.5762594893
	19 -2129.1231884058
	20 -2129.56418219462
};

\addplot [efwcolor]
table {%
	0 0
	1 -1532.66252587992
	2 -1866.56504485852
	3 -1979.24413388544
	4 -2032.70652173913
	5 -2062.93995859213
	6 -2081.67805383023
	7 -2094.06331953071
	8 -2102.62939958592
	9 -2109.03554175293
	10 -2114.07988267771
	11 -2118.02070393375
	12 -2121.26035196687
	13 -2123.91494133885
	14 -2126.016563147
	15 -2127.7270531401
	16 -2129.20410628019
	17 -2130.46566597654
	18 -2131.53778467909
	19 -2132.43961352657
	20 -2133.26483781919
};

\addplot [thick, efwcolor]
table {%
	0 0
	1 -1532.66407867495
	2 -1871.83592132505
	3 -1990.32953761215
	4 -2043.56004140787
	5 -2071.32125603865
	6 -2087.38474810214
	7 -2097.75845410628
	8 -2104.96583850932
	9 -2109.94289164941
	10 -2113.46928916494
	11 -2115.92874396135
	12 -2117.91131815045
	13 -2119.48550724638
	14 -2120.75724637681
	15 -2121.88716356108
	16 -2122.8345410628
	17 -2123.62715665977
	18 -2124.34955141477
	19 -2125.00138026225
	20 -2125.58661145618
};

\addplot [very thick, efwcolor]
table {%
	0 0
	1 -1532.66407867495
	2 -1861.44496204279
	3 -1976.73274672188
	4 -2026.80572808834
	5 -2051.82522429262
	6 -2065.75224292616
	7 -2074.233436853
	8 -2079.80072463768
	9 -2083.50931677019
	10 -2086.09989648033
	11 -2087.89113181504
	12 -2089.24861973775
	13 -2090.30365769496
	14 -2091.19099378882
	15 -2091.87439613527
	16 -2092.40251897861
	17 -2092.85852311939
	18 -2093.23585231194
	19 -2093.53433402346
	20 -2093.80745341615
};

\addplot [ultra thick, efwcolor]
table {%
	0 0
	1 -1532.66252587992
	2 -1842.38992408558
	3 -1948.11732229124
	4 -1992.56935817805
	5 -2014.00189786059
	6 -2025.3676673568
	7 -2031.86507936508
	8 -2035.86162870945
	9 -2038.3850931677
	10 -2040.11714975845
	11 -2041.29951690821
	12 -2042.17287784679
	13 -2042.79761904762
	14 -2043.30693581781
	15 -2043.68374741201
	16 -2043.98636991028
	17 -2044.23464458247
	18 -2044.41890959282
	19 -2044.56055900621
	20 -2044.70169082126
};

\addplot [line width=2pt, efwcolor]
table {%
	0 0
	1 -1534.5
	2 -1820.5
	3 -1910.5
	4 -1945.5
	5 -1963
	6 -1972.5
	7 -1979
	8 -1982.5
	9 -1983.5
	10 -1984
	11 -1985
	12 -1985.5
	13 -1986
	14 -1986
	15 -1986.5
	16 -1986.5
	17 -1986.5
	18 -1986.5
	19 -1986.5
	20 -1986.5
};
\end{axis}

\node (mf) at (4.8,2.5) {\efw[\lambda = 1] = Mean Field};

\end{tikzpicture}
		}
		\caption{\label{fig:energy:lambdas} Energy per CRF iteration.}
	\end{subfigure}%
	\begin{subfigure}{0.33\textwidth}
		\centering
		\resizebox{\linewidth}{!}{
			\begin{tikzpicture}

\begin{axis}[
legend cell align={left},
legend style={fill opacity=0.8, draw opacity=1, text opacity=1, draw=white!80!black},
tick align=outside,
tick pos=left,
x grid style={white!69.0196078431373!black},
xmin=-0.02, xmax=2.62,
xtick style={color=black},
ymin=-2150.075, ymax=-1734.425,
scaled y ticks={real:1000},
ytick scale label code/.code={$\times 10^3$},
every y tick scale label/.style={at={(yticklabel cs:0.86)},yshift=12.pt,xshift=2pt,anchor=south west},
ytick style={color=black}
]

\addplot [ultra thick, pgmcolor]
table {%
0.1 -1890.73947550035
0.2 -1890.73947550035
0.3 -1890.73947550035
0.4 -1890.73947550035
0.5 -1890.73947550035
0.6 -1890.73947550035
0.7 -1890.73947550035
0.8 -1890.73947550035
0.9 -1890.73947550035
1 -1890.73947550035
1.1 -1890.73947550035
1.2 -1890.73947550035
1.3 -1890.73947550035
1.4 -1890.73947550035
1.5 -1890.73947550035
1.6 -1890.73947550035
1.7 -1890.73947550035
1.8 -1890.73947550035
1.9 -1890.73947550035
2 -1890.73947550035
2.1 -1890.73947550035
2.2 -1890.73947550035
2.3 -1890.73947550035
2.4 -1890.73947550035
2.5 -1890.73947550035
};

\addplot [ultra thick, mfcolor]
table {%
0.1 -1964.62525879917
0.2 -1964.62525879917
0.3 -1964.62525879917
0.4 -1964.62525879917
0.5 -1964.62525879917
0.6 -1964.62525879917
0.7 -1964.62525879917
0.8 -1964.62525879917
0.9 -1964.62525879917
1 -1964.62525879917
1.1 -1964.62525879917
1.2 -1964.62525879917
1.3 -1964.62525879917
1.4 -1964.62525879917
1.5 -1964.62525879917
1.6 -1964.62525879917
1.7 -1964.62525879917
1.8 -1964.62525879917
1.9 -1964.62525879917
2 -1964.62525879917
2.1 -1964.62525879917
2.2 -1964.62525879917
2.3 -1964.62525879917
2.4 -1964.62525879917
2.5 -1964.62525879917
};

\addplot [ultra thick, admmcolor]
table {%
0.1 -1926.95358868185
0.2 -1926.95358868185
0.3 -1926.95358868185
0.4 -1926.95358868185
0.5 -1926.95358868185
0.6 -1926.95358868185
0.7 -1926.95358868185
0.8 -1926.95358868185
0.9 -1926.95358868185
1 -1926.95358868185
1.1 -1926.95358868185
1.2 -1926.95358868185
1.3 -1926.95358868185
1.4 -1926.95358868185
1.5 -1926.95358868185
1.6 -1926.95358868185
1.7 -1926.95358868185
1.8 -1926.95358868185
1.9 -1926.95358868185
2 -1926.95358868185
2.1 -1926.95358868185
2.2 -1926.95358868185
2.3 -1926.95358868185
2.4 -1926.95358868185
2.5 -1926.95358868185
};

\addplot [ultra thick, pgdcolor]
table {%
0.1 -1862.48723257419
0.2 -1862.48723257419
0.3 -1862.48723257419
0.4 -1862.48723257419
0.5 -1862.48723257419
0.6 -1862.48723257419
0.7 -1862.48723257419
0.8 -1862.48723257419
0.9 -1862.48723257419
1 -1862.48723257419
1.1 -1862.48723257419
1.2 -1862.48723257419
1.3 -1862.48723257419
1.4 -1862.48723257419
1.5 -1862.48723257419
1.6 -1862.48723257419
1.7 -1862.48723257419
1.8 -1862.48723257419
1.9 -1862.48723257419
2 -1862.48723257419
2.1 -1862.48723257419
2.2 -1862.48723257419
2.3 -1862.48723257419
2.4 -1862.48723257419
2.5 -1862.48723257419
};

\addplot [ultra thick, fwcolor]
table {%
0.1 -2036.38008971705
0.2 -2036.38008971705
0.3 -2036.38008971705
0.4 -2036.38008971705
0.5 -2036.38008971705
0.6 -2036.38008971705
0.7 -2036.38008971705
0.8 -2036.38008971705
0.9 -2036.38008971705
1 -2036.38008971705
1.1 -2036.38008971705
1.2 -2036.38008971705
1.3 -2036.38008971705
1.4 -2036.38008971705
1.5 -2036.38008971705
1.6 -2036.38008971705
1.7 -2036.38008971705
1.8 -2036.38008971705
1.9 -2036.38008971705
2 -2036.38008971705
2.1 -2036.38008971705
2.2 -2036.38008971705
2.3 -2036.38008971705
2.4 -2036.38008971705
2.5 -2036.38008971705
};

\addplot [ultra thick, lfwcolor]
table {%
	0.1 -2039.2922705314
	0.2 -2044.06763285024
	0.3 -2049.78071083506
	0.4 -2055.88440303658
	0.5 -2061.90648723257
	0.6 -2067.36870255349
	0.7 -2071.81193926846
	0.8 -2074.8630089717
	0.9 -2076.52777777778
	1 -2076.6856452726
	1.1 -2075.47860593513
	1.2 -2072.9218426501
	1.3 -2069.11473429952
	1.4 -2064.23964803313
	1.5 -2058.61093857833
	1.6 -2052.14406487233
	1.7 -2044.97998619738
	1.8 -2037.25138026225
	1.9 -2029.11059351277
	2 -2020.68668046929
	2.1 -2011.8897515528
	2.2 -2002.85058661146
	2.3 -1993.56452726018
	2.4 -1984.13854382333
	2.5 -1974.4641131815
};

\addplot [ultra thick, efwcolor]
table {%
	0.1 -2048.60334713596
	0.2 -2062.93995859213
	0.3 -2071.13940648723
	0.4 -2071.32125603865
	0.5 -2064.39734299517
	0.6 -2051.82522429262
	0.7 -2034.71825396825
	0.8 -2014.00189786059
	0.9 -1990.58229813665
	1 -1964.6231884058
	1.1 -1936.92908902692
	1.2 -1907.31331953071
	1.3 -1875.9135610766
	1.4 -1843.2472394755
	1.5 -1809.56142167012
};
\end{axis}

\end{tikzpicture}
		}
		\caption{\label{fig:energy:relative-2} Per $\lambda$, at the $5\textsuperscript{th}$ CRF iteration.}
	\end{subfigure}%
	\caption{\label{fig:energy-2} \textbf{CRF energy} averaged over \num{1449} \emph{val} images of \textbf{PASCAL VOC} using DeepLabv3+ and Potts dense CRF. \textbf{(a)} Comparison between regularized Frank-Wolfe and the other methods for some selected values of the regularization weight $\lambda$. \textbf{(b)} Results of regularized Frank-Wolfe for different values of $\lambda$. \textbf{(c)} Energy per $\lambda$ after $5$ iterations.	Best viewed in color. %
	}
\end{figure*}

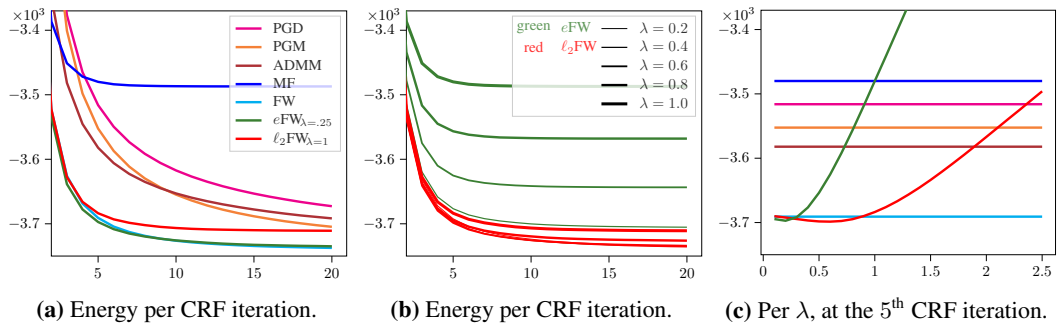
\begin{figure*}[!htb]
	\centering
	\begin{subfigure}{0.33\textwidth}
		\centering
		\resizebox{\linewidth}{!}{
			\begin{tikzpicture}

\definecolor{color0}{rgb}{1,0,1}
\definecolor{color1}{rgb}{1,0.647058823529412,0}
\definecolor{color2}{rgb}{0.803921568627451,0.52156862745098,0.247058823529412}
\definecolor{color3}{rgb}{0,1,1}

\begin{axis}[
legend cell align={left},
legend style={fill opacity=0.8, draw opacity=1, text opacity=1, draw=white!80!black},
tick align=outside,
tick pos=left,
x grid style={white!69.0196078431373!black},
xmin=2, xmax=20.95,
xtick style={color=black},
y grid style={white!69.0196078431373!black},
ymin=-3750, ymax=-3370,
scaled y ticks={real:1000},
ytick scale label code/.code={$\times 10^3$},
every y tick scale label/.style={at={(yticklabel cs:0.86)},yshift=12.pt,xshift=2pt,anchor=south west},
ytick style={color=black}
]

\addplot [ultra thick, pgdcolor]
table {%
0 0
1 -2701.58
2 -3195.522
3 -3376.772
4 -3464.964
5 -3516.086
6 -3549.532
7 -3573.412
8 -3591.524
9 -3605.698
10 -3617.282
11 -3626.782
12 -3634.798
13 -3641.804
14 -3647.858
15 -3653.21
16 -3657.958
17 -3662.18
18 -3666.132
19 -3669.592
20 -3672.736
};
\addlegendentry{PGD}

\addplot [ultra thick, pgmcolor]
table {%
0 0
1 -2701.578
2 -3195.524
3 -3400.666
4 -3498.054
5 -3552.424
6 -3587.066
7 -3611.438
8 -3629.712
9 -3643.73
10 -3654.78
11 -3663.89
12 -3671.668
13 -3678.082
14 -3683.566
15 -3688.336
16 -3692.414
17 -3696.116
18 -3699.268
19 -3702.132
20 -3704.656
};
\addlegendentry{PGM}

\addplot [ultra thick, admmcolor]
table {%
0 0
1 -2951.142
2 -3348.008
3 -3481.624
4 -3545.284
5 -3582.094
6 -3605.92
7 -3622.694
8 -3635.012
9 -3644.882
10 -3652.962
11 -3659.518
12 -3665.024
13 -3669.868
14 -3674.034
15 -3677.822
16 -3681.206
17 -3684.274
18 -3687.042
19 -3689.326
20 -3691.624
};
\addlegendentry{ADMM}

\addplot [ultra thick, mfcolor]
table {%
0 0
1 -3085.482
2 -3386.532
3 -3451.18
4 -3471.962
5 -3479.988
6 -3483.744
7 -3485.338
8 -3486.06
9 -3486.516
10 -3486.74
11 -3486.878
12 -3486.992
13 -3487.038
14 -3487.06
15 -3487.072
16 -3487.064
17 -3487.092
18 -3487.098
19 -3487.098
20 -3487.08
};
\addlegendentry{MF}

\addplot [ultra thick, fwcolor]
table {%
0 0
1 -3085.488
2 -3522.772
3 -3625.746
4 -3668.572
5 -3690.822
6 -3704.314
7 -3712.748
8 -3718.728
9 -3723.08
10 -3726.258
11 -3728.718
12 -3730.546
13 -3732.07
14 -3733.294
15 -3734.372
16 -3735.176
17 -3735.8
18 -3736.388
19 -3736.926
20 -3737.376
};
\addlegendentry{FW}

\addplot [ultra thick, efwcolor]
table {%
0 0
1 -3085.494
2 -3534.184
3 -3638.606
4 -3677.696
5 -3697.198
6 -3708.306
7 -3715.106
8 -3719.714
9 -3723.116
10 -3725.774
11 -3727.626
12 -3729.354
13 -3730.736
14 -3731.762
15 -3732.612
16 -3733.132
17 -3733.694
18 -3734.152
19 -3734.462
20 -3734.886
};
\addlegendentry{\efw[\lambda=.25]}

\addplot [ultra thick, lfwcolor]
table {%
0 0
1 -3085.486
2 -3522.246
3 -3627.604
4 -3666
5 -3683.654
6 -3693.074
7 -3698.482
8 -3702.064
9 -3704.774
10 -3706.54
11 -3707.808
12 -3708.62
13 -3709.128
14 -3709.532
15 -3709.904
16 -3710.206
17 -3710.448
18 -3710.57
19 -3710.702
20 -3710.778
};
\addlegendentry{\lfw[\lambda=1]}

\end{axis}

\end{tikzpicture}
		}
		\caption{\label{fig:energy:cityscapes:comparison}Energy per CRF iteration.}
	\end{subfigure}%
	\begin{subfigure}{0.33\textwidth}
		\centering
		\resizebox{\linewidth}{!}{
			\begin{tikzpicture}

\begin{axis}[
legend cell align={left},
legend style={fill opacity=0.8, draw opacity=1, text opacity=1, draw=white!80!black},
legend entries={{\textcolor{efwcolor}{\efw}}, {\textcolor{lfwcolor}{\lfw}}, { }, { }, { }, {$\lambda = 0.2$},{$\lambda = 0.4$},{$\lambda = 0.6$},{$\lambda = 0.8$}, {$\lambda = 1.0$}},
legend columns=5,
transpose legend,
tick align=outside,
tick pos=left,
x grid style={white!69.0196078431373!black},
xmin=2, xmax=20.95,
xtick style={color=black},
y grid style={white!69.0196078431373!black},
ymin=-3750, ymax=-3370,
scaled y ticks={real:1000},
ytick scale label code/.code={$\times 10^3$},
every y tick scale label/.style={at={(yticklabel cs:0.86)},yshift=12.pt,xshift=2pt,anchor=south west},
ytick style={color=black}
]

\addlegendimage{dash pattern=on 0pt off 3pt, mark=text, text mark=green, efwcolor}
\addlegendimage{dash pattern=on 0pt off 3pt, mark=text, text mark=red, lfwcolor}
\addlegendimage{empty legend}
\addlegendimage{empty legend}
\addlegendimage{empty legend}

\addlegendimage{black}
\addlegendimage{thick, black}
\addlegendimage{very thick, black}
\addlegendimage{ultra thick, black}
\addlegendimage{line width=2pt, black}

\addplot [efwcolor]
table {%
0 0
1 -3085.494
2 -3534.184
3 -3638.606
4 -3677.696
5 -3697.198
6 -3708.306
7 -3715.106
8 -3719.714
9 -3723.116
10 -3725.774
11 -3727.626
12 -3729.354
13 -3730.736
14 -3731.762
15 -3732.612
16 -3733.132
17 -3733.694
18 -3734.152
19 -3734.462
20 -3734.886
};
\addplot [thick, efwcolor]
table {%
0 0
1 -3085.482
2 -3514.298
3 -3619.508
4 -3658.284
5 -3676.606
6 -3686.244
7 -3691.988
8 -3695.566
9 -3698.104
10 -3700.178
11 -3701.506
12 -3702.476
13 -3703.152
14 -3703.728
15 -3704.148
16 -3704.438
17 -3704.69
18 -3704.94
19 -3705.122
20 -3705.278
};
\addplot [very thick, efwcolor]
table {%
0 0
1 -3085.48
2 -3477.768
3 -3574.98
4 -3609.922
5 -3625.198
6 -3632.626
7 -3636.626
8 -3638.988
9 -3640.518
10 -3641.484
11 -3642.208
12 -3642.596
13 -3642.86
14 -3643.07
15 -3643.23
16 -3643.31
17 -3643.4
18 -3643.436
19 -3643.468
20 -3643.454
};
\addplot [ultra thick, efwcolor]
table {%
0 0
1 -3085.484
2 -3434.068
3 -3516.608
4 -3544.892
5 -3556.59
6 -3561.902
7 -3564.59
8 -3566.158
9 -3566.684
10 -3567.18
11 -3567.49
12 -3567.672
13 -3567.8
14 -3567.892
15 -3567.9
16 -3567.92
17 -3567.928
18 -3567.938
19 -3567.958
20 -3567.956
};
\addplot [line width=2pt, efwcolor]
table {%
0 0
1 -3085.492
2 -3386.524
3 -3451.178
4 -3471.964
5 -3479.99
6 -3483.744
7 -3485.346
8 -3486.05
9 -3486.508
10 -3486.748
11 -3486.848
12 -3487.008
13 -3487.036
14 -3487.062
15 -3487.068
16 -3487.09
17 -3487.09
18 -3487.09
19 -3487.094
20 -3487.108
};
\addplot [lfwcolor]
table {%
0 0
1 -3085.48
2 -3530.728
3 -3631.764
4 -3672.082
5 -3692.252
6 -3704.196
7 -3711.852
8 -3717.2
9 -3720.994
10 -3723.846
11 -3726.324
12 -3727.974
13 -3729.61
14 -3730.812
15 -3731.794
16 -3732.696
17 -3733.444
18 -3734.13
19 -3734.808
20 -3735.27
};
\addplot [thick, lfwcolor]
table {%
0 0
1 -3085.488
2 -3536.822
3 -3638.632
4 -3677.136
5 -3696.786
6 -3707.752
7 -3715.034
8 -3719.704
9 -3722.996
10 -3725.854
11 -3727.844
12 -3729.484
13 -3730.83
14 -3732.094
15 -3732.988
16 -3733.816
17 -3734.456
18 -3734.942
19 -3735.476
20 -3735.814
};
\addplot [very thick, lfwcolor]
table {%
0 0
1 -3085.48
2 -3536.944
3 -3640.996
4 -3679.778
5 -3698.59
6 -3709.036
7 -3715.536
8 -3719.834
9 -3723.13
10 -3725.586
11 -3727.408
12 -3728.898
13 -3730.122
14 -3731.082
15 -3731.83
16 -3732.322
17 -3732.826
18 -3733.242
19 -3733.546
20 -3733.952
};
\addplot [ultra thick, lfwcolor]
table {%
0 0
1 -3085.492
2 -3531.526
3 -3637.144
4 -3676.194
5 -3694.76
6 -3704.53
7 -3710.612
8 -3714.582
9 -3717.486
10 -3719.538
11 -3721.374
12 -3722.6
13 -3723.426
14 -3724.032
15 -3724.582
16 -3725.006
17 -3725.268
18 -3725.578
19 -3725.838
20 -3726.13
};
\addplot [line width=2pt, lfwcolor]
table {%
0 0
1 -3085.486
2 -3522.246
3 -3627.604
4 -3666
5 -3683.654
6 -3693.074
7 -3698.482
8 -3702.064
9 -3704.774
10 -3706.54
11 -3707.808
12 -3708.62
13 -3709.128
14 -3709.532
15 -3709.904
16 -3710.206
17 -3710.448
18 -3710.57
19 -3710.702
20 -3710.778
};
\end{axis}

\end{tikzpicture}
		}
		\caption{\label{fig:energy:cityscapes:lambdas}Energy per CRF iteration.}
	\end{subfigure}%
	\begin{subfigure}{0.33\textwidth}
		\centering
		\resizebox{\linewidth}{!}{
			\begin{tikzpicture}

\definecolor{color0}{rgb}{1,0,1}
\definecolor{color1}{rgb}{1,0.647058823529412,0}
\definecolor{color2}{rgb}{0.803921568627451,0.52156862745098,0.247058823529412}
\definecolor{color3}{rgb}{0,1,1}

\begin{axis}[
legend cell align={left},
legend style={fill opacity=0.8, draw opacity=1, text opacity=1, draw=white!80!black},
tick align=outside,
tick pos=left,
x grid style={white!69.0196078431373!black},
xmin=-0.02, xmax=2.62,
xtick style={color=black},
y grid style={white!69.0196078431373!black},
ymin=-3750, ymax=-3370,
scaled y ticks={real:1000},
ytick scale label code/.code={$\times 10^3$},
every y tick scale label/.style={at={(yticklabel cs:0.86)},yshift=12.pt,xshift=2pt,anchor=south west},
ytick style={color=black}
]

\addplot [ultra thick, pgdcolor]
table {%
0.1 -3516.086
0.2 -3516.086
0.3 -3516.086
0.4 -3516.086
0.5 -3516.086
0.6 -3516.086
0.7 -3516.086
0.8 -3516.086
0.9 -3516.086
1 -3516.086
1.1 -3516.086
1.2 -3516.086
1.3 -3516.086
1.4 -3516.086
1.5 -3516.086
1.6 -3516.086
1.7 -3516.086
1.8 -3516.086
1.9 -3516.086
2 -3516.086
2.1 -3516.086
2.2 -3516.086
2.3 -3516.086
2.4 -3516.086
2.5 -3516.086
};
\addplot [ultra thick, fwcolor]
table {%
0.1 -3690.822
0.2 -3690.822
0.3 -3690.822
0.4 -3690.822
0.5 -3690.822
0.6 -3690.822
0.7 -3690.822
0.8 -3690.822
0.9 -3690.822
1 -3690.822
1.1 -3690.822
1.2 -3690.822
1.3 -3690.822
1.4 -3690.822
1.5 -3690.822
1.6 -3690.822
1.7 -3690.822
1.8 -3690.822
1.9 -3690.822
2 -3690.822
2.1 -3690.822
2.2 -3690.822
2.3 -3690.822
2.4 -3690.822
2.5 -3690.822
};
\addplot [ultra thick, admmcolor]
table {%
0.1 -3582.094
0.2 -3582.094
0.3 -3582.094
0.4 -3582.094
0.5 -3582.094
0.6 -3582.094
0.7 -3582.094
0.8 -3582.094
0.9 -3582.094
1 -3582.094
1.1 -3582.094
1.2 -3582.094
1.3 -3582.094
1.4 -3582.094
1.5 -3582.094
1.6 -3582.094
1.7 -3582.094
1.8 -3582.094
1.9 -3582.094
2 -3582.094
2.1 -3582.094
2.2 -3582.094
2.3 -3582.094
2.4 -3582.094
2.5 -3582.094
};
\addplot [ultra thick, mfcolor]
table {%
0.1 -3479.988
0.2 -3479.988
0.3 -3479.988
0.4 -3479.988
0.5 -3479.988
0.6 -3479.988
0.7 -3479.988
0.8 -3479.988
0.9 -3479.988
1 -3479.988
1.1 -3479.988
1.2 -3479.988
1.3 -3479.988
1.4 -3479.988
1.5 -3479.988
1.6 -3479.988
1.7 -3479.988
1.8 -3479.988
1.9 -3479.988
2 -3479.988
2.1 -3479.988
2.2 -3479.988
2.3 -3479.988
2.4 -3479.988
2.5 -3479.988
};
\addplot [ultra thick, pgmcolor]
table {%
0.1 -3552.424
0.2 -3552.424
0.3 -3552.424
0.4 -3552.424
0.5 -3552.424
0.6 -3552.424
0.7 -3552.424
0.8 -3552.424
0.9 -3552.424
1 -3552.424
1.1 -3552.424
1.2 -3552.424
1.3 -3552.424
1.4 -3552.424
1.5 -3552.424
1.6 -3552.424
1.7 -3552.424
1.8 -3552.424
1.9 -3552.424
2 -3552.424
2.1 -3552.424
2.2 -3552.424
2.3 -3552.424
2.4 -3552.424
2.5 -3552.424
};

\addplot [ultra thick, lfwcolor]
table {%
	0.1 -3690.088
	0.2 -3692.252
	0.3 -3694.518
	0.4 -3696.786
	0.5 -3698.252
	0.6 -3698.59
	0.7 -3697.614
	0.8 -3694.76
	0.9 -3689.984
	1 -3683.654
	1.1 -3675.966
	1.2 -3666.94
	1.3 -3656.724
	1.4 -3645.7
	1.5 -3633.782
	1.6 -3621.182
	1.7 -3608.11
	1.8 -3594.748
	1.9 -3580.916
	2 -3566.868
	2.1 -3552.742
	2.2 -3538.552
	2.3 -3524.584
	2.4 -3510.206
	2.5 -3495.988
};
\addplot [ultra thick, efwcolor]
table {%
	0.1 -3694.574
	0.2 -3697.198
	0.3 -3691.716
	0.4 -3676.606
	0.5 -3653.698
	0.6 -3625.198
	0.7 -3592.316
	0.8 -3556.59
	0.9 -3518.876
	1 -3479.99
	1.1 -3441.206
	1.2 -3401.766
	1.3 -3362.5
	1.4 -3323.644
	1.5 -3284.862
};
\end{axis}

\end{tikzpicture}
		}
		\caption{\label{fig:energy:cityscapes:relative}Per $\lambda$, at the $5\textsuperscript{th}$ CRF iteration.}
	\end{subfigure}%
	\caption{\label{fig:energy:cityscapes} \textbf{CRF energy} averaged over \num{500} \emph{val} images of \textbf{Cityscapes} using DeepLabv3+ and Potts dense CRF. \textbf{(a)} Comparison between regularized Frank-Wolfe and the other methods for some selected values of the regularization weight $\lambda$. \textbf{(b)} Results of regularized Frank-Wolfe for different values of $\lambda$. \textbf{(c)} Energy per $\lambda$ after $5$ iterations.	Best viewed in color.
	}
\end{figure*}

\end{document}